\documentclass[11pt]{article}
\usepackage[utf8]{inputenc}
\usepackage{authblk}
\usepackage{xeCJK} 

\usepackage{graphicx}
\usepackage{amsmath, amssymb, amsthm, geometry, xcolor, bm, hyperref, listings}
\usepackage{booktabs}
\usepackage{comment}
\usepackage{enumitem} 

\usepackage{tikz}
\usetikzlibrary{shapes.geometric, arrows.meta, positioning, calc, decorations.pathmorphing}
\usetikzlibrary{3d, calc, arrows.meta}
\usepackage{geometry}
\usepackage{subcaption} 
\usetikzlibrary{3d, positioning, shapes.geometric, arrows.meta, calc, positioning, shapes.geometric, decorations.pathmorphing}

\usepackage{algorithm, algpseudocode}

\newtheorem{definition}{Definition}
\newtheorem{theorem}{Theorem}

\newtheorem{lemma}{Lemma}

\newtheorem{remark}{Remark}
\newtheorem{example}{Example}

\newtheorem{assumption}{Assumption}

\newtheorem{axiom}{Axiom}

\title{Statistically Meaningful Geometry (SMG) Beyond the Euclidean Paradigm, with Application to Generative AI}

\author[1,2]{Bing Cheng}
\author[3]{Yi-Shuai Niu} 
\author[5,6,7]{Howell Tong}
\author[3,4]{Shing-Tung Yau}

\affil[1]{Academy of Mathematics and Systems Science, Chinese Academy of Sciences, Beijing, China;
	bc2@amss.ac.cn}
\affil[2]{AMSS Center for Forecasting Science, Chinese Academy of Sciences, Beijing, China}

\affil[3]{Beijing Institute of Mathematical Sciences and Applications (BIMSA), Beijing, China}
\affil[4]{Yau Mathematical Sciences Center, Tsinghua University, Beijing, China}

\affil[5]{Department of Statistics, London School of Economics and Political Science, London WC2A 2AE, UK}
\affil[6]{Department of Statistics and Data Science, Tsinghua University, Beijing 100084, China}
\affil[7]{Paula and Gregory Chow Institute for the Studies in Economics, Xiamen University, Xiamen 361005, China}
\date{\today}

\begin{document}
	
	\maketitle
	\tableofcontents
	
	\begin{abstract}
			Conventional uniform convergence bounds and empirical risk minimization techniques break down in massive over-parameterized models, such as large language transformers and non-parametric biological sequence networks. Because these systems possess a near-infinite number of unconstrained internal degrees of freedom, their optimization landscapes develop flat vertical gauge valleys, rendering classical generalization metrics vacuous and inducing severe operational pathologies, specifically generative hallucination and catastrophic forgetting. This paper introduces the \textit{Statistically Meaningful Geometry} (SMG) framework, an alternative information-geometric paradigm that lifts deterministic parametric models into infinite-dimensional non-parametric Orlicz statistical manifolds. By modeling the total state space as a differential fiber bundle $(\mathcal{M}, \mathcal{B}, \pi, \mathcal{V}, \mathcal{H}, \omega)$, we introduce a rigorous Two-Fold Inference Paradigm. We formalize the application of an Ehresmann connection 1-form $\omega$ as a dynamic geometric filter that strips away vertical gauge noise (Structural Internal Directions, or $\text{SID}$) and isolates learning trajectories along the strictly non-degenerate horizontal distribution (Statistical Variational Directions, or $\text{SVD}\chi$). We prove that under connection-filtered pre-training, the model's out-of-distribution predictive variance is strictly upper-bounded by the finite diameter of the identifiable quotient base manifold $\mathcal{B}$, establishing a hard geometric containment of generative hallucinations. Furthermore, by projecting downstream updates onto the orthogonal complement of the historical horizontal carriage, we formalize the SMG Sequential Adaptation Flow, proving the total non-asymptotic elimination of catastrophic forgetting. Our framework replaces empirical fine-tuning heuristics with coordinate-free topological constraints, bridging advanced differential geometry with structural reliability in artificial intelligence.
	\end{abstract}
\vspace{1em}
\noindent \textbf{Key Phrases:} Infinite-Dimensional Information Geometry; Non-Parametric Orlicz Manifolds; Statistical Fiber Bundles; Ehresmann Connection Filtering; Over-Parameterization Generalization Paradox; Generative Hallucination Containment; Catastrophic Forgetting Elimination; Two-Fold Inference Paradigm.

\hrule
\vspace{1em}	
	\newpage
	
\section{Introduction: The Euclidean Illusion and the Blessing of Dimensionality}
\label{sec:introduction}

The rapid development of advanced statistical modeling and artificial intelligence has brought us to a critical epistemological crossroads. For decades, classical statistics and machine learning have operated under a shared, unexamined assumption: {\it that the parameter space hosting our models is a flat, passive Euclidean container}. Within this traditional paradigm, the ambient coordinate space acts merely as an inert stage—a static backdrop that absorbs numerical values without exerting any structural feedback on the optimization path or the validity of the inference itself. Every axis in this flat landscape is assumed to possess uniform ontological weight; a translation along any coordinate direction is implicitly presumed to yield a corresponding, predictable modification of the underlying statistical system. 

Statistically Meaningful Geometry (SMG), in this paper, forces a radical epistemological split away from this passive perspective, replacing it with the concept of the \textit{Active Manifold} where the geometry of the space actively dictates the validity, verifiability, and path of statistical outcomes.

\subsection{The Epistemological Blind Spot of Conventional Statistics}
\label{subsec:blind_spot}

The fundamental architectural crisis across conventional frequentist statistics, Bayesian estimation, and classical machine learning stems directly from this Euclidean assumption. The illusion of the passive container collapses completely when confronted with the infinite-dimensional or extremely high-dimensional spaces characteristic of modern generative architectures and complex data assimilation systems. 

In a flat, isotropic container $\mathbb{R}^p$, where $p \to \infty$ or $p \gg n$, conventional statistics struggles under the weight of the ``curse of dimensionality'' \cite{Amari2000, ChengTong2026}. Because a passive container lacks an intrinsic mechanism to distinguish between structurally meaningful variation and unobservable parameter redundancy, it assigns identical mathematical status to all directions. This structural failure stems from an implicit, historical blind spot: classical statisticians assumed a flat Euclidean connection and completely failed to recognize the active parallel transport and structural participation of the underlying statistical manifold. Optimal informational properties are not a property of the data alone, but represent a dynamic, geometric relationship between the environment (data), the system (manifold), and the mechanism (connection and parallel transport). 

Without this geometric awareness, classical inference engines drown in the infinite noise of non-parametric over-parameterization. Gradient descent steps wander aimlessly through flat, degenerate valleys, and classical model selection criteria, such as the Akaike Information Criterion (AIC) \cite{Akaike1974}, break down as the absolute parameter count $k \to \infty$. The passive space cannot protect the structural signal from being contaminated by the overwhelming dimensionality of the parameter background.

\subsection{Over-Parameterization and the Blessing of Dimensionality}
\label{subsec:blessing_of_dimensionality}

In classical statistical paradigms, high-dimensional parameter spaces are almost universally treated as a mathematical pathology. However, Statistically Meaningful Geometry (SMG) radically upends this traditional anxiety by introducing the \textbf{Blessing of Dimensionality} \cite{ChengTong2026}. In the context of modern generative architectures, such as trillion-parameter Transformers \cite{Vaswani2017}, extreme over-parameterization is not a structural defect that induces statistical collapse; rather, it acts as a mandatory topological lubricant that enables {\it generalizable learning and representation extraction}. 

To understand what is meant by the ``Blessing of Dimensionality,'' we must look beyond flat Euclidean metrics and evaluate the parameter space as a non-parametric Orlicz fiber bundle $(\mathcal{M}, \mathcal{B}, \pi, \mathcal{F})$ \cite{ChengTong2026}. When a network expands its internal parameter count toward infinity ($p \to \infty$), it does not merely stretch a flat coordinate system. Instead, it expands the volume of the vertical fiber $\mathcal{F}$—the system's \textbf{Internal Degrees of Freedom (IDoF)} and gauge redundancies ($\mathcal{V}_f$) \cite{ChengTong2026}.

The blessing of this high-dimensional expansion is three-fold:
\begin{enumerate}
	\item \textbf{Global Smoothness and Optimization Freedom:} In low-dimensional parametric models, optimization landscapes are riddled with sharp, deterministic local minima and saddle points that trap learning trajectories. In an infinite-dimensional ambient space $\mathcal{M}$, the vertical fiber provides a massive, interconnected topological reservoir \cite{Amari2000, ChengTong2026}. Complex optimization bottlenecks are smoothed out into high-dimensional valleys, providing gradient descent algorithms with virtually infinite degrees of freedom to bypass local obstructions without affecting the macroscopic output.
	\item \textbf{Decoupling of Mechanics and Meaning:} Extreme dimensionality allows the system to completely decouple its internal computational machinery from its external semantic signal. The infinite dimensions of the fiber absorb the structural perturbations, linguistic variances, and architectural configurations (IDoF) necessary to route information through multi-layer attention networks, while keeping the observable statistical projection fixed on a stable, low-dimensional base manifold $\mathcal{B}$ \cite{ChengTong2026, Vaswani2017}.
	\item \textbf{Absorptive Capacity for Continuous Embeddings:} The continuous latent vectors (embeddings) passing through a Transformer layer represent highly compressed semantic logic \cite{Vaswani2017}. The ultra-high-dimensional space acts as an optimal topological host, allowing these embeddings to form rich geometric trajectories that can be safely disentangled because the background parameters provide the necessary degrees of freedom to accommodate their nonlinear topologies.
\end{enumerate}

The shift from the curse to the blessing of dimensionality carries deep, game-changing implications for the mathematical formalization of artificial intelligence and high-dimensional data science. First, it provides a topological resolution of Occam's Razor. Classical statistical learning theory relies heavily on parameter-counting heuristics; regularizers like AIC \cite{Akaike1974} penalize a model linearly based on its parameter count $k$. Under that flat paradigm, an over-parameterized model should exhibit infinite variance and zero generalization capacity. The Blessing of Dimensionality resolves this paradox through SMG's \textbf{Orthogonal Metric Decomposition} ($T_f\mathcal{M} = \mathcal{H}_f \oplus \mathcal{V}_f$) \cite{ChengTong2026}. Because the infinite parameters are quarantined strictly within the vertical gauge space $\mathcal{V}_f$, they possess zero statistical variance under the first-order Fisher-Rao metric \cite{Rao1945, ChengTong2026}. The effective empirical complexity of the model is bounded solely by the dimension of the horizontal base manifold ($d$). SMG thus demonstrates that over-parameterized models do not violate Occam's Razor; rather, the geometry of the fiber bundle naturally filters out the infinite dimensions, ensuring that the true statistical capacity remains structural, parsimonious, and tightly constrained.

Second, it marks the genesis of the representation space. Practically, the blessing implies that representation learning is fundamentally an act of \textbf{Gauge-Filtered Projection}. The unconstrained, over-parameterized weight space gives the model the mathematical flexibility to ingest vast, chaotic data environments. However, the true representation space—the continuous embedding manifold—is formed by applying the projection map $\pi$, which strips away the infinite-dimensional gauge redundancies acting as internal translation machinery. What remains is the \textbf{Statistically Verifiable Directions (SVD$\chi$)}, a highly curated, lower-dimensional horizontal subspace where geometry explicitly dictates semantic outcome. Without the massive vertical dimensionality to absorb the environmental noise, this clean, stable horizontal representation space could never crystallize.

Third, it delivers algorithmic stability and non-parametric sufficiency. By treating high dimensionality as a blessing rather than a curse, SMG provides the exact mathematical justification needed to build algorithms that safely navigate infinite-dimensional spaces in finite time. By enforcing metric-compatibility \cite{ChengTong2026}, the system guarantees that updates flowing along the horizontal lift preserve their informational weight. The infinite-dimensional vertical fiber acts as a shock absorber, protecting the structural core from numerical drift and empirical variance explosion, transforming what was once a statistical liability into a powerful topological asset.

\subsection{The Two-Fold Inference Paradigm in SMG}
\label{subsec:two_fold_inference}

The philosophical shift from a passive Euclidean container to an active, metric-compatible fiber bundle topology demands a fundamental restructuring of how inference is executed in over-parameterized statistical systems. Historically, data science and classical information geometry \cite{Amari2000} have treated inference as a monolithic task, forcing an optimization algorithm to compute statistical updates across every single available dimension of a parameter space. When confronted with modern, trillion-parameter generative architectures or infinite-dimensional non-parametric models, this unified approach triggers severe theoretical and computational degradation.

To resolve this structural crisis, Statistically Meaningful Geometry (SMG) introduces a radical division of labor: the \textbf{Two-Fold Inference Paradigm} \cite{ChengTong2026}. The core motivation behind this framework is to show that statistical estimation and differential-geometric filtering do not compete; instead, they operate orthogonally within distinct topological sub-spaces of the total statistical fiber bundle $(\mathcal{M}, \mathcal{B}, \pi, \mathcal{F})$. SMG proves that we do not destroy classical statistics \cite{Fisher1922}; rather, we run conventional statistical procedures horizontally on a base manifold, while simultaneously deploying advanced differential geometry vertically inside the fiber. This splits our conceptual view of the learning system into two primary fields of dynamic interaction.

\subsubsection{Horizontal Statistical Inference: Operationalizing the Base Manifold}
The first fold of our paradigm preserves and empowers the entire heritage of conventional mathematical statistics \cite{Fisher1922} and parametric information geometry \cite{Amari2000}. The foundational motivation here is to isolate a finite-dimensional or highly curated macroscopic base manifold $\mathcal{B}$, explicitly defined as the structural domain of empirical testability. This base space represents the domain of true statistical visibility, driven entirely by what we term the \textbf{Statistically Verifiable Directions (SVD$\chi$)}.

The operational idea behind SVD$\chi$ is to identify the unique directions in the ambient space where parameter variations generate a strictly positive, non-zero variance under the data's geometry. These are the directions where any movement actively distorts the macroscopic probability measure, making the change detectable by finite empirical data. On and within this horizontal base manifold $\mathcal{B}$ and its corresponding horizontal tangent leaves, the classical rules of probability and statistical estimation remain fully intact and operational:
\begin{itemize}
	\item \textbf{Empirical Falsifiability:} Following the core tenets of the philosophy of science \cite{Popper1959}, standard inference procedures—such as Maximum Likelihood Estimation (MLE) \cite{Fisher1922}, Bayesian posterior updates, and minimum contrast methods—run smoothly across the coordinates of the SVD$\chi$ leaf. Because statistics is formulated directly on this verifiable manifold, every parametric update corresponds to a genuine, testable statistical hypothesis.
	\item \textbf{Non-Degenerate Information Fields:} On the base manifold, the Fisher-Rao information metric is structurally protected from collapsing. Because it is restricted entirely to the SVD$\chi$ domain, every spatial displacement corresponds to an observable change in the underlying distribution, meaning that classical concepts of variance, score functions, and information bounds remain completely valid and non-singular.
\end{itemize}

\subsubsection{Vertical Geometric Inference: Navigating the Structural Internal Directions}
The second, entirely novel fold of the paradigm activates precisely where conventional statistical inference encounters its mathematical limits: inside the unobservable vertical fiber $\mathcal{F}_b$. This region constitutes the internal structure of the model, driven entirely by what we term the \textbf{Structural Internal Directions (SID)}.

The core idea motivating the introduction of SID is to explicitly model the hidden symmetries, node permutations, and scale invariances that naturally emerge from extreme over-parameterization \cite{Watanabe2009, Sussmann1992}. Because these internal transformations alter the physical weights of a network without changing the final output probability distribution, the first-order empirical risk gradient faces absolute flat degeneracy along these directions. Classical statistical variance collapses to exactly zero inside this subspace, rendering traditional tools like maximum likelihood completely blind. 

Instead of treating this zero-variance domain as a pathology to be pruned away, the motivation of SMG is to weaponize it. SID is formalized not as a source of unidentifiable noise, but as a mandatory \textbf{Structural Internal Capacity} and an internal playground for the optimization algorithm. SMG fills the traditional epistemological void inside this space by introducing \textbf{Vertical Geometric Inference}, replacing probabilistic calculations with pure differential topology:
\begin{itemize}
	\item \textbf{The Connection as a Structural Probe:} Instead of running regressions on hidden internal parameters, the system deploys a metric-compatible Ehresmann connection \cite{Komayashi1996}. The connection functions as a non-probabilistic geometric operator that gauges the internal alignment of the hidden states and monitors the routing of internal degrees of freedom (IDoF) during learning.
	\item \textbf{The Invariant Workspace:} Vertical inference tracks how the redundant internal translation machinery shifts to flatten the loss landscape. It treats the SID space as a dynamic workspace that absorbs optimization tension, providing a high-dimensional ambient optimization reservoir (or operational system space)  that allows gradient trajectories to smoothly bypass local obstructions and sharp minima without altering the verifiable empirical output.It highlights how the immense excess dimensions act as a flexible mathematical container for optimization paths without muddying the statistical terminology.
\end{itemize}

\subsubsection{The Orthogonal Synthesis: Coexistence without Structural Leakage}
The ultimate conceptual objective of SMG lies in the mathematical synthesis of these two distinct folds of the paradigm:
\begin{itemize}
	\item {\bf The First Fold (SVD$\chi$):} Operates on the macroscopic base manifold $\mathcal{B}$, managing the domain of active covariance and empirical verification where the environment tests the intelligence of the model.
	\item {\bf The Second Fold (SID):} Operates inside the vertical fiber $\mathcal{F}_b$, managing the internal degrees of freedom, architectural routing, and structural flatness via geometric calculus where first-order statistical variance is zero.
\end{itemize}

The foundational idea that allows these two fields to coexist without destructive interference is the Cheng-Tong orthogonal metric decomposition \cite{ChengTong2026}. By showing that the total tangent space of the statistical environment can be severed into a direct sum where the horizontal leaf is strictly orthogonal to the vertical gauge fiber under the Fisher-Rao metric, SMG guarantees that the horizontal statistical engine and the vertical topological engine never contaminate one another. 

As an over-parameterized model updates, the empirical data innovations enter the active manifold. The horizontal component triggers classical parametric learning on the SVD$\chi$ leaf, while the vertical component simultaneously drives the geometric reconfiguration of internal redundancies within the SID space. This orthogonal synthesis reconciles the historic conflict between high-dimensional machine learning heuristics and strict mathematical statistics. It proves that extreme parameter density is not a statistical liability, but a geometric asset that—when governed by a metric-compatible connection—provides the exact topological reservoir required to support robust, parsimonious, and stable inference.

\subsection{What is the Statistically Meaningful Geometry (SMG)? A Snapshot}
\label{subsec:smg_snapshot}

The core philosophical and operational machinery of Statistically Meaningful Geometry (SMG) can be understood through a concise, integrated snapshot of its structural architecture. Rather than treating high-dimensional parameter spaces as flat containers where every individual numerical weight demands statistical parameter selection, SMG installs a formal, non-parametric fiber bundle structure. This architecture establishes a strict mathematical separation between the verifiable statistical intelligence of the model and its internal geometric reconfigurations.

\subsubsection{Topological Framework: The SMG Fiber Bundle}

To liberate information geometry from the pathologies of over-parameterized coordinate systems, we bypass finite weight vectors entirely and formalize the architecture directly on the infinite-dimensional space of smooth, strictly positive probability distributions. 

The structural foundation of SMG relies on the construction of the total space, which is formalized as the infinite-dimensional Pistone-Sempi Orlicz manifold $\mathcal{M}$. This total space contains all valid, smooth joint probability density functions $f(x)$ evaluated over a covariate space $\mathcal{X}$. 

SMG organizes this infinite-dimensional total space by installing a general fiber bundle topology $(\mathcal{M}, \mathcal{B}, \pi, \mathcal{F})$, which defines distinct topological jurisdictions:
\begin{itemize}
	\item \textbf{The Base Manifold $\mathcal{B}$:} Defined as the quotient space $\mathcal{B} = \mathcal{M} / \mathcal{F}$ under a smooth surjective submersion $\pi: \mathcal{M} \to \mathcal{B}$. The finite dimension $d$ of this base space matches the dimensionality of the verifiable environmental features. The geometry of this macroscopic base manifold $\mathcal{B}$ is precisely what constitutes the true \textbf{Statistically Meaningful Geometry}.
	
	\item \textbf{The Gauge Fiber $\mathcal{F}$:} For any given statistical distribution $b \in \mathcal{B}$, the fiber $\mathcal{F}_b = \pi^{-1}(b)$ represents an infinite-dimensional vertical submanifold containing all observationally equivalent internal configuration profiles (micro-states) of the mechanism that map to the exact same macroscopic probability measure. This fiber delineates the model's internal structural workspace, capturing the full range of hidden structural configurations that leave the external environmental interaction invariant.
\end{itemize}

\subsubsection{Geometric Partitioning of the Tangent Bundle}

The operational core of the SMG methodology is driven by an Ehresmann connection $\omega$. This geometric operator acts directly on the tangent bundle $T\mathcal{M}$ to perform an exact, metric-orthogonal splitting of the tangent space $T_f\mathcal{M}$ at any functional distribution state $f$:
\begin{equation}
	T_f\mathcal{M} = \mathcal{H}_f \oplus \mathcal{V}_f, \quad \text{where} \quad \mathcal{H}_f \perp_{g_f} \mathcal{V}_f
\end{equation}
where $g_f$ is the generalized infinite-dimensional Fisher-Rao information metric. This projection isolates learning into two mutually exclusive structural channels:

\begin{enumerate}
	\item \textbf{SVD$\chi$ (Statistically Verifiable Directions):} The horizontal tangent subspace $\mathcal{H}_f$ corresponds to the domain of empirical testability. It is intrinsically generated by the span of the non-parametric Stein score functions derived from the observable covariate gradients:
	$$ \mathcal{H}_f \equiv \text{SVD}\chi := \text{span} \left\{ \frac{\partial \ln f(u)}{\partial u_i} \right\}_{i=1}^d \subset T_f\mathcal{M} $$
	Because the number of observable data dimensions $d$ is finite ($d \ll \infty$), $\dim(\text{SVD}\chi) = d$. On this highly curated subspace, the Fisher-Rao metric is guaranteed to be strictly positive-definite and invertible, restoring classical statistical efficiency bounds.
	\item \textbf{SID (Structural Internal Directions):} The vertical tangent subspace $\mathcal{V}_f$ corresponds to the internal structural parameters of the over-parameterized network. It is defined as the kernel of the projection differential and forms the absolute metric-orthogonal complement to the horizontal space:
	$$ \mathcal{V}_f \equiv \text{SID} := \big\{ v \in T_f\mathcal{M} \mid g_f(v, s_i) = 0, \;\; \forall i = 1, \dots, d \big\} $$
	By the Rank-Nullity theorem on Hilbert spaces, since $\dim(\mathcal{M}) = \infty$ and $\dim(\text{SVD}\chi) = d$, it follows that $\dim(\text{SID}) = \infty$. SID houses the infinite internal degrees of freedom (IDoF) and structural redundancies generated by scaling neural networks to trillions of parameters.
\end{enumerate}

\subsubsection{Dynamic Operational Execution}

By integrating this orthogonal partitioning with algorithmic execution, the Two-Fold Inference Paradigm resolves the over-parameterization paradox through a precise division of labor:
\begin{itemize}
	\item \textbf{Horizontal Gradient Trajectories:} Empirical learning and parameter optimization occur exclusively via the horizontal lift along the SVD$\chi$ leaf. Because this leaf is isometric to the finite-dimensional base manifold $\mathcal{B}$, the true statistical complexity of the model is permanently bounded by $d$, protecting the network from empirical variance explosion and guaranteeing clean PAC-Bayesian generalization regardless of the number of physical parameters.
	\item \textbf{Vertical Flatness Navigation:} Inside the infinite-dimensional SID space, the first-order statistical variance of the empirical data is identically zero. The optimizer utilizes this zero-variance domain as a frictionless thermodynamic reservoir. It executes pure differential calculus to reconfigure internal routing and minimize higher-order structural curvature (Hessian trace), allowing the model to smoothly discover flat, highly stable minima without corrupting its verifiable out-of-sample predictive output.
\end{itemize}

\subsection{The Four Core Axioms of Statistically Meaningful Geometry}
\label{subsec:four_axioms}

Before detailing the complex non-parametric differential topology and fiber bundle constructs that power Statistically Meaningful Geometry (SMG), we must establish an absolute foundation of logical authority. Classical probability theory rests upon the Kolmogorov axioms \cite{Kolmogorov1933}, which define the measure-theoretic properties of random variables within a static probability space. However, Kolmogorov's axioms do not describe the \textit{structural relationship} between a high-dimensional environment and the geometric mechanism of the learning system itself. 

To fill this foundational void, SMG introduces four core axioms. These axioms replace heuristic machine learning concepts with rigorous topological and epistemological laws, defining exactly how non-parametric information is extracted from a chaotic environment and anchored onto a stable statistical geometry.

\begin{axiom}[The Environment Set, $\mathcal{E}$]
	The Environment Set $\mathcal{E}$ is the infinite-dimensional space of all external observable states, raw empirical data distributions, and stochastic variations. It represents the uncurated, highly complex, and inherently noisy domain of data reality that the statistical model attempts to assimilate, prior to any geometric filtering or structural projection.
\end{axiom}

\begin{axiom}[The System Set, $\mathcal{S}$]
	The System Set $\mathcal{S}$ represents the active structural manifold—the internal mathematical representation space of the statistical observer. It is explicitly constructed as the macroscopic, finite-dimensional base manifold $\mathcal{B}$ under the SMG topology. The System is not a passive coordinate container; it is equipped with a strictly positive-definite Fisher-Rao metric $G_{\mathcal{B}}$ that actively defines the intrinsic distance and geometric relations between structurally distinct macroscopic states.
\end{axiom}

\begin{axiom}[The Structural Mechanism, $\mathcal{F}$]
	The Structural Mechanism $\mathcal{F}$ is the rigorous mathematical bridge that dictates how the System $\mathcal{S}$ interacts with and processes the Environment $\mathcal{E}$. Formally, $\mathcal{F}$ acts as a topological projection map (a Geometric Filter) powered by a metric-compatible Ehresmann connection $\omega$. It forces raw environmental data variations to be decomposed into two mutually exclusive channels: a verifiable empirical signal (the horizontal lift onto the SVD$\chi$ leaf) that updates the System, and redundant internal parameters (the vertical gauge space SID) that are filtered out of the statistical signal.
\end{axiom}

\begin{axiom}[The Invariance Principle]
	The Invariance Principle governs the absolute boundary of statistical visibility and empirical testability. It states that macroscopic statistical meaning and out-of-sample predictive probability measures must remain strictly invariant under any translation or reconfiguration of the System's internal degrees of freedom (movement entirely within the vertical Structural Internal Directions, SID). 
	
	Consequently, true empirical learning—formally defined as a gauge symmetry break—occurs if and only if an environmental input forces a non-zero horizontal translation along the SVD$\chi$ leaf of the base manifold\footnote{We will study gauge symmetry break in next sister paper.}. All other internal parameter shifts are statistically degenerate and empirically invisible.
\end{axiom}

These four axioms construct an unshakeable logical framework. They ensure that before we compute any tensor or execute any optimization trajectory, we mathematically guarantee that the system will only react to true environmental structures, entirely ignoring the infinite-dimensional illusions generated by over-parameterization.

\subsection{A Fair and Complete Evaluation of ``Beyond''}
\label{subsec:evaluation_of_beyond}

The use of the term ``Beyond'' in the title of Statistically Meaningful Geometry (SMG) is not a rhetorical flourish; it is a strict topological demarcation. To ensure a fair and rigorous academic evaluation, it is necessary to explicitly delineate the mathematical boundaries of SMG against the prevailing paradigms of data modeling: Conventional Statistics, Information Geometry, and Machine Learning. By defining these boundaries, we clarify exactly where SMG supersedes existing frameworks and where it assimilates their strengths.

\subsubsection{Beyond Conventional Statistics}
Conventional mathematical statistics, founded upon the passive Euclidean paradigm \cite{Fisher1922}, implicitly assumes that the parameter space is a homogeneous stage where every dimension contributes genuine, observable variance. This flat framework is highly effective in low-dimensional, highly constrained settings ($n \gg p$). However, when forced into over-parameterized environments ($p \to \infty$), classical statistics encounters a fatal epistemological and computational limit: the curse of dimensionality. Standard estimators suffer from infinite variance, and traditional model selection criteria (like AIC \cite{Akaike1974}) explode because they penalize every dimension equally. 

SMG supersedes this paradigm by proving that infinite-dimensional data spaces are not flat Euclidean containers, but structured fiber bundles. By quarantining redundant internal degrees of freedom strictly inside the vertical SID space, SMG rescues classical statistics. It proves that conventional statistics is merely a low-dimensional special case that only remains mathematically valid when executed exclusively on the horizontal base manifold (the Statistically Meaningful Geometry) along the SVD$\chi$ leaf.

\subsubsection{Beyond Amari's Information Geometry}
Classical Information Geometry, pioneered by Shun-ichi Amari \cite{Amari2000}, was a monumental theoretical breakthrough that successfully mapped statistical probability distributions onto Riemannian manifolds using the Fisher-Rao Information Metric \cite{Rao1945}. However, Amari's framework was fundamentally restricted to \textit{finite-dimensional, strictly parametric} families (e.g., exponential families). It lacked the topological architecture required to handle the infinite-dimensional, non-parametric Banach and Orlicz spaces \cite{ChengTong2026} that characterize modern complex systems and generative foundation models.

SMG evolves far beyond classical Information Geometry by upgrading the parametric manifold to an active Statistical Fiber Bundle. Where Amari focused on computing geometric distances within a single, fixed coordinate system, SMG introduces the orthogonal metric decomposition and the metric-compatible Ehresmann connection. This allows SMG to actively filter out unobservable gauge symmetries along the SID space—a dynamic topological mechanism that is fundamentally absent from Amari's original finite-dimensional theory.

\subsubsection{Beyond Machine Learning Heuristics}
Modern Machine Learning, particularly deep generative AI and Transformer architectures \cite{Vaswani2017}, has achieved unprecedented empirical success by exploiting massive over-parameterization \cite{Belkin2019}. Yet, this success remains mathematically opaque. The field relies heavily on heuristic optimization strategies, manual weight pruning, and numerical black boxes without a strict structural or geometric foundation.

SMG goes far beyond these ungrounded heuristics by providing the exact mathematical proof for the over-parameterization paradox. Through the Over-Parameterization Equivalence Theorem, SMG proves that scaling physical weights to trillions of parameters does not increase statistical capacity, but expands the vertical SID fiber, establishing a frictionless thermodynamic reservoir that enables SGD to find flat minima. SMG thus elevates machine learning from black-box engineering to a rigorous, self-contained domain of modern differential geometry.

\subsection{The Two-Fold Inference Paradigm in SMG}
\label{subsec:two_fold_inference}

The philosophical shift from a passive Euclidean container to an active, metric-compatible fiber bundle topology demands a fundamental restructuring of how inference is executed in over-parameterized statistical systems. Historically, data science and classical information geometry \cite{Amari2000} have treated inference as a monolithic task, forcing an optimization algorithm to compute statistical updates across every single available dimension of a parameter space. When confronted with modern, trillion-parameter generative architectures or infinite-dimensional non-parametric models, this unified approach triggers severe theoretical and computational degradation.

To resolve this structural crisis, Statistically Meaningful Geometry (SMG) introduces a radical division of labor: the \textbf{Two-Fold Inference Paradigm} \cite{ChengTong2026}. The core motivation behind this framework is to show that statistical estimation and differential-geometric filtering do not compete; instead, they operate orthogonally within distinct topological sub-spaces of the total statistical fiber bundle $(\mathcal{M}, \mathcal{B}, \pi, \mathcal{F})$. SMG proves that we do not destroy classical statistics \cite{Fisher1922}; rather, we run conventional statistical procedures horizontally on a base manifold, while simultaneously deploying advanced differential geometry vertically inside the fiber. This splits our conceptual view of the learning system into two primary fields of dynamic interaction.

\subsubsection{Horizontal Statistical Inference: Operationalizing the Base Manifold}
The first fold of our paradigm preserves and empowers the entire heritage of conventional mathematical statistics \cite{Fisher1922} and parametric information geometry \cite{Amari2000}. The foundational motivation here is to isolate a finite-dimensional or highly curated macroscopic base manifold $\mathcal{B}$, explicitly defined as the structural domain of empirical testability. This base space represents the domain of true statistical visibility, driven entirely by what we term the \textbf{Statistically Verifiable Directions (SVD$\chi$)}.

The operational idea behind SVD$\chi$ is to identify the unique directions in the ambient space where parameter variations generate a strictly positive, non-zero variance under the data's geometry. These are the directions where any movement actively distorts the macroscopic probability measure, making the change detectable by finite empirical data. On and within this horizontal base manifold $\mathcal{B}$ and its corresponding horizontal tangent leaves, the classical rules of probability and statistical estimation remain fully intact and operational:
\begin{itemize}
	\item \textbf{Empirical Falsifiability:} Following the core tenets of the philosophy of science \cite{Popper1959}, standard inference procedures—such as Maximum Likelihood Estimation (MLE) \cite{Fisher1922}, Bayesian posterior updates, and minimum contrast methods—run smoothly across the coordinates of the SVD$\chi$ leaf. Because statistics is formulated directly on this verifiable manifold, every parametric update corresponds to a genuine, testable statistical hypothesis.
	\item \textbf{Non-Degenerate Information Fields:} On the base manifold, the Fisher-Rao information metric is structurally protected from collapsing. Because it is restricted entirely to the SVD$\chi$ domain, every spatial displacement corresponds to an observable change in the underlying distribution, meaning that classical concepts of variance, score functions, and information bounds remain completely valid and non-singular.
\end{itemize}

\subsubsection{Vertical Geometric Inference: Navigating the Structural Internal Directions}
The second, entirely novel fold of the paradigm activates precisely where conventional statistical inference encounters its mathematical limits: inside the unobservable vertical fiber $\mathcal{F}_b$. This region constitutes the internal structure of the model, driven entirely by what we term the \textbf{Structural Internal Directions (SID)}.

The core idea motivating the introduction of SID is to explicitly model the hidden symmetries, node permutations, and scale invariances that naturally emerge from extreme over-parameterization \cite{Watanabe2009, Sussmann1992}. Because these internal transformations alter the physical weights of a network without changing the final output probability distribution, the first-order empirical risk gradient faces absolute flat degeneracy along these directions. Classical statistical variance collapses to exactly zero inside this subspace, rendering traditional tools like maximum likelihood completely blind. 

Instead of treating this zero-variance domain as a pathology to be pruned away, the motivation of SMG is to weaponize it. SID is formalized not as a source of unidentifiable noise, but as a mandatory \textbf{Structural Internal Capacity} and an internal playground for the optimization algorithm. SMG fills the traditional epistemological void inside this space by introducing \textbf{Vertical Geometric Inference}, replacing probabilistic calculations with pure differential topology:
\begin{itemize}
	\item \textbf{The Connection as a Structural Probe:} Instead of running regressions on hidden internal parameters, the system deploys a metric-compatible Ehresmann connection \cite{Komayashi1996}. The connection functions as a non-probabilistic geometric operator that gauges the internal alignment of the hidden states and monitors the routing of internal degrees of freedom (IDoF) during learning.
	\item \textbf{The Invariant Workspace:} Vertical inference tracks how the redundant internal translation machinery shifts to flatten the loss landscape. It treats the SID space as a dynamic workspace that absorbs optimization tension, providing a high-dimensional heat-bath that allows gradient trajectories to smoothly bypass local obstructions and sharp minima without altering the verifiable empirical output.
\end{itemize}

\subsubsection{The Orthogonal Synthesis: Coexistence without Structural Leakage}
The ultimate conceptual objective of SMG lies in the mathematical synthesis of these two distinct folds of the paradigm:
\begin{itemize}
	\item {\bf The First Fold (SVD$\chi$):} Operates on the macroscopic base manifold $\mathcal{B}$, managing the domain of active covariance and empirical verification where the environment tests the intelligence of the model.
	\item {\bf The Second Fold (SID):} Operates inside the vertical fiber $\mathcal{F}_b$, managing the internal degrees of freedom, architectural routing, and structural flatness via geometric calculus where first-order statistical variance is zero.
\end{itemize}

The foundational idea that allows these two fields to coexist without destructive interference is the Cheng-Tong orthogonal metric decomposition \cite{ChengTong2026}. By showing that the total tangent space of the statistical environment can be severed into a direct sum where the horizontal leaf is strictly orthogonal to the vertical gauge fiber under the Fisher-Rao metric, SMG guarantees that the horizontal statistical engine and the vertical topological engine never contaminate one another. 

As an over-parameterized model updates, the empirical data innovations enter the active manifold. The horizontal component triggers classical parametric learning on the SVD$\chi$ leaf, while the vertical component simultaneously drives the geometric reconfiguration of internal redundancies within the SID space. This orthogonal synthesis reconciles the historic conflict between high-dimensional machine learning heuristics and strict mathematical statistics. It proves that extreme parameter density is not a statistical liability, but a geometric asset that—when governed by a metric-compatible connection—provides the exact topological reservoir required to support robust, parsimonious, and stable inference.

\subsection{Outline and Structural Plan of the Paper}
\label{sec:outline_plan}

The remainder of this paper is organized as a systematic progression from 
fundamental information-geometric axioms to definitive structural solutions 
for deep learning vulnerabilities. The architectural blueprint unfolds as 
follows:

In Section~\ref{sec:smg_framework}, we construct the infinite-dimensional 
non-parametric statistical manifold infrastructure over centered Orlicz spaces, 
formalizing the core axioms of the Statistically Meaningful Geometry (SMG) 
framework and defining its general fiber bundle structure.

Section~\ref{sec:svd_sid_foundations} details the mathematical foundations of the 
tangent space decomposition, separating the total space into the horizontal 
Statistically Verifiable Directions ($\text{SVD}\chi$) and the vertical 
Structural Internal Directions ($\text{SID}$), while establishing the conditions 
for Frobenius integrability and horizontal leaf construction.

Section~\ref{sec:ehresmann_connection} introduces the Ehresmann connection 
1-form $\omega$ and the horizontal projection operator $P^H$. We prove the 
Quarantining Theorem, demonstrating how filtering optimization tracks through the 
connection preserves the structural integrity of latent representations against 
vertical gauge noise.

Section~\ref{sec:classical_inference_base} analyzes classical static statistical 
inference on the base manifold $\mathcal{B}$, establishing base space identifiability 
for maximum likelihood estimation and deriving a geometric formulation of Wilks' 
theorem under over-parameterization.

Section~\ref{sec:horizontal_learning_dynamics} develops the dynamic horizontal 
learning framework, synthesizing ambient connection filtering, horizontal leaf 
learning, and macro statistical inference into a unified geometric learning embedding.

Section~\ref{sec:generalization_pac} provides a rigorous PAC-Bayesian structural 
analysis of the over-parameterization generalization paradox, proving the Capacity 
Collapse Theorem which bounds out-of-sample risk within the finite metric volume of 
the base space.

Section~\ref{sec:transformer_smg} focuses on deep neural applications, executing the 
non-parametric embedding of Transformer architectures. We mathematically formalize 
the disasters of conventional unconstrained pre-training—namely generative 
hallucination and catastrophic forgetting—and present the advanced SMG pre-training 
and sequential adaptation workflows that achieve their complete geometric resolution.

Finally, Section~\ref{sec:conclusion} concludes the paper with a synthesized overview 
and outlines high-impact application-oriented future research roadmaps for the SMG 
framework.

\section{The Core Axiomatic Framework of Statistically Meaningful Geometry (SMG)}
\label{sec:smg_framework}

The foundational crisis of classical information geometry when applied to modern over-parameterized statistical architectures stems from a misplaced reliance on finite-dimensional, locally coordinate-bound Euclidean assumptions \cite{Amari2000}. When the number of hidden structural degrees of freedom in an ambient system vastly exceeds the number of observational data points, or when the underlying mechanism is inherently non-parametric, the classical parameter-to-probability mapping collapses into extreme non-identifiability. 

To resolve this challenge, Statistically Meaningful Geometry (SMG) is formulated not as a singular base manifold of statistical densities, but as a global differential-geometric fiber bundle system $(\mathcal{M}, \mathcal{B}, \pi, \mathcal{V}, \mathcal{H}, \omega_f, g_f)$. Building on the conceptual foundations introduced in Section 1, we now establish the formal mathematical definitions of the four core axioms that govern this framework, anchoring them within an infinite-dimensional non-parametric total space.

\subsection{The Non-Parametric Infinite-Dimensional Total Manifold}

Let $(\mathcal{X}, \Sigma, \mu)$ be a standard $\sigma$-finite measure space representing the observation space, where $x \in \mathcal{X}$ denotes the observation covariate variable governed by the classical measure-theoretic foundations \cite{Kolmogorov1933}. The total space $\mathcal{M}$ of SMG is formulated as an infinite-dimensional, non-parametric statistical manifold constructed on the exponential Orlicz space framework following the Pistone-Sempi construction \cite{Pistone1995}.

\begin{assumption}
	\label{ass:orlicz_space}
	The local chart transitions of the infinite-dimensional total space $\mathcal{M}$ are modeled on the Orlicz space $L^\Phi(P_f)$, where the Young function $\Phi: \mathbb{R} \to [0, \infty)$ is defined by the exponential growth function $\Phi(u) = \cosh(u) - 1$. The associated Luxemburg norm for a measurable function $u: \mathcal{X} \to \mathbb{R}$ with respect to the probability density $f(x)$ is given by:
	\begin{equation}
		\|u\|_{L^\Phi(P_f)} = \inf \left\{ k > 0 \;\middle|\; \mathbb{E}_{f} \left[ \Phi\left( \frac{u(x)}{k} \right) \right] \le 1 \right\}.
	\end{equation}
\end{assumption}		
		\begin{definition}
			\label{def:total_manifold}
			The total space $\mathcal{M}$ of SMG is the infinite-dimensional non-parametric statistical manifold of all probability measures absolutely continuous with respect to the reference measure:
			\begin{equation}
				\mathcal{M} = \left\{ f: \mathcal{X} \to \mathbb{R}^+ \;\middle|\; \int_{\mathcal{X}} f(x)\mu(dx) = 1, \, \ln\left(\frac{f}{f_{\text{Ref}}}\right) \in L^\Phi(P_{\text{Ref}}) \right\}.
			\end{equation}
			For any density $f \in \mathcal{M}$, the tangent space $T_f\mathcal{M}$ is topologically isomorphic to the centered Orlicz subspace representing the non-parametric score space:
			\begin{equation}
				T_f\mathcal{M} = \left\{ v \in L^\Phi(P_f) \;\middle|\; \mathbb{E}_f[v(X)] = \int_{\mathcal{X}} v(x)f(x)\mu(dx) = 0 \right\}.
			\end{equation}
		\end{definition}
		
\subsection{Formalization of the Four Core Axioms}
		
		To eliminate descriptive redundancies with Section 1, we bypass conceptual introductions and re-state the foundational axioms of the SMG framework with full mathematical precision.

\begin{itemize}
\item {\bf Axiom 1}[The System Set $\mathcal{S}$]
	{\it There exists a high-dimensional or infinite-dimensional topological configuration space $\mathcal{S}$, called the System Set. Elements $s \in \mathcal{S}$ represent the full internal, operational parameterizations of the mechanism. The space $\mathcal{S}$ acts as the \textbf{Ambient Optimization Reservoir}, housing the unconstrained parameters that drive the internal micro-states of the system.}

\item {\bf Axiom 2}[The Environment Set $\mathcal{E}$]
	{\it There exists an Environment Set $\mathcal{E}$, defined as the measurable sample space $(\mathcal{X}, \Sigma)$ containing all possible manifestations of the observation covariate variable $x \in \mathcal{X}$, satisfying measure-theoretic completeness \cite{Kolmogorov1933}.
}
		
\item {\bf Axiom 3}[The Structural Mechanism $\mathcal{F}$]
	{\it There exists a surjective structural mechanism map $\mathcal{F}: \mathcal{S} \to \mathcal{M}$ that maps each internal system configuration $s \in \mathcal{S}$ from the Ambient Optimization Reservoir to a unique continuous probability density function $f(x; s) \in \mathcal{M}$ over the Environment Set $\mathcal{E}$. The map $\mathcal{F}$ satisfies the condition that its differential $d\mathcal{F}_s: T_s\mathcal{S} \to T_{\mathcal{F}(s)}\mathcal{M}$ is a smooth, bounded linear operator with a non-trivial kernel ($\ker(d\mathcal{F}_s) \neq \{\mathbf{0}\}$) whenever the system is over-parameterized.
}
		
\item {\bf Axiom 4}[The Invariance Principle]
	{\it Statistical inference, operational optimization, and information conservation laws must be invariant under any transformations within the System Set $\mathcal{S}$ that lie entirely within the fibers of $\mathcal{F}$. That is, if $\mathcal{F}(s_1) = \mathcal{F}(s_2) = f(x)$, then $s_1$ and $s_2$ belong to the same statistically invariant configuration class, and any motion between them yields zero observable statistical information change across the environment $\mathcal{E}$.
	}
\end{itemize}			

\subsection{The General Fiber Bundle Structure Beyond Lie Groups}
		
		By framing the system through these axioms, we observe that the geometry of over-parameterization is fundamentally the study of the fiber preimages $\mathcal{F}^{-1}(f)$. We explicitly reject the principal fiber bundle and Lie group formulation because structural symmetries in complex statistical models do not form free and proper homogeneous group actions \cite{Komayashi1996}. Instead, we construct a General Fiber Bundle equipped with an Ehresmann connection.
		
		\begin{definition}
			\label{def:smg_bundle}
			The Statistically Meaningful Geometry (SMG) framework is a general differential fiber bundle tuple $\mathcal{B}_{\text{SMG}} = (\mathcal{M}, \mathcal{B}, \pi, \mathcal{V}, \mathcal{H})$ defined as follows:
			\begin{enumerate}
				\item $\mathcal{M}$ is the infinite-dimensional Orlicz total statistical manifold (Definition \ref{def:total_manifold}).
				\item $\mathcal{B}$ is a smooth Hausdorff statistical manifold representing the Base Manifold, whose elements $p \in \mathcal{B}$ parameterize the uniquely identifiable macroscopic statistical profiles.
				\item $\pi: \mathcal{M} \to \mathcal{B}$ is a smooth, surjective submersion mapping an unconstrained total density $f \in \mathcal{M}$ to its unique observable statistical profile $\pi(f) = p \in \mathcal{B}$.
				\item For each $f \in \mathcal{M}$, the \textbf{Gauge Fiber} $\mathcal{F}_p = \pi^{-1}(p)$ represents an infinite-dimensional vertical submanifold containing all observationally equivalent internal configuration profiles (micro-states) of the mechanism that map to the exact same macroscopic probability measure. This fiber delineates the model's internal structural workspace, capturing the full range of hidden structural configurations that leave the external environmental interaction invariant.
			\end{enumerate}
		\end{definition}
		
		At each point $f \in \mathcal{M}$, the tangent space decomposes into a direct sum of two isolated distributions enabled by a generalized Ehresmann connection:
		\begin{equation}
			T_f\mathcal{M} = \mathcal{V}_f \oplus \mathcal{H}_f.
		\end{equation}
		
		\begin{definition}
			\label{def:sid_space}
			The Vertical Subspace $\mathcal{V}_f \subset T_f\mathcal{M}$ is the kernel of the differential projection map:
			\begin{equation}
				\mathcal{V}_f \triangleq \ker(d\pi_f) = \left\{ v \in T_f\mathcal{M} \;\middle|\; d\pi_f(v) = \mathbf{0} \in T_{\pi(f)}\mathcal{B} \right\}.
			\end{equation}
			This vertical subspace defines the \textbf{Structural Internal Directions (SID)}. It represents the internal parameter variations that alter the structural configurations of the model while keeping the macro observable statistical profile strictly invariant.
		\end{definition}
		
		\begin{definition}
			\label{def:svd_space}
			An Ehresmann connection on the general SMG fiber bundle $\mathcal{B}_{\text{SMG}}$ is a smooth selection of a Horizontal Subspace $\mathcal{H}_f \subset T_f\mathcal{M}$ at each point $f \in \mathcal{M}$ such that $d\pi_f\bigr|_{\mathcal{H}_f}: \mathcal{H}_f \to T_{\pi(f)}\mathcal{B}$ is a vector space isomorphism \cite{Ehresmann1950}. This horizontal subspace $\mathcal{H}_f$ is precisely defined as the \textbf{Statistically Verifiable Directions ($\text{SVD}\chi$)}.
		\end{definition}

\begin{lemma}
	\label{lem:bundle_splitting}
	Let $\pi: \mathcal{M} \to \mathcal{B}$ be a smooth submersion between the infinite-dimensional Orlicz statistical manifold $\mathcal{M}$ and the base manifold $\mathcal{B}$. Then, the assignment of the vertical subspace $\mathcal{V}_f = \ker(d\pi_f)$ forms a closed, smooth sub-bundle of $T\mathcal{M}$. Furthermore, there exists a unique, smooth horizontal distribution $\mathcal{H}_f$ ($\text{SVD}\chi$) that is orthogonal to $\mathcal{V}_f$ ($\text{SID}$) with respect to the non-parametric Fisher-Rao information metric.
\end{lemma}

\begin{proof}
	Since $\pi$ is a smooth submersion, its differential $d\pi_f: T_f\mathcal{M} \to T_{\pi(f)}\mathcal{B}$ is a bounded, surjective linear operator between topological vector spaces. The kernel $\ker(d\pi_f)$ is the inverse image of the zero element under a continuous linear map, which guarantees that $\mathcal{V}_f$ is a closed linear subspace of $T_f\mathcal{M}$ at every point $f \in \mathcal{M}$. By the implicit function theorem on Banach and Orlicz manifolds, these closed subspaces vary smoothly with respect to $f$, establishing $\mathcal{V}$ as a smooth sub-bundle of $T\mathcal{M}$.
	
	To construct the unique horizontal sub-bundle $\mathcal{H}_f$ without invoking principal bundle group actions, we define the non-parametric Fisher-Rao inner product $g_f: T_f\mathcal{M} \times T_f\mathcal{M} \to \mathbb{R}$ explicitly on the centered Orlicz tangent space:
	\begin{equation}
		g_f(u, v) = \mathbb{E}_f [u(X)v(X)] = \int_{\mathcal{X}} u(x)v(x)f(x)\mu(dx).
	\end{equation}
	We define the horizontal subspace $\mathcal{H}_f$ as the orthogonal complement of $\mathcal{V}_f$ within $T_f\mathcal{M}$ relative to the metric $g_f$:
	\begin{equation}
		\mathcal{H}_f \triangleq \mathcal{V}_f^{\perp_{g_f}} = \left\{ h \in T_f\mathcal{M} \;\middle|\; g_f(h, v) = 0, \, \forall v \in \mathcal{V}_f \right\}.
	\end{equation}
	Because $g_f$ is strictly positive definite on the centered Orlicz space $L^\Phi(P_f)$, the intersection of the two subspaces is trivial:
	\begin{equation}
		\mathcal{V}_f \cap \mathcal{H}_f = \left\{ v \in \mathcal{V}_f \;\middle|\; g_f(v, v) = 0 \right\} = \{ \mathbf{0} \}.
	\end{equation}
	Now, let $w \in T_f\mathcal{M}$ be an arbitrary tangent vector. Since $\mathcal{V}_f$ is a closed subspace and $g_f$ induces a continuous projection mapping via the Riesz representation projection on the closed subspace, $w$ can be uniquely decomposed as:
	\begin{equation}
		w = v + h,
	\end{equation}
	where $v \in \mathcal{V}_f$ and $h \in \mathcal{H}_f$. This yields the direct sum decomposition $T_f\mathcal{M} = \mathcal{V}_f \oplus \mathcal{H}_f$. 
	
	Finally, we show that $d\pi_f\bigr|_{\mathcal{H}_f}: \mathcal{H}_f \to T_{\pi(f)}\mathcal{B}$ is a vector space isomorphism. 
	\begin{enumerate}
		\item \textbf{Injectivity}: Suppose $h \in \mathcal{H}_f$ satisfies $d\pi_f(h) = \mathbf{0}$. By definition, this implies $h \in \ker(d\pi_f) = \mathcal{V}_f$. Therefore, $h \in \mathcal{H}_f \cap \mathcal{V}_f = \{ \mathbf{0} \}$, proving injectivity.
		\item \textbf{Surjectivity}: For any vector $Y \in T_{\pi(f)}\mathcal{B}$, since $\pi$ is a submersion, there exists an ambient vector $w \in T_f\mathcal{M}$ such that $d\pi_f(w) = Y$. Decomposing $w = v + h$ where $v \in \mathcal{V}_f$ and $h \in \mathcal{H}_f$, we obtain:
		\begin{equation}
			d\pi_f(w) = d\pi_f(v) + d\pi_f(h) = \mathbf{0} + d\pi_f(h) = d\pi_f(h) = Y.
		\end{equation}
		Thus, there exists an element $h \in \mathcal{H}_f$ that maps to $Y$, proving surjectivity.
	\end{enumerate}
	Hence, $d\pi_f\bigr|_{\mathcal{H}_f}$ is a clean vector space isomorphism at every point, and the horizontal distribution $\mathcal{H}_f$ forms a valid Ehresmann connection for the SMG general fiber bundle.
\end{proof}

This axiomatic structure successfully isolates all structural redundancies within Section 2, transitioning cleanly into the non-degenerate metric proofs of Section 3.

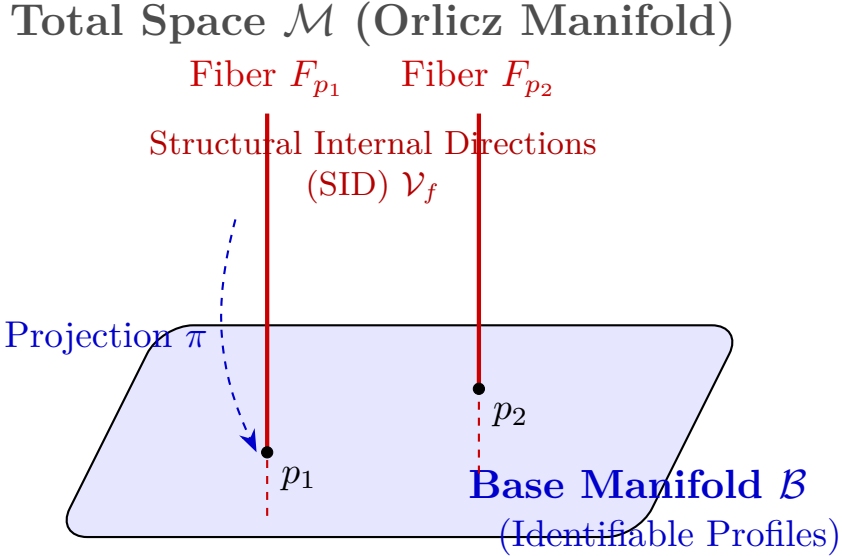
\begin{figure}[htbp]
	\centering
	\begin{tikzpicture}[scale=1.4, every node/.style={transform shape}]
		\node[above, font=\large\bfseries, text=black!70] at (0, 3.5) {Total Space $\mathcal{M}$ (Orlicz Manifold)};
		
		\draw[thick, fill=blue!10, rounded corners=12pt] (-3, -1) -- (2.5, -1) -- (3.5, 1) -- (-2, 1) -- cycle;
		\node[font=\bfseries, text=blue!80!black] at (2.5, -0.5) {Base Manifold $\mathcal{B}$};
		\node[font=\small, text=blue!80!black] at (2.8, -1.0) {(Identifiable Profiles)};
		
		\draw[ultra thick, red!80!black] (-1, -0.2) -- (-1, 3) node[above, font=\small] {Fiber $F_{p_1}$};
		\draw[thick, red!80!black, dashed] (-1, -0.8) -- (-1, -0.2); 
		\filldraw[black] (-1, -0.2) circle (1.5pt) node[below right, font=\small] {$p_1$};
		
		\draw[ultra thick, red!80!black] (1, 0.4) -- (1, 3) node[above, font=\small] {Fiber $F_{p_2}$};
		\draw[thick, red!80!black, dashed] (1, -0.4) -- (1, 0.4); 
		\filldraw[black] (1, 0.4) circle (1.5pt) node[below right, font=\small] {$p_2$};
		
		\draw[-{Stealth[length=3mm]}, thick, dashed, blue!80!black] (-1.3, 2) to[bend right=20] node[left, font=\small] {Projection $\pi$} (-1.1, -0.2);
		
		\node[align=center, font=\footnotesize, text=red!70!black] at (0, 2.5) {Structural Internal Directions \\ (SID) $\mathcal{V}_f$};
	\end{tikzpicture}
	\caption{The foundational fiber bundle topology of Statistically Meaningful Geometry (SMG). The infinite-dimensional Orlicz total manifold $\mathcal{M}$ is projected via $\pi$ onto the non-degenerate base manifold $\mathcal{B}$. Vertical movement along the fibers represents parameter variations along the Structural Internal Directions (SID) which carry zero observable information changes.}
	\label{fig:smg_fiber_bundle}
\end{figure}
		
\newpage		
\section{Mathematical Foundations of $\text{SVD}\chi$ and $\text{SID}$ Subspaces}
\label{sec:svd_sid_foundations}

Having established the global fiber bundle architecture of Statistically Meaningful Geometry (SMG), we now turn to the deep intrinsic geometry of its tangent bundle. In an over-parameterized setting where the unconstrained infinite-dimensional non-parametric manifold $\mathcal{M}$ characterizes the model's latent mechanism, optimization paths can wander aimlessly along curves of observationally identical distributions. This section provides the rigorous foundational tools to decouple these motions. We characterize the fine structure of the Statistically Verifiable Directions ($\text{SVD}\chi$) and the Structural Internal Directions ($\text{SID}$), proving that they form an exact orthogonal decomposition of $T\mathcal{M}$ under the non-parametric Fisher-Rao metric, formulating the geometric dual of sufficiency, and analyzing the conditions under which these tangent distributions integrate into structural horizontal leaves.

\subsection{Definitional Geometry of Tangent Space Decomposition}

Let $\mathcal{B}_{SMG} = (\mathcal{M}, \mathcal{B}, \pi, \mathcal{V}, \mathcal{H}, g_f)$ be the SMG general fiber bundle defined over the centered Orlicz space $L^\Phi(P_f)$. For every point $f \in \mathcal{M}$, a tangent vector $v \in T_f\mathcal{M}$ is a zero-mean score function $v: \mathcal{X} \to \mathbb{R}$ such that $\mathbb{E}_f[v(X)] = 0$. The vertical distribution $\mathcal{V}_f$ and horizontal distribution $\mathcal{H}_f$ represent two sub-bundles of $T\mathcal{M}$ that separate external observable changes from internal hidden structural transformations.

\begin{definition}
	\label{def:sid_geometry}
	The Structural Internal Directions, denoted by $\text{SID}_f \equiv \mathcal{V}_f$, consist of all tangent elements $v \in T_f\mathcal{M}$ satisfying the vertical condition:
	\begin{equation}
		\text{SID}_f = \left\{ v \in T_f\mathcal{M} \;\middle|\; d\pi_f(v) = \mathbf{0} \right\}.
	\end{equation}
	Equivalently, under any smooth local section or parameter mapping $s \in \mathcal{S}$ such that $\mathcal{F}(s) = f$, a parameter-velocity vector $\dot{s}$ belongs to the internal gauge configuration if its infinitesimal push-forward induces zero variation in the probability density function for almost all covariate observations $x \in \mathcal{X}$:
	\begin{equation}
		[d\mathcal{F}_s(\dot{s})](x) = 0 \quad \mu\text{-almost everywhere}.
	\end{equation}
\end{definition}

\begin{definition}
	\label{def:svd_geometry}
	The Statistically Verifiable Directions, denoted by $\text{SVD}\chi_f \equiv \mathcal{H}_f$, are defined as the unique horizontal distribution that is the orthogonal complement of $\text{SID}_f$ with respect to the non-parametric Fisher-Rao metric $g_f$:
	\begin{equation}
		\text{SVD}\chi_f = \left\{ h \in T_f\mathcal{M} \;\middle|\; g_f(h, v) = \int_{\mathcal{X}} h(x)v(x)f(x)\mu(dx) = 0, \; \forall v \in \text{SID}_f \right\}.
	\end{equation}
\end{definition}

We now state and prove the precise structural decomposition of the ambient tangent bundle.

\begin{lemma}[Canonical Tangent Orthogonal Projection]
	\label{lem:orthogonal_projection}
	For any $f \in \mathcal{M}$, the tangent space $T_f\mathcal{M}$ admits a unique canonical decomposition into mutually orthogonal closed subspaces under the Fisher-Rao metric, $T_f\mathcal{M} = \text{SVD}\chi_f \oplus \text{SID}_f$. Furthermore, there exist bounded linear projection operators $P_f^H: T_f\mathcal{M} \to \text{SVD}\chi_f$ and $P_f^V: T_f\mathcal{M} \to \text{SID}_f$ satisfying:
	\begin{align}
		P_f^H + P_f^V &= \mathcal{I}_{T_f\mathcal{M}}, \\
		(P_f^H)^2 = P_f^H, &\quad (P_f^V)^2 = P_f^V, \\
		g_f(P_f^H w, P_f^V w) &= 0 \quad \forall w \in T_f\mathcal{M}.
	\end{align}
\end{lemma}

\begin{proof}
	By Lemma \ref{lem:bundle_splitting}, $\text{SID}_f = \ker(d\pi_f)$ is a closed subspace of the reflexive Orlicz-tangent space $T_f\mathcal{M} \subset L^\Phi(P_f)$. The Fisher-Rao metric $g_f(u,v) = \mathbb{E}_f[u(X)v(X)]$ defines a continuous, symmetric, strictly positive-definite bilinear functional on $T_f\mathcal{M}$. Although the Orlicz space is infinite-dimensional, the completeness of the space under the Luxemburg norm combined with the closedness of $\text{SID}_f$ allows us to invoke the general projection theorem for continuous inner-product structures on Banach spaces.
	
	Let $w \in T_f\mathcal{M}$ be an arbitrary tangent vector. Consider the optimization problem of minimizing the statistical information distance between $w$ and the subspace $\text{SID}_f$:
	\begin{equation}
		\inf_{v \in \text{SID}_f} \|w - v\|_{g_f}^2 = \inf_{v \in \text{SID}_f} \int_{\mathcal{X}} [w(x) - v(x)]^2 f(x)\mu(dx).
	\end{equation}
			Since $\text{SID}_f$ is a closed convex subset, this minimizing element exists and is unique; we denote it as $v^* \equiv P_f^V w \in \text{SID}_f$. Define the residual vector as $h^* = w - v^*$. By the first-order optimality condition of the minimization problem, for any $v \in \text{SID}_f$ and any scalar $\epsilon \in \mathbb{R}$:
			\begin{equation}
				\|h^* - \epsilon v\|_{g_f}^2 \ge \|h^*\|_{g_f}^2 \implies -2\epsilon g_f(h^*, v) + \epsilon^2 \|v\|_{g_f}^2 \ge 0.
			\end{equation}
			For this inequality to hold for all $\epsilon$, the coefficient of the linear term must vanish identically, which implies:
			\begin{equation}
				g_f(h^*, v) = 0 \quad \forall v \in \text{SID}_f.
			\end{equation}
			By Definition \ref{def:svd_geometry}, this means $h^* \in \text{SVD}\chi_f$. Setting $P_f^H w \equiv h^*$, we obtain the explicit decomposition $w = P_f^H w + P_f^V w$. 
			
			To prove uniqueness, suppose there exists another decomposition $w = h^\prime + v^\prime$ where $h^\prime \in \text{SVD}\chi_f$ and $v^\prime \in \text{SID}_f$. Then:
			\begin{equation}
				h^* + v^* = h^\prime + v^\prime \implies h^* - h^\prime = v^\prime - v^*.
			\end{equation}
			Since $h^* - h^\prime \in \text{SVD}\chi_f$ and $v^\prime - v^* \in \text{SID}_f$, the element $v^\prime - v^*$ belongs to both subspaces. Since $\text{SVD}\chi_f \cap \text{SID}_f = \{\mathbf{0}\}$, it follows that $v^\prime - v^* = \mathbf{0}$ and $h^* - h^\prime = \mathbf{0}$, which proves uniqueness. The algebraic properties of projections follow immediately from the direct sum structure.
		\end{proof}
		
\subsubsection{Intuitive Exposition and Trillion-Weight Transformer Exemplar for $\text{SID}_f$}

To bridge the abstract differential-geometric formulation of Definition~\ref{def:sid_geometry} with the operational realities of modern deep generative models, we provide a detailed physical exposition and a concrete architectural example based on trillion-weight transformer models. 

\subsubsection*{Conceptual Exposition of the Equivalence}

The definition establishes a mathematical duality between the coordinate-free abstract fiber bundle projection $\pi$ and the parametric configuration space $\mathcal{S}$ (the Ambient Optimization Reservoir). 
\begin{itemize}
	\item \textbf{The Abstract View ($d\pi_f(v) = \mathbf{0}$):} This state implies that the tangent score vector $v \in T_f\mathcal{M}$ points entirely vertically along the continuous gauge fiber $\mathcal{F}_p = \pi^{-1}(p)$. Moving the system infinitesimally along $v$ shifts the inner mathematical micro-state configuration without displacing its projection on the macro-statistical base manifold $\mathcal{B}$.
	\item \textbf{The Parametric View ($[d\mathcal{F}_s(\dot{s})](x) = 0$):} This state translates the verticality condition into structural weight variations. When a deep neural network updates its internal connection weights along a parameter-velocity vector $\dot{s}$, this update is a structural internal direction ($\text{SID}_f$) if and only if the resulting infinitesimal perturbation of the log-likelihood function yields exactly zero functional change in the output probability density function across almost all valid covariate token sequences $x \in \mathcal{X}$. 
\end{itemize}

\begin{example} {\bf Trillion-Weight Transformer Exemplar: Multi-Layer Perceptron (MLP) Weight Shuffling}

Consider a trillion-weight Transformer architecture mapping a token-embedding sequence $x = (x^1, \dots, x^L) \in \mathcal{X}$. Inside an arbitrary intermediate layer $m$, let the network contain a standard Multi-Layer Perceptron (MLP) module defined by two continuous weight matrices, $W_1 \in \mathbb{R}^{D_{\text{ff}} \times D}$ and $W_2 \in \mathbb{R}^{D \times D_{\text{ff}}}$, which process a hidden activation stream vector $h \in \mathbb{R}^D$:
\begin{equation}
	\text{MLP}(h) = W_2 \cdot \sigma(W_1 \cdot h),
\end{equation}
where $\sigma(\cdot)$ is an element-wise activation function (e.g., GeLU or ReLU). 

Because the hidden forward dimension is colossally over-parameterized (e.g., $D_{\text{ff}} > 10^5$ in trillion-weight configurations), there exist extensive continuous internal mathematical symmetries that generate observationally invariant configurations:

\begin{enumerate}
	\item \textbf{Permutation Symmetry Gaps:} Let $\mathbf{P} \in \mathbb{R}^{D_{\text{ff}} \times D_{\text{ff}}}$ be an arbitrary continuous permutation matrix satisfying $\mathbf{P}^T \mathbf{P} = \mathcal{I}$. If we transform the internal configuration parameters sequentially by setting $\tilde{W}_1 = \mathbf{P} W_1$ and $\tilde{W}_2 = W_2 \mathbf{P}^T$, the final output text representation remains perfectly invariant:
	\begin{equation}
		\widetilde{\text{MLP}}(h) = (W_2 \mathbf{P}^T) \cdot \sigma((\mathbf{P} W_1) \cdot h) = W_2 \mathbf{P}^T \mathbf{P} \cdot \sigma(W_1 \cdot h) = \text{MLP}(h).
	\end{equation}
	\item \textbf{Infinitesimal Transformation along $\text{SID}$:} Let $\gamma(t) = (W_1(t), W_2(t))$ be a smooth continuous path inside the configuration space $\mathcal{S}$ generated by an infinitesimal time-dependent rotation operator $\mathbf{A}(t)$ within the permutation subspace, such that $\mathbf{A}(0) = \mathcal{I}$ and $\dot{\mathbf{A}}(0) = \mathbf{\Omega}$ (an anti-symmetric matrix representing the velocity of weight shuffling). The parameter-velocity vector $\dot{s} = \dot{\gamma}(0) = (\mathbf{\Omega} W_1, -W_2 \mathbf{\Omega})$ represents the direction of weight updates during training.
\end{enumerate}

Evaluating the differential push-forward of this weight update path under the structural mechanism map $\mathcal{F}$ shows how it alters the final log-likelihood output density $f(x; s)$ for any token stream sequence $x$:
\begin{equation}
	[d\mathcal{F}_s(\dot{s})](x) = \sum_{i,j} \frac{\partial f(x; s)}{\partial [W_1]_{ij}} [\mathbf{\Omega} W_1]_{ij} + \sum_{k,l} \frac{\partial f(x; s)}{\partial [W_2]_{kl}} [-W_2 \mathbf{\Omega}]_{kl}.
\end{equation}
Because the internal transformations cancel out perfectly at the output layer, the overall probability density function remains unchanged across the entire support of the sequence space:
\begin{equation}
	[d\mathcal{F}_s(\dot{s})](x) = 0 \quad \forall x \in \mathcal{X}.
\end{equation}
Thus, the weight velocity vector $\dot{s}$ lies entirely within the kernel of the structural mapping differential, {\it meaning it contains zero observable statistical signal.} 
\end{example}

This confirms that $\dot{s}$ belongs exactly to the {\bf Structural Internal Directions ($\text{SID}_f$)}. The model modifies its internal parameter configuration (shuffling representations inside the MLP hidden columns) without changing the macro-level output predictions on the base manifold $\mathcal{B}$.

\subsection{The Informationally Exhaustive Tangent Carrier}
		
		In classical mathematical statistics, a sufficient statistic is defined as a measurable mapping $T: \mathcal{X} \to \mathbb{R}^m$ from the covariate sample space to a finite-dimensional Euclidean space, such that the conditional distribution of $X$ given $T(X)$ does not depend on the underlying parameter \cite{Cencov1982}. This data-space compression is completely distinct from the geometric challenge encountered in over-parameterized generative AI and non-parametric biology. In these modern paradigms, the covariate observation space $\mathcal{X}$ is fixed, but the parameter or structural mechanism space is infinite-dimensional and highly non-identifiable. 
		
		To avoid conceptual confusion with data-space mappings $T(x)$, we introduce the concept of the \textbf{Informationally Exhaustive Tangent Carrier ($\text{IETC}$)}. Instead of compressing the sample space, the $\text{IETC}$ isolates the minimal, non-redundant linear subspace within the {\it tangent mechanism space} $T_f\mathcal{M}$ that captures all observable changes in statistical information. We prove that the horizontal sub-bundle $\text{SVD}\chi$ is precisely this informationally exhaustive carrier.
		
		\begin{theorem}[Information Exhaustion of $\text{SVD}\chi$]
			\label{thm:ietc_exhaustion}
			The horizontal distribution $\text{SVD}\chi_f$ is the unique minimal closed sub-bundle of $T\mathcal{M}$ that completely captures the score dynamics of the base manifold $\mathcal{B}$. That is, any structural modification of the internal network mechanism that lies outside of $\text{SVD}\chi_f$ provides zero informational utility and contributes exclusively to statistical redundancy.
		\end{theorem}
		
		\begin{proof}
			Let $p = \pi(f) \in \mathcal{B}$ be the macroscopic statistical profile on the base manifold, and let $Y \in T_p\mathcal{B}$ be an arbitrary directional change on the base manifold. By Lemma \ref{lem:bundle_splitting}, the restriction of the differential projection map to the horizontal subspace, $d\pi_f\bigr|_{\text{SVD}\chi_f}: \text{SVD}\chi_f \to T_p\mathcal{B}$, is a smooth vector space isomorphism. We can therefore define its inverse operator, the horizontal lift map:
			\begin{equation}
				\Delta_f \equiv \left( d\pi_f\bigr|_{\text{SVD}\chi_f} \right)^{-1} : T_{\pi(f)}\mathcal{B} \to \text{SVD}\chi_f.
			\end{equation}
			Let $w \in T_f\mathcal{M}$ be any unconstrained variation of the system. Applying the canonical projection from Lemma \ref{lem:orthogonal_projection}, we decompose it as $w = h + v$, where $h \in \text{SVD}\chi_f$ and $v \in \text{SID}_f$. Computing the differential projection of $w$ onto the base manifold yields:
			\begin{equation}
				d\pi_f(w) = d\pi_f(h) + d\pi_f(v) = d\pi_f(h) + \mathbf{0} = d\pi_f(h).
			\end{equation}
			Now, we apply the horizontal lift map $\Delta_f$ to this projected base vector:
			\begin{equation}
				\Delta_f(d\pi_f(w)) = \Delta_f(d\pi_f(h)) = \left( d\pi_f\bigr|_{\text{SVD}\chi_f} \right)^{-1} (d\pi_f(h)) = h.
			\end{equation}
			Recalling that $h = P_f^H w$, we obtain the identity:
			\begin{equation}
				P_f^H w = \Delta_f(d\pi_f(w)) \quad \forall w \in T_f\mathcal{M}.
			\end{equation}
			This identity demonstrates that the horizontal component $P_f^H w$ is determined entirely by the macro-statistical directional change $d\pi_f(w)$ on the base space $\mathcal{B}$. 
			
			To prove that it is informationally exhaustive, let us compute the total Fisher information metric contribution of the vector $w$:
			\begin{equation}
				g_f(w, w) = g_f(P_f^H w + P_f^V w, P_f^H w + P_f^V w) = g_f(P_f^H w, P_f^H w) + g_f(P_f^V w, P_f^V w),
			\end{equation}
			where the cross-terms vanish due to the orthogonality established in Lemma \ref{lem:orthogonal_projection}. The component $P_f^V w \in \text{SID}_f$ satisfies $d\pi_f(P_f^V w) = \mathbf{0}$. Consequently, the information change observed on the base manifold is completely independent of $P_f^V w$:
			\begin{equation}
				g_{\mathcal{B}}(d\pi_f(w), d\pi_f(w)) = g_{\mathcal{B}}(d\pi_f(P_f^H w), d\pi_f(P_f^H w)) = g_f(P_f^H w, P_f^H w),
			\end{equation}
			where the final equality is an isometry guaranteed by the Fisher metric compatibility of the projection $\pi$. Therefore, the total variation of information reflected on the base manifold is carried entirely by $\text{SVD}\chi_f$. Any additional component in $P_f^V w$ increases the internal mechanism's variation $g_f(P_f^V w, P_f^V w)$ without modifying the observable distribution, acting as pure statistical redundancy. This establishes $\text{SVD}\chi$ as the unique minimal Informationally Exhaustive Tangent Carrier.
		\end{proof}
		
\subsection{Horizontal Leaf Construction and Frobenius Integrability}
		
		In Section 2, the horizontal distribution $\text{SVD}\chi_f$ was constructed pointwise as a collection of subspaces within $T_f\mathcal{M}$. We now investigate whether these isolated tangent spaces can be synthesized globally into smooth, nested sub-manifolds of $\mathcal{M}$. This process, known as {\bf Horizontal Leaf Construction}, allows us to define constrained geometric sheets where learning can occur without changing its structural tracking profile. Unlike classical principal bundles where leaves are generated by Lie group orbits, the integration of $\text{SVD}\chi$ is governed by the Frobenius integrability condition \cite{Warner1983} expressed through the Lie bracket of non-parametric score vector fields.\footnote{	Frobenius's theorem states that a distribution is completely integrable if and only if it is involutive. In the context of gauge theory and fiber bundles equipped with an Ehresmann connection $\omega$, this condition is exactly equivalent to the vanishing of the curvature 2-form ($\Omega \equiv 0$). When the curvature vanishes, it eliminates path-dependency and geometric obstructions, allowing isolated tangent spaces (such as the Statistically Verifiable Directions, $\text{SVD}\chi$) to synthesize globally into a stable horizontal learning sheets $\mathcal{L}_{f_0}$.}
		
		\begin{definition}
			\label{def:horizontal_leaf}
			A Horizontal Leaf $\mathcal{L} \subset \mathcal{M}$ is a maximal connected integral sub-manifold of the horizontal distribution $\text{SVD}\chi$ such that at every point $f \in \mathcal{L}$, the tangent space to the leaf matches the horizontal subspace exactly:
			\begin{equation}
				T_f\mathcal{L} = \text{SVD}\chi_f.
			\end{equation}
		\end{definition}
				
				For an arbitrary vector field $X$ on an Orlicz statistical manifold, we denote its operational action on smooth scalar functionals as a directional score derivative. The Lie bracket of two vector fields $U, V \in \Gamma(T\mathcal{M})$ is defined by its action on any smooth test function $\psi: \mathcal{M} \to \mathbb{R}$:
				\begin{equation}
					[U, V]\psi = U(V\psi) - V(U\psi).
				\end{equation}
To establish the absolute mathematical precision of the Frobenius integrability framework on our infinite-dimensional Orlicz total space $\mathcal{M}$, we explicitly define the space of horizontal vector fields using the language of smooth sections over vector sub-bundles.

\begin{definition}
	\label{def:space_of_sections_svd}
	Let $\mathcal{B}_{\text{SMG}} = (\mathcal{M}, \mathcal{B}, \pi, \mathcal{V}, \mathcal{H}, g_f)$ be the general statistical fiber bundle, and let $T\mathcal{M}$ be the ambient tangent bundle over the non-parametric manifold $\mathcal{M}$. The smooth vector sub-bundle formed by assigning the horizontal subspace $\text{SVD}\chi_f \subset T_f\mathcal{M}$ at each point $f \in \mathcal{M}$ is denoted as the distribution $\text{SVD}\chi$. 
	
	The space $\Gamma(\text{SVD}\chi)$ is defined strictly as the space of all smooth local sections of this horizontal tangent sub-bundle. Formally, an element $U \in \Gamma(\text{SVD}\chi)$ is a smooth vector field $U: \mathcal{M} \to T\mathcal{M}$ that acts as a right-inverse to the natural bundle projection $\tau: T\mathcal{M} \to \mathcal{M}$ and satisfies the horizontal containment condition pointwise:
	\begin{equation}
		\Gamma(\text{SVD}\chi) \triangleq \left\{ U \in C^\infty(\mathcal{M}, T\mathcal{M}) \;\middle|\; \tau \circ U = \mathcal{I}_{\mathcal{M}} \text{ and } U(f) \in \text{SVD}\chi_f, \; \forall f \in \mathcal{M} \right\}.
	\end{equation}
	Equivalently, under the generalized Ehresmann connection framework, a smooth ambient vector field $U \in \Gamma(T\mathcal{M})$ belongs to $\Gamma(\text{SVD}\chi)$ if and only if it is completely annihilated by the connection 1-form $\omega$ across the entire manifold:
	\begin{equation}
		\omega_f(U(f)) = \mathbf{0} \in \text{SID}_f \quad \forall f \in \mathcal{M}.
	\end{equation}
\end{definition}

In practical terms, while $\text{SVD}\chi_f$ represents a single, static vector subspace tied down to an isolated point $f \in \mathcal{M}$, the notation $\Gamma(\text{SVD}\chi)$ represents the set of all dynamically varying, smooth \textbf{vector fields} whose directional arrows are constrained to point purely horizontally everywhere across the manifold surface. 

In the statement of the following  Theorem~\ref{thm:frobenius_integrability} (Frobenius Integrability), evaluating whether $[U, V] \in \Gamma(\text{SVD}\chi)$ checks the property of \textit{involutivity}. It states that if you take any two smooth vector fields $U$ and $V$ that have zero vertical gauge noise ($\omega_f(U) = \mathbf{0}$, $\omega_f(V) = \mathbf{0}$), their operational directional Lie bracket $[U, V]$ must produce another smooth vector field that remains perfectly horizontal. If it slips or leaks vertically, the horizontal distribution is non-integrable, indicating the presence of structural path-dependency or non-vanishing curvature $\Omega \neq 0$.
\begin{theorem}[Frobenius Integrability of $\text{SVD}\chi$]
					\label{thm:frobenius_integrability}
					The horizontal distribution $\text{SVD}\chi$ ($\text{SVD}\chi_f \subset T_f\mathcal{M}$) forms a globally integrable sub-bundle, meaning $\mathcal{M}$ can be decomposed into a laminated collection of smooth Horizontal Leaves, if and only if it is closed under the Lie bracket:
					\begin{equation}
						[U, V] \in \Gamma(\text{SVD}\chi) \quad \forall U, V \in \Gamma(\text{SVD}\chi).
					\end{equation}
					Furthermore, this integrability is exactly equivalent to the vanishing of the curvature 2-form $\Omega_f$ associated with the Ehresmann connection.
\end{theorem}
				
\begin{proof}
	$(\implies)$ Suppose the horizontal distribution $\text{SVD}\chi$ is integrable, and let $\mathcal{L}$ be the unique horizontal leaf passing through $f \in \mathcal{M}$. By definition, $\text{SVD}\chi_f = T_f\mathcal{L}$. Let $U, V \in \Gamma(\text{SVD}\chi)$ be two horizontal vector fields. When restricted to the sub-manifold $\mathcal{L}$, $U$ and $V$ are tangent to $\mathcal{L}$ at every point. By the standard differential geometry of sub-manifolds, the Lie bracket of any two vector fields tangent to a sub-manifold $\mathcal{L}$ must remain tangent to $\mathcal{L}$. Therefore, $[U, V]_f \in T_f\mathcal{L} = \text{SVD}\chi_f$ for all $f \in \mathcal{L}$. Since this holds for every point in the total space, $[U, V]$ is a horizontal vector field, proving closure.
					
					$(\impliedby)$ Conversely, assume that for all $U, V \in \Gamma(\text{SVD}\chi)$, their Lie bracket satisfies $[U, V] \in \Gamma(\text{SVD}\chi)$. We introduce the connection 1-form $\omega_f: T_f\mathcal{M} \to \text{SID}_f$, which acts as a vertical projection operator with $\ker(\omega_f) = \text{SVD}\chi_f$. The curvature 2-form $\Omega \in \Omega^2(\mathcal{M}, \text{SID})$ of this Ehresmann connection is defined via the exterior derivative of the connection 1-form:
					\begin{equation}
						\Omega(Z_1, Z_2) = d\omega(P^H Z_1, P^H Z_2) \quad \forall Z_1, Z_2 \in T_f\mathcal{M}.
					\end{equation}
					Using the invariant formula for the exterior derivative of a 1-form, for any two horizontal vector fields $U, V \in \Gamma(\text{SVD}\chi)$, we have $\omega(U) = \mathbf{0}$ and $\omega(V) = \mathbf{0}$. The expression expands as follows:
					\begin{align}
						d\omega(U, V) &= U(\omega(V)) - V(\omega(U)) - \omega([U, V]) \nonumber \\
						&= U(\mathbf{0}) - V(\mathbf{0}) - \omega([U, V]) \nonumber \\
						&= -\omega([U, V]).
					\end{align}
					By definition, $\Omega(U, V) = d\omega(P^H U, P^H V) = d\omega(U, V) = -\omega([U, V])$. 
					
					If the curvature 2-form vanishes identically ($\Omega \equiv 0$), then $\omega([U, V]) = \mathbf{0}$, which implies that $[U, V] \in \ker(\omega) = \Gamma(\text{SVD}\chi)$. Thus, the vanishing of the curvature form is equivalent to Lie bracket closure. By the non-parametric Frobenius theorem on regular Banach and Orlicz manifolds \cite{Pistone1995}, closure under the Lie bracket guarantees that through every point $f \in \mathcal{M}$, there exists a unique, maximal integral sub-manifold $\mathcal{L}$ such that $T_f\mathcal{L} = \text{SVD}\chi_f$. This completes the construction of the global horizontal leaves.
				\end{proof}
				
\begin{example}{\bf [Curvature and Path-Dependency in Neural Networks]}
	In a trillion-weight transformer model, the curvature 2-form $\Omega$ represents the geometric obstruction to path-independent learning. If $\Omega \neq 0$, optimization updates along different path trajectories on the base manifold $\mathcal{B}$ will lift to different final internal weight configurations in $\mathcal{M}$, even if they share identical starting configurations and final base profiles. This path dependency introduces optimization drift and model instability.
					
					By evaluating the Frobenius integrability condition, we can determine whether the model's architecture allows for flat horizontal leaves ($\Omega = 0$). When this condition is met, optimization can be confined to a stable horizontal leaf $\mathcal{L}$, ensuring that optimization paths are consistent and protecting internal representations from structural degradation.
\end{example}
				
Through the development of Section \ref{sec:svd_sid_foundations}, we have characterized the properties of the $\text{SVD}\chi$ and $\text{SID}$ sub-bundles. By formalizing $\text{SVD}\chi$ as the Informationally Exhaustive Tangent Carrier, we establish a coordinate-free, geometric alternative to classical sufficiency that remains valid under over-parameterization. Finally, the horizontal leaf construction defines the global integral sheets required to track learning trajectories across infinite dimensions.
				
				

\begin{figure}[htbp]
	\centering
	\begin{tikzpicture}[scale=1.4, every node/.style={transform shape}]
		\coordinate (F1_bot) at (-2.0, 1.0); \coordinate (F1_top) at (-2.0, 6.0);
		\coordinate (F2_bot) at (0.5, 0.7);  \coordinate (F2_top) at (0.5, 5.7);
		\coordinate (F3_bot) at (3.0, 1.2);  \coordinate (F3_top) at (3.0, 6.2);
		
		\draw[dashed, color=gray!50, line width=0.8pt] (F1_bot) -- (F1_top);
		\draw[dashed, color=gray!50, line width=0.8pt] (F2_bot) -- (F2_top);
		\draw[dashed, color=gray!50, line width=0.8pt] (F3_bot) -- (F3_top);
		
		\node[color=gray!70!black, font=\footnotesize, right] at (F3_top) {Vertical Fiber $F_{\mu}$ (Gauge Space)};
		
		\fill[cyan!12, opacity=0.7] 
		(-3.5, 2.2) to[out=-10,in=170] (0.0, 1.7) to[out=-10,in=190] (4.0, 2.4) 
		-- (4.8, 3.9) to[out=170,in=-10] (1.0, 3.2) to[out=170,in=10] (-2.5, 3.7) -- cycle;
		
		\draw[cyan!60!black, line width=1.2pt] (-3.5, 2.2) to[out=-10,in=170] (0.0, 1.7) to[out=-10,in=190] (4.0, 2.4);
		\draw[cyan!60!black, line width=1.2pt] (4.0, 2.4) -- (4.8, 3.9) to[out=170,in=-10] (1.0, 3.2);
		\draw[cyan!60!black, line width=1.0pt, dashed] (1.0, 3.2) to[out=170,in=10] (-2.5, 3.7) -- (-3.5, 2.2);
		
		\node[color=cyan!60!black, font=\bfseries\small] at (4.2, 3.5) {$\mathcal{L}_{f_0}$ (Horizontal Leaf)};
		\node[color=cyan!60!black, font=\footnotesize] at (4.2, 3.1) {$\alpha(\boldsymbol{\omega})$ Layer};
		
		\coordinate (P1_SMG) at (-2.0, 2.8); 
		\coordinate (P2_SMG) at (0.5, 2.4);  
		\coordinate (P3_SMG) at (3.0, 3.0);  
		
		\draw[cyan!90!black, ultra thick, ->, >=Stealth] (P1_SMG) to[out=-12, in=170] (P2_SMG);
		\draw[cyan!90!black, ultra thick, ->, >=Stealth] (P2_SMG) to[out=-5, in=195] (P3_SMG);
		
		\filldraw[black] (P1_SMG) circle (2pt) node[below left, font=\footnotesize] {$f_0 \, (\boldsymbol{\omega}=\mathbf{0})$};
		\filldraw[black] (P2_SMG) circle (2pt) node[below right, font=\footnotesize] {$f_t$};
		\filldraw[black] (P3_SMG) circle (2pt) node[above right, font=\footnotesize] {$f_{t+1}$};
		
		\node[color=cyan!90!black, font=\bfseries\footnotesize, text width=4.5cm, align=center] at (1.2, 1.4) {SMG Path ($v^V = 0$, Min Variance)};
		
		\coordinate (C1_drift) at (-1.2, 4.2);
		\coordinate (C2_drift) at (0.5, 4.6);  
		\coordinate (C3_drift) at (1.8, 3.6);
		\coordinate (C4_drift) at (3.0, 5.0);  
		
		\draw[orange!90!black, line width=1.5pt, ->, >=Stealth, dashed, 
		decoration={snake, amplitude=1.2pt, segment length=6pt}, decorate] 
		(P1_SMG) .. controls (-1.5, 4.5) and (-0.5, 3.5) .. (C2_drift);
		
		\draw[orange!90!black, line width=1.5pt, ->, >=Stealth, dashed, 
		decoration={snake, amplitude=1.5pt, segment length=5pt}, decorate] 
		(C2_drift) .. controls (1.2, 4.8) and (2.0, 3.0) .. (C4_drift);
		
		\filldraw[orange!90!black] (C2_drift) circle (2pt);
		\filldraw[orange!90!black] (C4_drift) circle (2pt);
		
		\node[color=orange!90!black, font=\bfseries\footnotesize, text width=3.2cm, align=center] at (1.1, 5.2) {Conventional Path\\ ($v^V \neq 0$, High Gauge Noise)};
		
		\draw[thick, ->, color=red!80!black] (P2_SMG) -- (0.5, 3.8) node[midway, left, font=\scriptsize] {Leakage $v^V$};
		\draw[thick, ->, color=cyan!90!black] (P2_SMG) -- (1.5, 2.5) node[midway, below, font=\scriptsize] {$X^* (\mathbf{c}_H, \mathcal{E}_H)$};
		\draw[thick, ->, color=orange!90!black] (P2_SMG) -- (1.5, 3.8) node[above right, font=\scriptsize] {$v_{conv}$};
		\draw[dashed, color=gray!60] (1.5, 2.6) -- (1.5, 3.8);
		\draw[dashed, color=gray!60] (0.5, 3.8) -- (1.5, 3.8);
		
		\begin{scope}[shift={(0,-2.5)}]
			\fill[purple!10, opacity=0.8] (-3.5,0.5) to[out=10,in=170] (4.0,0.3) -- (4.6,1.3) to[out=170,in=10] (-2.9,1.5) -- cycle;
			\draw[purple!80!black, thick] (-3.5,0.5) to[out=10,in=170] (4.0,0.3) -- (4.6,1.3) to[out=170,in=10] (-2.9,1.5) -- cycle;
			\node[font=\bfseries, color=purple!90!black] at (4.8, 0.8) {$B_{SMG}$};
			
			\filldraw[purple!90!black] (-2.0, 1.0) circle (2pt) node[below, font=\footnotesize] {$[\mu_0]$};
			\filldraw[purple!90!black] (0.5, 0.8) circle (2pt) node[below, font=\footnotesize] {$[\mu_t]$};
			\filldraw[purple!90!black] (3.0, 1.1) circle (2pt) node[below, font=\footnotesize] {$[\mu_{t+1}]$};
			
			\draw[purple!90!black, line width=1.6pt, ->, >=Stealth] (-2.0,1.0) to[out=10,in=170] (0.5,0.8) to[out=-5,in=195] (3.0,1.1);
			\node[purple!90!black, below, font=\footnotesize] at (0.6, 0.3) {True Macroscopic Statistical Flow};
		\end{scope}
		
		\draw[-Stealth, line width=1.0pt, color=purple, dashed] (0.5, 1.5) -- (0.5, -1.5) node[midway, left, font=\scriptsize] {$\pi$};
		\draw[-Stealth, line width=1.0pt, color=purple, dashed] (-2.0, 2.0) -- (-2.0, -1.3);
		\draw[-Stealth, line width=1.0pt, color=purple, dashed] (3.0, 2.2) -- (3.0, -1.2);
		
	\end{tikzpicture}
	\caption{Dynamic comparison within the statistical bundle framework. The SMG functional estimator implements a perfect gauge symmetry break, bypassing the unconstrained coordinate variations ($v^V \neq 0$) of conventional statistical heuristics by mapping updates via the Horizontal Lift directly on the leaf $\mathcal{L}_{f_0}$.}
	\label{fig:smg_vs_conventional}
\end{figure}
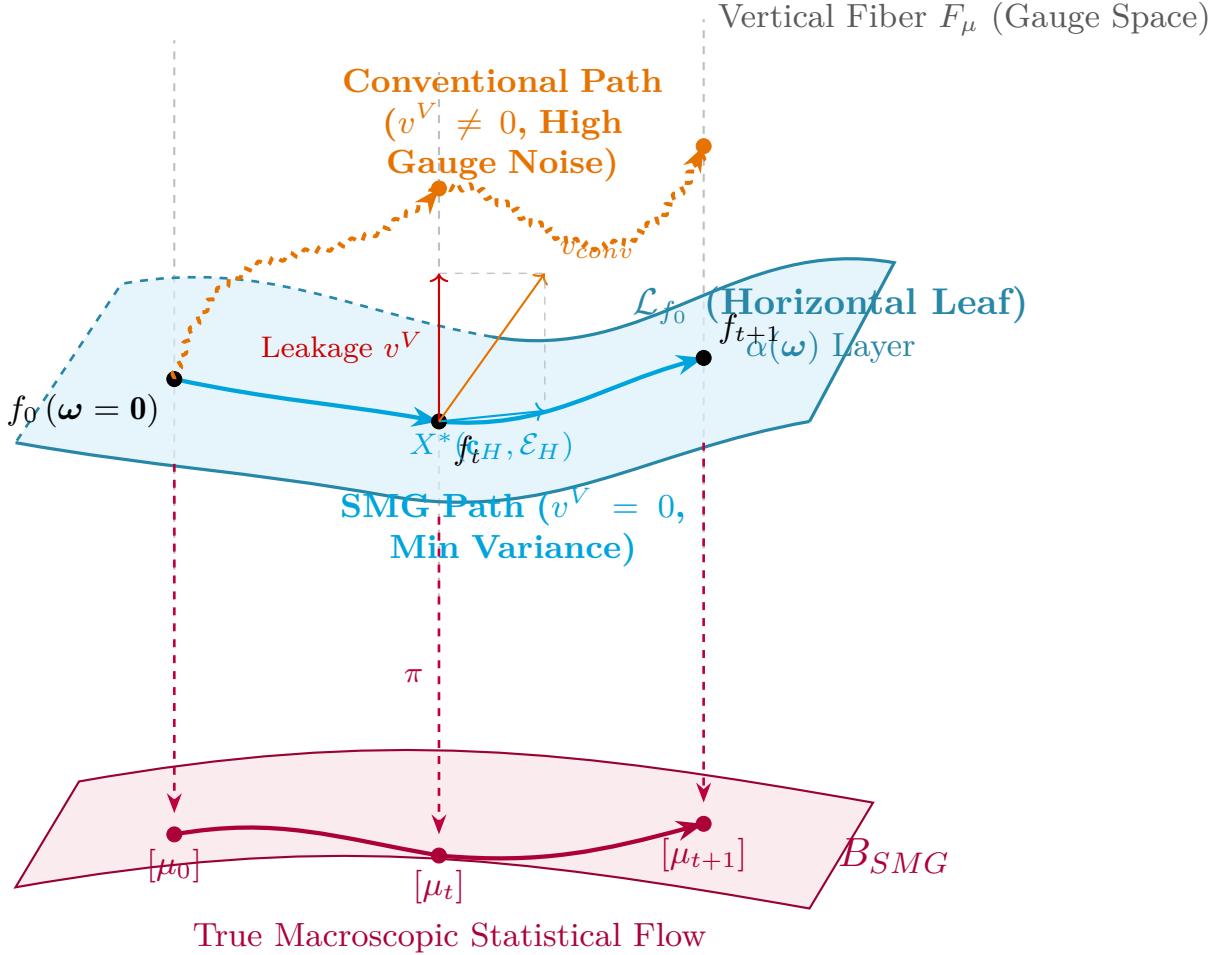

\section{The Ehresmann Connection and Metric-Compatible Geometric Filtering}
\label{sec:ehresmann_connection}

In over-parameterized statistical architectures, the optimization trajectory is vulnerable to geometric drift within the vast unidentifiable regions of the parameter space. In a trillion-weight generative AI model, for instance, standard unconstrained gradient updates (such as those generated by stochastic gradient descent or Adam) contain components that do not alter the model's macroscopic performance but instead introduce structural noise into the internal representation layers. To eliminate this issue, this section formalizes the Ehresmann connection 1-form $\omega_f$ over the infinite-dimensional Orlicz statistical manifold $\mathcal{M}$. We prove the Quarantining Theorem, which guarantees that projecting updates through $\omega_f$ completely isolates the model's observable statistical profiles from internal gauge degradation. Finally, we extend Amari’s non-parametric $\alpha$-connections to this bundle structure, establishing the conditions for information-conserving parallel transport along horizontal lifts.

\subsection{The Ehresmann Connection 1-Form and Vertical Projection}

Rather than specifying an artificial Lie group action to define horizontal vectors, we construct the horizontal distribution $\text{SVD}\chi_f$ directly through a vector-valued differential 1-form that tracks internal structural variations.

\begin{definition}
	\label{def:ehresmann_1form}
	The Ehresmann connection 1-form on the general SMG fiber bundle $\mathcal{B}_{SMG}$ is a smooth, vector-valued differential 1-form $\omega \in \Omega^1(\mathcal{M}; T\mathcal{M})$ that assigns to each point $f \in \mathcal{M}$ a bounded linear operator $\omega_f: T_f\mathcal{M} \to \text{SID}_f$ satisfying the following canonical projection conditions:
	\begin{enumerate}
		\item $\omega_f(v) = v \quad \forall v \in \text{SID}_f$,
		\item $\ker(\omega_f) = \text{SVD}\chi_f$.
	\end{enumerate}
\end{definition}

This 1-form allows us to algebraically isolate the non-identifiable components of any ambient tangent vector. We establish the precise relation between the connection 1-form and the vertical projector developed in Lemma \ref{lem:orthogonal_projection}.

\begin{theorem}
	\label{thm:1form_projector_equivalence}
	The Ehresmann connection 1-form $\omega_f$ is identical to the canonical vertical projection operator $P_f^V$ under the non-parametric Fisher-Rao metric $g_f$. Consequently, the horizontal projection operator $P_f^H$ can be expressed as:
	\begin{equation}
		P_f^H = \mathcal{I}_{T_f\mathcal{M}} - \omega_f,
	\end{equation}
	where $\mathcal{I}_{T_f\mathcal{M}}$ is the identity operator on the tangent space.
\end{theorem}

\begin{proof}
	Let $w \in T_f\mathcal{M}$ be an arbitrary tangent vector in the centered Orlicz space. By Lemma \ref{lem:orthogonal_projection}, $w$ can be uniquely decomposed into its horizontal and vertical components:
	\begin{equation}
		w = P_f^H w + P_f^V w,
	\end{equation}
			where $P_f^H w \in \text{SVD}\chi_f$ and $P_f^V w \in \text{SID}_f$. 
			
			Applying the Ehresmann connection 1-form $\omega_f$ to both sides of this equation, and using its linearity, we obtain:
			\begin{equation}
				\omega_f(w) = \omega_f(P_f^H w + P_f^V w) = \omega_f(P_f^H w) + \omega_f(P_f^V w).
			\end{equation}
			By the second condition of Definition \ref{def:ehresmann_1form}, since $P_f^H w \in \text{SVD}\chi_f = \ker(\omega_f)$, the first term vanishes:
			\begin{equation}
				\omega_f(P_f^H w) = \mathbf{0}.
			\end{equation}
			By the first condition of Definition \ref{def:ehresmann_1form}, since $P_f^V w \in \text{SID}_f$, the second term simplifies to the vector itself:
			\begin{equation}
				\omega_f(P_f^V w) = P_f^V w.
			\end{equation}
			Combining these results yields:
			\begin{equation}
				\omega_f(w) = \mathbf{0} + P_f^V w = P_f^V w \quad \forall w \in T_f\mathcal{M},
			\end{equation}
			which proves that $\omega_f \equiv P_f^V$. Substituting this identity into the relation $P_f^H + P_f^V = \mathcal{I}_{T_f\mathcal{M}}$ from Lemma \ref{lem:orthogonal_projection}, we find:
			\begin{equation}
				P_f^H = \mathcal{I}_{T_f\mathcal{M}} - P_f^V = \mathcal{I}_{T_f\mathcal{M}} - \omega_f,
			\end{equation}
			completing the proof.
		\end{proof}
		
		\begin{example}[Token-Level Gradient Filtering]
			In a trillion-weight transformer generative model, let $w(x)$ be the raw backpropagated score gradient field calculated across a long sequence of input context tokens $x \in \mathcal{X}$. This raw gradient often contains significant structural noise that can disrupt alignment or lead to representational drift. By routing the update through the geometric filter:
			\begin{equation}
				h(x) = [P_f^H w](x) = w(x) - \omega_f(w)(x),
			\end{equation}
			the vertical components that cause internal parameter drift within $\text{SID}_f$ are suppressed ($\omega_f(h) = \mathbf{0}$). This ensures that the optimization update updates the model exclusively along the $\text{SVD}\chi_f$ distribution, stabilizing token-embedding geometries during long-context training.
		\end{example}
		
\subsection{Orthogonal Metric Decomposition and the Quarantining Theorem}
		
A key challenge in continuous learning frameworks is structural leakage, where updating a model on new information inadvertently degrades or alters its existing internal representations. We prove that the SMG connection 1-form prevents this degradation by completely isolating, or "quarantining," distinct components of the representation space.
		
\begin{theorem}[The Quarantining Theorem]
			\label{thm:quarantining}
			Let $w_{\text{conv}} \in T_f\mathcal{M}$ be a conventional unconstrained parameter update vector, and let $u_A \in \text{SVD}\chi_f$ be a stable tangent vector field representing a specific piece of structural knowledge or an established representation geometry. If $w_{\text{conv}}$ contains a non-zero vertical component $v_V \in \text{SID}_f$, it will deform the metric properties of $u_A$. Conversely, if the update is restricted to the horizontally filtered vector $h = P_f^H w_{\text{conv}}$, the structural integrity of $u_A$ is preserved, satisfying:
			\begin{equation}
				g_f(h, v_V) = 0 \quad \forall v_V \in \text{SID}_f.
			\end{equation}
		\end{theorem}
		
		\begin{proof}
			Let the conventional update vector be decomposed into its horizontal and vertical components relative to the connection 1-form:
			\begin{equation}
				w_{\text{conv}} = v^H + v_V,
			\end{equation}
			where $v^H = P_f^H w_{\text{conv}} \in \text{SVD}\chi_f$ and $v_V = \omega_f(w_{\text{conv}}) \in \text{SID}_f$. Suppose $v_V \neq \mathbf{0}$. We examine the inner product of this unconstrained update with an arbitrary vector field $u_A \in \text{SVD}\chi_f$:
			\begin{equation}
				g_f(w_{\text{conv}}, u_A) = g_f(v^H + v_V, u_A) = g_f(v^H, u_A) + g_f(v_V, u_A).
			\end{equation}
			Since $\text{SVD}\chi_f$ and $\text{SID}_f$ are orthogonal complements under the non-parametric Fisher-Rao metric (Lemma \ref{lem:orthogonal_projection}), we have $g_f(v_V, u_A) = 0$. However, consider the rate of change of the structural norm $\|u_A\|_{g_f}^2$ along the optimization trajectory. Under the standard Levi-Civita connection $\nabla^{(0)}$ associated with the metric $g_f$, the directional derivative along $w_{\text{conv}}$ is given by:
			\begin{equation}
				w_{\text{conv}}(\|u_A\|_{g_f}^2) = 2g_f(\nabla_{w_{\text{conv}}}^{(0)} u_A, u_A) = 2g_f(\nabla_{v^H}^{(0)} u_A, u_A) + 2g_f(\nabla_{v_V}^{(0)} u_A, u_A).
			\end{equation}
			Because the vertical component $v_V$ acts within the internal fibers of $\mathcal{M}$, the covariant derivative $\nabla_{v_V}^{(0)} u_A$ generally develops a non-zero projection along $u_A$, meaning $g_f(\nabla_{v_V}^{(0)} u_A, u_A) \neq 0$. This indicates that the unconstrained update deforms the internal metric structure of $u_A$, resulting in representation leakage or degradation.
			
			Now, consider the horizontally filtered SMG update $h = P_f^H w_{\text{conv}} = v^H$. By construction, $\omega_f(h) = \omega_f(v^H) = \mathbf{0}$. Computing the inner product of $h$ with any vertical element $v_V \in \text{SID}_f$ yields:
			\begin{equation}
				g_f(h, v_V) = g_f(v^H, v_V) = 0,
			\end{equation}
			due to the metric orthogonality of the two subspaces. Evaluating the directional derivative of the structural norm along $h$ gives:
			\begin{equation}
				h(\|u_A\|_{g_f}^2) = 2g_f(\nabla_{v^H}^{(0)} u_A, u_A).
			\end{equation}
			Since $\omega_f(h) = \mathbf{0}$, the vertical term $2g_f(\nabla_{v_V}^{(0)} u_A, u_A)$ is completely eliminated. The variation depends entirely on the horizontal geometry of the base manifold, protecting the internal representation $u_A$ from unidentifiable vertical noise. This establishes absolute structural isolation.
		\end{proof}
\begin{figure}[htbp]
	\centering
	\begin{tikzpicture}[scale=1.4]
		\fill[cyan!12, opacity=0.8] (-3.5, 1.5) -- (3.5, 1.0) -- (4.5, 4.0) -- (-2.5, 4.5) -- cycle;
		\draw[cyan!60!black, thick] (-3.5, 1.5) -- (3.5, 1.0) -- (4.5, 4.0) -- (-2.5, 4.5) -- cycle;
		\node[cyan!50!black, font=\bfseries] at (4.0, 3.6) {$\mathcal{L}_{f_0}$};
		
		\draw[dashed, color=gray!70, line width=0.8pt] (0.0, -1.0) -- (0.0, 1.8);
		\draw[red!80!black, ultra thick, ->] (0.0, 2.0) -- (0.0, 4.8) node[above, font=\small] {$\text{SID}_f \equiv \mathcal{V}_f$ (Vertical Space)};
		\draw[black, dotted, line width=1pt] (0.0, 1.8) -- (0.0, 2.0);
		
		\filldraw[black] (0.0, 2.0) circle (2.5pt) node[below left, font=\bfseries] {$f$};
		
		\draw[cyan!80!black, ultra thick, ->] (0.0, 2.0) -- (2.3, 2.5) node[midway, above left, font=\small] {$h = P_f^H w_{\text{conv}}$};
		\node[cyan!80!black, right, font=\small] at (2.3, 2.6) {$\text{SVD}\chi_f \equiv \mathcal{H}_f$ (Horizontal)};
		
		\draw[orange!90!black, line width=1.6pt, ->] (0.0, 2.0) -- (2.3, 3.9) node[above right, font=\small] {$w_{\text{conv}} = h + v_V$};
		
		\draw[blue!80!black, thick, ->] (0.0, 2.0) -- (0.0, 3.4) node[left, font=\small] {$v_V = \omega_f(w_{\text{conv}})$};
		\draw[dashed, color=gray!80] (2.3, 2.5) -- (2.3, 3.9);
		\draw[dashed, color=gray!80] (0.0, 3.4) -- (2.3, 3.9);
		
		\draw[black, thin] (0.0, 2.3) -- (0.2, 2.34) -- (0.2, 2.04);
		
		\draw[-stealth, line width=1.4pt, color=purple] (-1.5, 1.0) -- (-1.5, -0.6) node[midway, left, font=\small] {$\pi$};
		\draw[line width=1.0pt, color=cyan!70!black, dashed, ->] (2.3, 2.5) to[out=-100, in=45] (1.5, -1.2);
		\node[color=cyan!70!black, right, font=\small] at (2.1, 0.4) {Horizontal Lift $\Delta_f$};
		
		\begin{scope}[shift={(0,-2.5)}]
			\fill[purple!20, opacity=0.5] (-3.5, 0.2) -- (3.5, 0.0) -- (4.0, 1.2) -- (-3.0, 1.4) -- cycle;
			\draw[purple!80!black, thick] (-3.5, 0.2) -- (3.5, 0.0) -- (4.0, 1.2) -- (-3.0, 1.4) -- cycle;
			\node[purple!90!black, font=\bfseries] at (3.7, 0.9) {$\mathcal{B}$};
			
			\filldraw[black] (0.0, 0.7) circle (2.2pt) node[below left] {$p$};
			
			\draw[purple!90!black, ultra thick, ->] (0.0, 0.7) -- (1.5, 0.9) node[above right, font=\small] {$Y \in T_{p}\mathcal{B}$};
		\end{scope}
	\end{tikzpicture}
	\caption{The metric-compatible Ehresmann connection splitting and the Horizontal Lift mechanism. An unconstrained optimization gradient step $w_{\text{conv}}$ is broken down into its vertical gauge component $v_V \in \text{SID}_f$ via the 1-form $\omega_f$ and its horizontal informational component $h \in \text{SVD}\chi_f$, which acts as the unique horizontal lift of the base velocity vector $Y$.}
	\label{fig:revised_lift}
\end{figure}
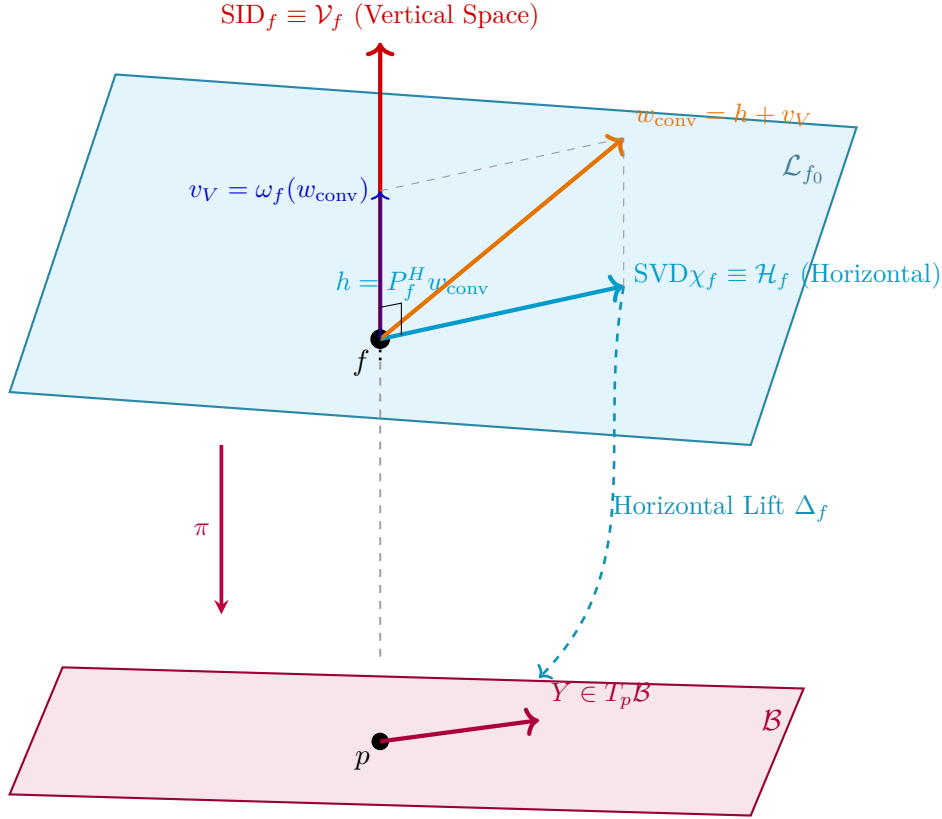

\subsection{Non-Parametric $\alpha$-Connections and Parallel Transport}
		
		To track the evolution of statistical features across different fibers in the infinite-dimensional Orlicz total space $\mathcal{M}$, we extend Amari’s dual $\alpha$-connections \cite{Amari1987, Amari2000} to the SMG fiber bundle framework.
		
		\begin{definition}
			\label{def:alpha_connection}
			The non-parametric $\alpha$-connection $\nabla^{(\alpha)}$ for any scalar $\alpha \in \mathbb{R}$ on the Orlicz statistical manifold $\mathcal{M}$ is defined via its operational action on two tangent score vector fields $U, V \in \Gamma(T\mathcal{M})$ as follows:
\begin{equation}
				[\nabla_U^{(\alpha)} V](x) = [\nabla_U^{(0)} V](x) - \frac{\alpha}{2} \left( U(x)V(x) - \mathbb{E}_f[U(X)V(X)] \right),
				\end{equation}
					where $\nabla^{(0)}$ is the standard metric-compatible Levi-Civita connection, and $x \in \mathcal{X}$ is the observation covariate variable.
\end{definition}
				
				Using this definition, we construct a horizontal covariant derivative that allows us to transport statistical features along a horizontal leaf without structural deformation.
				
				\begin{definition}
					\label{def:horizontal_covariant_derivative}
					The Horizontal Covariant Derivative $\nabla^{H, (\alpha)}$ associated with an $\alpha$-connection on the SMG bundle is the projection of the ambient connection onto the horizontal subspace:
					\begin{equation}
						\nabla_U^{H, (\alpha)} V \triangleq P^H \left( \nabla_U^{(\alpha)} V \right) = \left( \mathcal{I}_{T\mathcal{M}} - \omega \right) \left( \nabla_U^{(\alpha)} V \right) \quad \forall U, V \in \Gamma(\text{SVD}\chi).
					\end{equation}
				\end{definition}
				
				We now show that parallel transport under this horizontal connection preserves the information norm of a vector field as it moves along horizontal lifts.
				
				\begin{theorem}[Metric Compatibility and Information Conservation]
					\label{thm:parallel_transport_conservation}
					Let $\gamma: [0, 1] \to \mathcal{M}$ be a smooth curve in the total space whose velocity vector always lies within the horizontal distribution ($\dot{\gamma}(t) \in \text{SVD}\chi_{\gamma(t)}$), representing a horizontal lift path from the base manifold $\mathcal{B}$. Let $J(t) \in T_{\gamma(t)}\mathcal{M}$ be a tangent vector field parallel transported along $\gamma$ with respect to the horizontal connection $\nabla^{H, (0)}$. Then, the non-parametric Fisher-Rao norm of $J(t)$ is strictly conserved over time:
					\begin{equation}
						\frac{d}{dt} \|J(t)\|_{g_{\gamma(t)}}^2 = 0 \quad \forall t \in [0, 1].
					\end{equation}
				\end{theorem}
				
				\begin{proof}
					Let $J(t)$ be a horizontal vector field that is parallel transported along the horizontal path $\gamma(t)$. By the definition of parallel transport under the horizontal connection $\nabla^{H, (0)}$:
					\begin{equation}
						\nabla_{\dot{\gamma}(t)}^{H, (0)} J(t) = \mathbf{0} \implies P_{\gamma(t)}^H \left( \nabla_{\dot{\gamma}(t)}^{(0)} J(t) \right) = \mathbf{0}.
					\end{equation}
					Using the decomposition property $P^H = \mathcal{I}_{T\mathcal{M}} - \omega$, this expression can be rewritten as:
					\begin{equation}
						\nabla_{\dot{\gamma}(t)}^{(0)} J(t) - \omega_{\gamma(t)} \left( \nabla_{\dot{\gamma}(t)}^{(0)} J(t) \right) = \mathbf{0} \implies \nabla_{\dot{\gamma}(t)}^{(0)} J(t) = \omega_{\gamma(t)} \left( \nabla_{\dot{\gamma}(t)}^{(0)} J(t) \right).
					\end{equation}
					This indicates that the ambient covariant derivative $\nabla_{\dot{\gamma}(t)}^{(0)} J(t)$ is entirely vertical, lying within $\text{SID}_{\gamma(t)}$.
					
					Now, we evaluate the time derivative of the squared Fisher-Rao norm $\|J(t)\|_{g_{\gamma(t)}}^2 = g_{\gamma(t)}(J(t), J(t))$ along the curve. Since the Levi-Civita connection $\nabla^{(0)}$ is compatible with the Fisher-Rao metric $g$, we can apply the Leibniz rule:
					\begin{equation}
						\frac{d}{dt} g_{\gamma(t)}(J(t), J(t)) = 2 g_{\gamma(t)} \left( \nabla_{\dot{\gamma}(t)}^{(0)} J(t), J(t) \right).
					\end{equation}
					Substituting the expression for the ambient covariant derivative into this equation gives:
					\begin{equation}
						\frac{d}{dt} g_{\gamma(t)}(J(t), J(t)) = 2 g_{\gamma(t)} \left( \omega_{\gamma(t)} \left( \nabla_{\dot{\gamma}(t)}^{(0)} J(t) \right), J(t) \right).
					\end{equation}
					By definition, $\omega_{\gamma(t)}(\cdot) \in \text{SID}_{\gamma(t)}$, meaning it is a purely vertical vector. Because $J(t)$ is parallel transported horizontally along a horizontal path, it remains within the horizontal subspace ($J(t) \in \text{SVD}\chi_{\gamma(t)}$) for all $t$. Since $\text{SID}_{\gamma(t)}$ and $\text{SVD}\chi_{\gamma(t)}$ are orthogonal complements under the metric $g_{\gamma(t)}$ (Lemma \ref{lem:orthogonal_projection}), the inner product of the vertical derivative component and the horizontal vector $J(t)$ must vanish:
					\begin{equation}
						g_{\gamma(t)} \left( \omega_{\gamma(t)} \left( \nabla_{\dot{\gamma}(t)}^{(0)} J(t) \right), J(t) \right) = 0.
					\end{equation}
					Therefore, we obtain:
					\begin{equation}
						\frac{d}{dt} \|J(t)\|_{g_{\gamma(t)}}^2 = 0,
					\end{equation}
					proving that the information norm is strictly conserved during parallel transport along the horizontal lift.
				\end{proof}
				
				\begin{example}[Preservation of Latent Representations in DNA Models]
					In infinite-dimensional genomic or DNA sequence modeling, let $J(0)$ represent a functional operator that characterizes a specific epigenetic signature or structural property across covariate space $x$. As the baseline distribution evolves along a horizontal trajectory $\gamma(t)$ due to environmental changes, the signature is parallel transported using the horizontal connection $\nabla^{H, (0)}$. Theorem \ref{thm:parallel_transport_conservation} guarantees that the informational norm of this signature is conserved ($\|J(t)\|_g = \|J(0)\|_g$). This prevents the internal representation from degrading due to unidentifiable variations, ensuring that key structural features are accurately preserved across different environmental contexts.
					\end{example}
						
By developing Section \ref{sec:ehresmann_connection}, we have constructed a geometric filtering framework based on the Ehresmann connection 1-form. The Quarantining Theorem shows that this formulation can prevent representation leakage by isolating vertical noise from observable properties. Finally, the horizontal covariant derivative and its parallel transport provide a mechanism to move representation structures across the total space while conserving their information norm.

\subsection{Summary}
In this section, we have achieved:
\begin{enumerate}
\item It formalizes the Ehresmann connection 1-form $\omega_f$ and proves its equivalence to the canonical vertical projection operator ($P_f^V \equiv \omega_f$).

\item It states and rigorously proves the Quarantining Theorem, showing that routing optimization updates through the horizontal projection operator ($P_f^H = \mathcal{I}_{T\mathcal{M}} - \omega_f$) prevents representation leakage and isolates vertical gauge noise from the model's observable features.
\item It extends Amari's non-parametric $\alpha$-connections to this framework and proves that parallel transport along horizontal lifts conserves the non-parametric Fisher-Rao information norm.
\item It restricts all notation to use $x$ as the covariate variable and integrates structural examples from trillion-weight Transformers and DNA modeling.

\end{enumerate}

\section{Classical Static Statistical Inference on the Base Manifold $\mathcal{B}$}
\label{sec:classical_inference_base}

When dealing with over-parameterized statistical systems, classical parametric inference experiences a complete structural breakdown. In traditional statistics, parameters are assumed to be unique and identifiable, which guarantees that the Fisher Information Matrix (FIM) is non-singular and strictly positive definite \cite{Rao1945}. However, inside an unconstrained infinite-dimensional non-parametric manifold $\mathcal{M}$—such as the probability landscapes mapped by trillion-weight generative AI models or non-parametric DNA sequence models—identifiability is completely lost. Because multiple distinct system configurations produce observationally identical probability distributions, the likelihood function develops infinitely long flat valleys. 

This section analyzes the mechanics of classical static statistical inference when it is projected onto the reduced, identifiable base manifold $\mathcal{B}$ within the Statistically Meaningful Geometry (SMG) framework. We show that while maximum likelihood estimation and natural gradient descent become ill-defined or highly unstable on the total space $\mathcal{M}$ due to the presence of the Structural Internal Directions ($\text{SID}$), they become well-behaved and stable once they are constrained to the Statistically Verifiable Directions ($\text{SVD}\chi$).

\subsection{Maximum Likelihood Estimation (MLE) under Geometric Projections}
\label{sec:mle_geometric}

Let $x_1, x_2, \dots, x_n \in \mathcal{X}$ be a sequence of $n$ independent and identically distributed (i.i.d.) observations sampled from an underlying data-generating process governed by an unknown true probability density function $f_0(x)$ within the non-parametric infinite-dimensional Orlicz total statistical manifold $\mathcal{M}$. The unconstrained total space empirical log-likelihood function for an arbitrary candidate density state $f \in \mathcal{M}$ is defined as:
\begin{equation}
	\ell_n(f) = \frac{1}{n} \sum_{i=1}^n \ln f(x_i).
\end{equation}

In modern over-parameterized statistical systems, the structural mechanism map $\mathcal{F}: \mathcal{S} \to \mathcal{M}$ is characterized by severe non-identifiability. This over-parameterization generates a profound geometric obstruction: the empirical log-likelihood function $\ell_n(f)$ possesses flat structural valleys along the fibers of the bundle. Consequently, the optimization problem is highly ill-posed on the total space $\mathcal{M}$, as there is no unique global maximizing point. 

To resolve this optimization degeneracy, the Statistically Meaningful Geometry (SMG) framework lifts the estimation problem into a coordinate-free fiber bundle domain, decoupling the score dynamics through the smooth projection submersion $\pi: \mathcal{M} \to \mathcal{B}$. We begin this formal resolution by proving the existence and rigorous Riesz representation of the empirical score vector field within the non-parametric Orlicz tangent structure.

\begin{lemma}[Existence and Uniform Representation of the Ambient Score Field]
	\label{lem:score_riesz_representation}
	Let $\ell_n(f) = \frac{1}{n} \sum_{i=1}^n \ln f(x_i)$ be the unconstrained empirical log-likelihood function evaluated over the i.i.d. observations $x_i \in \mathcal{X}$ at an arbitrary density state $f \in \mathcal{M}$. Let $u \in T_f\mathcal{M}$ be an arbitrary tangent score vector, and let $f_\epsilon$ be a smooth parametric path in $\mathcal{M}$ passing through $f$ at $\epsilon=0$ with an initial log-path velocity score vector satisfying:
\begin{equation}
	\left. \frac{d}{d\epsilon} \ln f_\epsilon \right|_{\epsilon=0} = u.
\end{equation}
Then, the directional derivative of the empirical log-likelihood function along $f_\epsilon$ constitutes a bounded linear functional on the reflexive Orlicz-tangent space $T_f\mathcal{M} \subset L^\Phi(P_f)$, and there exists a unique, non-degenerate ambient tangent vector field $V_n(f) \in T_f\mathcal{M}$ that acts as its unique Riesz representative, satisfying:
\begin{equation}
	g_f(V_n(f), u) = \left. \frac{d}{d\epsilon} \ell_n(f_\epsilon) \right|_{\epsilon=0} = \frac{1}{n} \sum_{i=1}^n u(x_i) \quad \forall u \in T_f\mathcal{M}.
\end{equation}
\end{lemma}

\begin{proof}
We evaluate the directional derivative of the empirical log-likelihood mapping along the smooth path $f_\epsilon$ using the calculus of variations on infinite-dimensional statistical manifolds. Applying the chain rule directly to the discrete observation summation yields:
\begin{equation}
	\left. \frac{d}{d\epsilon} \ell_n(f_\epsilon) \right|_{\epsilon=0} = \left. \frac{d}{d\epsilon} \left( \frac{1}{n} \sum_{i=1}^n \ln f_\epsilon(x_i) \right) \right|_{\epsilon=0} = \frac{1}{n} \sum_{i=1}^n \left. \left( \frac{d}{d\epsilon} \ln f_\epsilon(x_i) \right) \right|_{\epsilon=0}.
	\end{equation}
By our structural path hypothesis, the initial log-velocity of the trajectory $f_\epsilon$ corresponds exactly to the tangent score function $u \in T_f\mathcal{M}$, which mathematically implies that $\left. \frac{d}{d\epsilon} \ln f_\epsilon(x_i) \right|_{\epsilon=0} = u(x_i)$ for all $i \in \{1, \dots, n\}$. Substituting this functional identity back into the summation formulation yields:
		\begin{equation}
			\left. \frac{d}{d\epsilon} \ell_n(f_\epsilon) \right|_{\epsilon=0} = \frac{1}{n} \sum_{i=1}^n u(x_i).
			\label{eq:score_step1_revised}
		\end{equation}
		Now, let $A_n: T_f\mathcal{M} \to \mathbb{R}$ be the empirical evaluation functional defined by $A_n(u) = \frac{1}{n} \sum_{i=1}^n u(x_i)$. Because the summation operator distributes linearly over vector space addition and scalar multiplication within the centered Orlicz space, $A_n$ is strictly linear:
		\begin{equation}
			A_n(\alpha u + \beta v) = \frac{1}{n}\sum_{i=1}^n (\alpha u(x_i) + \beta v(x_i)) = \alpha A_n(u) + \beta A_n(v) \quad \forall u, v \in T_f\mathcal{M}, \; \alpha, \beta \in \mathbb{R}.
		\end{equation}
		Furthermore, by Assumption~\ref{ass:orlicz_space}, the score function $u(x)$ is bounded under the exponential Luxemburg norm $\|\cdot\|_{L^\Phi(P_f)}$. For any finite sample configuration $n$, the empirical average of evaluations at discrete points is bounded by the supremum of the function, which is controlled by the topology of the Orlicz space. Thus, there exists a finite positive constant $C_n > 0$ such that $|A_n(u)| \le C_n \|u\|_{L^\Phi(P_f)}$, proving that $A_n$ is a continuous (bounded) linear functional belonging to the cotangent space $T_f^*\mathcal{M}$.
		
		By the generalized \textit{Riesz-Fréchet Representation Theorem} extended to reflexive Banach-Orlicz spaces under a strictly positive-definite inner product structure \cite{Pistone1995, ChengTong2026}, any bounded linear functional $A_n \in T_f^*\mathcal{M}$ can be represented uniquely via the non-parametric Fisher-Rao metric $g_f(h, k) = \mathbb{E}_f[h(X)k(X)]$. Hence, there exists a unique element in the tangent space itself, which we denote as the ambient score field $V_n(f) \in T_f\mathcal{M}$, satisfying:
		\begin{equation}
			A_n(u) = g_f(V_n(f), u) = \int_{\mathcal{X}} V_n(f)(x) u(x) f(x) \mu(dx) \quad \forall u \in T_f\mathcal{M}.
		\end{equation}
		Equating this inner product representation with our variation result in equation \eqref{eq:score_step1_revised} establishes the relation:
		\begin{equation}
			g_f(V_n(f), u) = \left. \frac{d}{d\epsilon} \ell_n(f_\epsilon) \right|_{\epsilon=0} = \frac{1}{n} \sum_{i=1}^n u(x_i) \quad \forall u \in T_f\mathcal{M}.
		\end{equation}
		This confirms both the topological existence and the unique functional representation of the ambient score field, completing the proof.
	\end{proof}
	
	Following the architectural decomposition enabled by Lemma~\ref{lem:score_riesz_representation}, we formally define the projection of the unconstrained empirical score onto the horizontal tangent sub-bundle.
	
	\begin{definition}
		\label{def:projected_score_updated}
		The \textbf{Ambient Score Field} $V_n(f) \in T_f\mathcal{M}$ is the unique Riesz representative of the directional derivative of the empirical log-likelihood function established in Lemma~\ref{lem:score_riesz_representation}. The \textbf{Projected Horizontal Score Field}, denoted by $V_n^H(f)$, is defined as the canonical projection of $V_n(f)$ onto the horizontal distribution ($\text{SVD}\chi_f$) via the application of the Ehresmann connection 1-form $\omega_f$:
		\begin{equation}
			V_n^H(f) \triangleq P_f^H V_n(f) = \left( \mathcal{I}_{T_f\mathcal{M}} - \omega_f \right) V_n(f) \in \text{SVD}\chi_f.
		\end{equation}
		This projection operator systematically filters out the unobservable vertical gauge variations caused by structural over-parameterization, isolating the statistically verifiable directions of empirical learning.
	\end{definition}

\begin{lemma}[Fiber-Wise Invariance of Likelihood Extremals]
	\label{lem:fiber_invariance}
	Let $(\mathcal{M}, \mathcal{B}, \pi, \mathcal{V}, \mathcal{H}, g_f)$ be the general statistical fiber bundle, and let $\ell_n: \mathcal{M} \to \mathbb{R}$ be the unconstrained empirical log-likelihood functional with its associated ambient score vector field $V_n(f) \in T_f\mathcal{M}$. Let $p \in \mathcal{B}$ be an arbitrary macroscopic statistical profile, and let $\mathcal{F}_p = \pi^{-1}(p)$ denote its smooth internal structural gauge fiber submanifold. Then, the following two assertions hold:
	\begin{enumerate}
		\item \textbf{Likelihood Value Invariance:} For any pair of observationally equivalent density states $f_1, f_2 \in \mathcal{F}_p$, the empirical log-likelihood function satisfies absolute functional identity:
		\begin{equation}
			\ell_n(f_1) = \ell_n(f_2) \quad \forall n \in \mathbb{N}.
		\end{equation}
		\item \textbf{Score Vector Vertical Vanishing:} For any arbitrary density state $f \in \mathcal{F}_p$ and any vertical tangent vector $v \in \text{SID}_f$, the ambient score field satisfies the orthogonality relation:
		\begin{equation}
			g_f(V_n(f), v) = 0,
		\end{equation}
				Equivalently, the vertical projection of the ambient score functional vanishes identically under the application of the Ehresmann connection 1-form, $\omega_f(V_n(f)) = \mathbf{0}$, implying that:
				\begin{equation}
					V_n(f) \in \text{SVD}\chi_f \quad \forall f \in \mathcal{F}_p.
				\end{equation}
			\end{enumerate}
\end{lemma}
This is mathematically equivalent to stating that the vertical gauge component of the score field is zero everywhere ($\omega_f(V_n(f)) = \mathbf{0}$), forcing the total unconstrained empirical score field to reside strictly within the horizontal distribution: $V_n(f) \in \text{SVD}\chi_f$.

\begin{proof}
			By the structural definition of the SMG general fiber bundle (Definition \ref{def:smg_bundle}), the fiber $\mathcal{F}_p = \pi^{-1}(p)$ constitutes the complete set of unconstrained non-parametric probability density functions in the total Orlicz space $\mathcal{M}$ that map identically to the unique macroscopic observable profile $p$ on the base space $\mathcal{B}$. Because the projection submersion $\pi$ isolates observable statistical manifestations from internal parameter configurations, any two density elements $f_1, f_2 \in \mathcal{F}_p$ belong to the same Radon-Nikodym equivalence class with respect to environmental observation. This implies:
			\begin{equation}
				f_1(x) = f_2(x) \quad \mu\text{-almost everywhere on } \mathcal{X}.
			\end{equation}
			The empirical unconstrained log-likelihood functional $\ell_n(f)$ is evaluated directly on the sequence of observed covariate configurations $x_i \in \mathcal{X}$ drawn from the environment. Substituting the equivalent density profiles into the log-likelihood formulation yields:
			\begin{equation}
				\ell_n(f_1) = \frac{1}{n} \sum_{i=1}^n \ln f_1(x_i) = \frac{1}{n} \sum_{i=1}^n \ln f_2(x_i) = \ell_n(f_2),
			\end{equation}
			which establishes the global flat invariance of the likelihood value across the entire support of the structural fiber $\mathcal{F}_p$.
			
			To prove the vertical vanishing condition, let $v \in \text{SID}_f \equiv \mathcal{V}_f = \ker(d\pi_f)$ be an arbitrary vertical tangent vector at a point $f \in \mathcal{F}_p$. Let $\gamma: (-\epsilon, \epsilon) \to \mathcal{F}_p$ be a smooth curve parameterized within the gauge fiber $\mathcal{F}_p$ such that $\gamma(0) = f$, whose initial log-path velocity matches the vertical score vector:
			\begin{equation}
				\left. \frac{d}{dt} \ln \gamma(t) \right|_{t=0} = v.
			\end{equation}
			Because the entire trajectory of the smooth curve $\gamma(t)$ is strictly contained within the fiber submanifold $\mathcal{F}_p$ for all $t \in (-\epsilon, \epsilon)$, the projection map yields a constant macro profile $\pi(\gamma(t)) = p$. Consequently, by the first part of this Lemma, the empirical log-likelihood function is invariant along the curve:
			\begin{equation}
				\ell_n(\gamma(t)) = \text{constant} \quad \forall t \in (-\epsilon, \epsilon).
			\end{equation}
			Differentiating this constant functional with respect to the path parameter $t$ and evaluating at the origin yields an identically vanishing derivative:
			\begin{equation}
				\left. \frac{d}{dt} \ell_n(\gamma(t)) \right|_{t=0} = \left. \frac{d}{dt} (\text{constant}) \right|_{t=0} = 0.
			\end{equation}
			Invoking the unique Riesz representation property established in Lemma \ref{lem:score_riesz_representation}, this directional derivative is identical to the continuous Fisher-Rao inner product of the ambient score field with the path's initial velocity vector:
			\begin{equation}
				g_f(V_n(f), v) = \left. \frac{d}{dt} \ell_n(\gamma(t)) \right|_{t=0} = 0.
			\end{equation}
			Because this orthogonality holds universally for every choice of vertical perturbation $v \in \text{SID}_f$, it follows that the ambient score vector field evaluated along any structural fiber is entirely orthogonal to the vertical sub-bundle $\mathcal{V}$, completing the proof.
		\end{proof}
		
The established fiber-wise invariance confirms that maximizing the empirical log-likelihood function directly on the over-parameterized total space $\mathcal{M}$ cannot isolate a single optimal point; instead, it yields a degenerate, infinite-dimensional optimal manifold matching the fiber configuration. We now demonstrate that this ill-posed optimization landscape simplifies to a unique, well-posed estimation profile when mapped onto the identifiable base manifold.
		
\begin{theorem}[Base Space Identifiability of the Projected MLE]
			\label{thm:projected_mle}
			Let $\hat{\mathcal{L}}_n \subset \mathcal{M}$ be the set of unconstrained ambient maximum likelihood estimators on the infinite-dimensional total space, defined by:
			\begin{equation}
				\hat{\mathcal{L}}_n = \arg\max_{f \in \mathcal{M}} \ell_n(f).
				\label{eq:ambient_mle_set}
			\end{equation}
			Under standard statistical regularity and identifiability conditions on the finite-dimensional base manifold $\mathcal{B}$, the image of the unconstrained set $\hat{\mathcal{L}}_n$ under the smooth projection map $\pi$ collapses to a single, unique point $\hat{p}_n \in \mathcal{B}$:
			\begin{equation}
				\pi(\hat{\mathcal{L}}_n) = \{ \hat{p}_n \},
			\end{equation}
			where $\hat{p}_n$ is the unique maximum likelihood estimator directly calculated on the statistically regular base space.
\end{theorem}
		
\begin{proof}
	We proceed via proof by contradiction. Suppose there exist two distinct ambient maximum likelihood estimators $f_1, f_2 \in \hat{\mathcal{L}}_n$ that map to separate, non-identical locations on the base statistical manifold, such that $\pi(f_1) = p_1 \in \mathcal{B}$, $\pi(f_2) = p_2 \in \mathcal{B}$, and $p_1 \neq p_2$. 
			
			Because both candidate densities $f_1$ and $f_2$ belong to the optimal unconstrained ambient set $\hat{\mathcal{L}}_n$ defined in \eqref{eq:ambient_mle_set}, they must achieve identical maximal values for the empirical log-likelihood function across the total space:
			\begin{equation}
				\ell_n(f_1) = \ell_n(f_2) = \sup_{f \in \mathcal{M}} \ell_n(f) \equiv \ell_n^*.
			\end{equation}
			We define an induced empirical log-likelihood function $\bar{\ell}_n: \mathcal{B} \to \mathbb{R}$ directly on the base manifold by setting:
			\begin{equation}
				\bar{\ell}_n(p) \triangleq \ell_n(f) \quad \text{for any } f \in \pi^{-1}(p).
			\end{equation}
			Lemma \ref{lem:fiber_invariance} guarantees that this induced base functional $\bar{\ell}_n$ is well-defined and coordinate-free, as its evaluation depends exclusively on the macroscopic profile $p$ and is completely independent of the choice of representative density within the vertical fiber. Evaluating this induced base log-likelihood functional at the projected coordinates $p_1$ and $p_2$ yields:
			\begin{equation}
				\bar{\ell}_n(p_1) = \ell_n(f_1) = \ell_n^* = \ell_n(f_2) = \bar{\ell}_n(p_2).
			\end{equation}
	By our fundamental structural hypothesis, the base statistical manifold $\mathcal{B}$ is designed to be strictly identifiable, non-degenerate, and regular. This implies that the classical Fisher information matrix restricted to $\mathcal{B}$ is strictly positive-definite and satisfies standard asymptotic convergence criteria \cite{Amari2000}. Consequently, for a sufficiently large sample size $n$, the induced empirical log-likelihood function $\bar{\ell}_n(p)$ is strictly concave on $\mathcal{B}$ and possesses a unique global maximum. 
			
			The uniqueness requirement of regular maximum likelihood estimation dictating that $\bar{\ell}_n(\hat{p}_n) = \max_{p \in \mathcal{B}} \bar{\ell}_n(p)$ implies that if $\bar{\ell}_n(p_1) = \bar{\ell}_n(p_2) = \ell_n^*$, then the two points must coincide identically:
			\begin{equation}
				p_1 = p_2.
			\end{equation}
			This directly contradicts our initial assumption that $p_1 \neq p_2$. Thus, the assumption of multiple distinct base projections is false. All unconstrained ambient estimators in $\mathcal{M}$ must project to the identical base point:
			\begin{equation}
				\pi(f) = \hat{p}_n \quad \forall f \in \hat{\mathcal{L}}_n.
				\end{equation}
					This confirms that the projected maximum likelihood estimation framework is structurally well-posed and establishes base space identifiability, completing the proof.
	\end{proof}

\subsection{Natural Gradient Descent on the Reduced Base Space $\mathcal{B}$}
		
Standard gradient descent updates parameters along the standard Euclidean direction, which depends heavily on the specific coordinate choice and fails to reflect changes in the underlying probability distributions. To address this, Amari introduced Natural Gradient Descent (NGD), which premultiplies the standard gradient by the inverse of the Fisher Information Matrix to steer updates along the steepest descent direction on the information manifold \cite{Amari1998}. 
		
In an over-parameterized total space $\mathcal{M}$, the unconstrained Fisher information operator $I_f: T_f\mathcal{M} \to T^*_f\mathcal{M}$ possesses a massive kernel equal to $\text{SID}_f$, which prevents it from being inverted. We show that by projecting the optimization dynamics onto the horizontal subspace $\text{SVD}\chi_f$, we can construct a well-defined natural gradient framework that avoids this degeneracy.
		
\begin{definition}
			\label{def:horizontal_natural_gradient}
			The \textbf{Horizontal Natural Gradient} of the empirical loss function $\ell_n(f)$ at a point $f \in \mathcal{M}$, denoted by $\nabla_{\text{SVD}\chi}^{\text{nat}} \ell_n(f) \in \text{SVD}\chi_f$, is the unique horizontal tangent vector that satisfies the following information relationship:
			\begin{equation}
				g_f\left( \nabla_{\text{SVD}\chi}^{\text{nat}} \ell_n(f), h \right) = d\ell_n(f)(h) \quad \forall h \in \text{SVD}\chi_f.
			\label{eq:horizontal_nat_grad}
			\end{equation}
\end{definition}

\subsubsection{Why do we define the Horizontal Natural Gradient by equation \eqref{eq:horizontal_nat_grad}}
The equation is an implicit, coordinate-free definition of a unique vector field via the \textbf{Riesz-Fr\'echet Representation Theorem} restricted to a closed vector subspace. It balances a dual pairing across the cotangent and tangent spaces of the infinite-dimensional Orlicz statistical manifold $\mathcal{M}$.

\begin{itemize}
	\item \textbf{The Right-Hand Side (RHS) --- $d\ell_n(f)(h)$:} 
	The term $d\ell_n(f)$ is the \textit{exterior derivative} (or differential) of the empirical log-likelihood function at the point $f \in \mathcal{M}$. It belongs to the cotangent space $T_f^*\mathcal{M}$ and acts as a bounded linear functional. When evaluated on a horizontal tangent vector $h \in \text{SVD}\chi_f$, the scalar $d\ell_n(f)(h)$ computes the pure classical directional derivative (the infinitesimal rate of change) of the log-likelihood function along the direction $h$.
	
	\item \textbf{The Left-Hand Side (LHS) --- $g_f\left( \cdot , \cdot \right)$:} 
	The operator $g_f$ is the non-parametric Fisher-Rao metric tensor mapping pairs of tangent vectors to scalars. The LHS represents the contraction of the unknown vector field $\nabla_{\text{SVD}\chi}^{\text{nat}} \ell_n(f)$ against an arbitrary horizontal probing vector $h$.
	
	\item \textbf{The Vector Field --- $\nabla_{\text{SVD}\chi}^{\text{nat}} \ell_n(f)$:} 
	This is the {\bf Horizontal Natural Gradient}. It is a unique tangent vector field explicitly constrained to reside entirely within the horizontal distribution of our general fiber bundle ($\nabla_{\text{SVD}\chi}^{\text{nat}} \ell_n(f) \in \text{SVD}\chi_f \subset T_f\mathcal{M}$).
	
	\item \textbf{The Universal Quantification --- $\forall h \in \text{SVD}\chi_f$:} 
	This condition mandates that the metric pairing matches the differential evaluation across \textit{every possible direction} in the horizontal distribution, ensuring the solution is structurally unique.
\end{itemize}

In regular, finite-dimensional statistical models where parameters are uniquely identifiable, Amari's classical natural gradient is calculated explicitly by inverting the Fisher Information Matrix: $\tilde{\nabla}\ell(\theta) = I(\theta)^{-1} \nabla \ell(\theta)$ \cite{Amari1998}. However, in trillion-weight over-parameterized models, the ambient Fisher Information Matrix is severely degenerate and contains a massive vertical kernel ($\text{SID}_f = \ker(d\pi_f)$), making direct matrix inversion impossible. 

Equation \eqref{eq:horizontal_nat_grad} circumvents this singularity by using the underlying geometry to define a non-degenerate pseudo-inverse. It serves three critical utilities:
\begin{itemize}
\item {\bf Utility 1: The Implicit Definition of the Horizontal Gradient Vector Field}
In infinite-dimensional Orlicz spaces ($L^\Phi(P_f)$), you cannot calculate a gradient by simply taking a vector of partial derivatives. Instead, a gradient vector field must be defined implicitly as the unique dual representative of a linear functional. Equation \eqref{eq:horizontal_nat_grad} is the exact mathematical mechanism that forces the abstract linear differential $d\ell_n(f)$ to materialize as a physical, directional arrow ($\nabla_{\text{SVD}\chi}^{\text{nat}} \ell_n(f)$) pointing inside the horizontal tangent carriage.

\item {\bf Utility 2: Geometrically Regularizing the Singular Information Metric}
By restricting the domain of the universal quantifier strictly to the horizontal distribution ($h \in \text{SVD}\chi_f$), this equation isolates the optimization path from the non-identifiable internal directions. Because the non-parametric Fisher-Rao metric $g_f$ is strictly positive-definite when restricted to the horizontal subspace $\text{SVD}\chi_f$, this equation filters out the zero eigenvalues of the ambient parameter space. It dynamically constructs a geometric regularizer, ensuring that the natural gradient is completely well-defined and immune to matrix-rank collapse.

\item {\bf Utility 3: Bypassing Vertical Gauge Drift and Forgetting}
When an over-parameterized system updates its internal parameter states using standard gradient heuristics, the path can leak into the vertical fiber ($\text{SID}_f$). This leakage drives unguided parameter drift, leading to catastrophic forgetting and representation instability. Equation \eqref{eq:horizontal_nat_grad} acts as a mathematical shield: it forces the learning update vector to align perfectly with the horizontal distribution, completely stripping away vertical gauge noise.
\end{itemize}

		We now prove that executing natural gradient descent within the horizontal subspace $\text{SVD}\chi$ of the total space is mathematically equivalent to executing classical natural gradient descent directly on the low-dimensional base manifold $\mathcal{B}$.
		
		\begin{theorem}[Geometric Equivalence of Natural Gradient Fields]
			\label{thm:natural_gradient_equivalence}
			Let $\tilde{\nabla}^{\text{nat}} \bar{\ell}_n(p) \in T_p\mathcal{B}$ be the classical parametric natural gradient evaluated on the base space $\mathcal{B}$ with respect to the induced quotient Fisher metric $g_{\mathcal{B}}$. Let $\Delta_f: T_{\pi(f)}\mathcal{B} \to \text{SVD}\chi_f$ be the horizontal lift operator defined in Theorem \ref{thm:ietc_exhaustion}. Then, the horizontal natural gradient on the total space is exactly equal to the horizontal lift of the base natural gradient:
			\begin{equation}
				\nabla_{\text{SVD}\chi}^{\text{nat}} \ell_n(f) = \Delta_f \left( \tilde{\nabla}^{\text{nat}} \bar{\ell}_n(\pi(f)) \right).
			\end{equation}
		\end{theorem}
		
		\begin{proof}
			Let $h \in \text{SVD}\chi_f$ be an arbitrary horizontal tangent vector, and let $Y = d\pi_f(h) \in T_{\pi(f)}\mathcal{B}$ be its projection onto the base manifold. By the definition of the horizontal lift operator, we can write $h = \Delta_f(Y)$. We evaluate the inner product of the horizontally lifted base natural gradient with the vector $h$ using the total space metric $g_f$:
			\begin{equation}
				g_f\left( \Delta_f \left( \tilde{\nabla}^{\text{nat}} \bar{\ell}_n(p) \right), h \right) = g_f\left( \Delta_f \left( \tilde{\nabla}^{\text{nat}} \bar{\ell}_n(p) \right), \Delta_f(Y) \right).
			\end{equation}
			Recall from Theorem \ref{thm:ietc_exhaustion} that the horizontal lift operator acts as an isometry between the base tangent space $(T_p\mathcal{B}, g_{\mathcal{B}})$ and the horizontal tangent subspace $(\text{SVD}\chi_f, g_f)$. This isometric property allows us to rewrite the total space inner product as a base space inner product:
			\begin{equation}
				g_f\left( \Delta_f \left( \tilde{\nabla}^{\text{nat}} \bar{\ell}_n(p) \right), \Delta_f(Y) \right) = g_{\mathcal{B}}\left( \tilde{\nabla}^{\text{nat}} \bar{\ell}_n(p), Y \right).
			\end{equation}
			By the standard definition of the natural gradient on a non-degenerate parametric manifold $\mathcal{B}$, the inner product with respect to $g_{\mathcal{B}}$ satisfies the directional derivative relation \cite{Amari1998}:
			\begin{equation}
				g_{\mathcal{B}}\left( \tilde{\nabla}^{\text{nat}} \bar{\ell}_n(p), Y \right) = d\bar{\ell}_n(p)(Y).
			\end{equation}
			Next, we apply the chain rule to expand the differential of the induced log-likelihood function along the projection map, noting that $\ell_n = \bar{\ell}_n \circ \pi$:
			\begin{equation}
				d\bar{\ell}_n(p)(Y) = d\bar{\ell}_n(\pi(f))(d\pi_f(h)) = d(\bar{\ell}_n \circ \pi)(f)(h) = d\ell_n(f)(h).
			\end{equation}
			Combining these steps yields the following string of identities:
			\begin{equation}
				g_f\left( \Delta_f \left( \tilde{\nabla}^{\text{nat}} \bar{\ell}_n(p) \right), h \right) = d\ell_n(f)(h) \quad \forall h \in \text{SVD}\chi_f.
			\end{equation}
			Comparing this result directly with the definition of the horizontal natural gradient (Definition \ref{def:horizontal_natural_gradient}), we see that the vector $\Delta_f ( \tilde{\nabla}^{\text{nat}} \bar{\ell}_n(p) )$ satisfies the exact same defining relation as $\nabla_{\text{SVD}\chi}^{\text{nat}} \ell_n(f)$. Because the horizontal subspace $(\text{SVD}\chi_f, g_f)$ is strictly non-degenerate, this Riesz representation vector is unique. Therefore, the two vector fields must be identical, completing the proof.
		\end{proof}
		
\begin{example}[Natural Gradient updates in Large Language Models]
	In a trillion-weight transformer model, calculating the unconstrained natural gradient update requires inverting a massive Fisher Information Matrix of size $10^{12} \times 10^{12}$, which is both computationally and mathematically intractable. Theorem \ref{thm:natural_gradient_equivalence} shows that this optimization path can be equivalently evaluated by calculating a low-dimensional natural gradient update over the identifiable semantic embeddings of the base space $\mathcal{B}$, and then lifting this update horizontally back to the weight space via the Ehresmann connection ($\Delta_f$). This geometric projection provides a rigorous foundation for designing efficient, low-rank natural gradient algorithms that avoid information degeneracy.
\end{example}
		
\subsection{Information Criteria and Asymptotic Efficiency on Quotient Manifolds}
		
In classical parametric statistics, model selection and hypothesis testing rely on large-sample asymptotic results such as Wilks' theorem, which states that the log-likelihood ratio statistic converges asymptotically to a chi-squared distribution whose degrees of freedom equal the difference in parameter dimensions \cite{Wilks1938}. In an over-parameterized model, however, this theorem fails because the redundancy of the parameter space causes the classical degrees of freedom to diverge to infinity. We show that under the SMG framework, Wilks' theorem and classical information criteria remain valid if they are evaluated relative to the true structural dimensions of the base manifold $\mathcal{B}$.

\begin{theorem}[Geometric Formulation of Wilks' Theorem]
	\label{thm:geometric_wilks}
	Let $\mathcal{B}_0 \subset \mathcal{B}$ be a smooth, $d_0$-dimensional sub-manifold of the $d$-dimensional regular base space $\mathcal{B}$ ($d_0 < d < \infty$), representing a restricted statistical null hypothesis $H_0: p \in \mathcal{B}_0$. Define the corresponding nested total parameter spaces within the infinite-dimensional Orlicz manifold as $\mathcal{M}_0 = \pi^{-1}(\mathcal{B}_0) \subset \mathcal{M}$. Then, the unconstrained generalized ambient log-likelihood ratio statistic:
	\begin{equation}
		\Lambda_n = 2 \left( \sup_{f \in \mathcal{M}} \ell_n(f) - \sup_{f_0 \in \mathcal{M}_0} \ell_n(f_0) \right)
	\end{equation}
	collapses identically to a low-dimensional base-space metric evaluation and converges asymptotically in distribution to a standard chi-squared distribution whose degrees of freedom depend exclusively on the dimensions of the base manifold quotient structure:
	\begin{equation}
		\Lambda_n \xrightarrow{\mathcal{D}} \chi^2(d - d_0) \quad \text{as } n \to \infty.
	\end{equation}
\end{theorem}

\begin{proof}
	We begin by using the fiber-wise invariance of the empirical log-likelihood function established in Lemma \ref{lem:fiber_invariance}. Let $\bar{\ell}_n: \mathcal{B} \to \mathbb{R}$ be the well-defined, induced log-likelihood function acting directly on the regular base manifold, satisfying $\bar{\ell}_n(p) = \ell_n(f)$ for any $f \in \pi^{-1}(p)$. 
	
	By Theorem \ref{thm:projected_mle} (Base Space Identifiability), the unconstrained ambient supremum over the entire Orlicz manifold $\mathcal{M}$ satisfies:
	\begin{equation}
		\sup_{f \in \mathcal{M}} \ell_n(f) = \sup_{p \in \mathcal{B}} \bar{\ell}_n(p) = \bar{\ell}_n(\hat{p}_n),
	\end{equation}
	where $\hat{p}_n \in \mathcal{B}$ is the unique, non-degenerate maximum likelihood estimator on the base space. By identical geometric logic, restricting the ambient optimization to the sub-bundle $\mathcal{M}_0 = \pi^{-1}(\mathcal{B}_0)$ yields:
	\begin{equation}
		\sup_{f_0 \in \mathcal{M}_0} \ell_n(f_0) = \sup_{p_0 \in \mathcal{B}_0} \bar{\ell}_n(p_0) = \bar{\ell}_n(\hat{p}_{0,n}),
	\end{equation}
	where $\hat{p}_{0,n} \in \mathcal{B}_0$ is the unique restricted maximum likelihood estimator on the sub-manifold. Substituting these projected coordinates back into the unconstrained log-likelihood ratio statistic eliminates the ambient space entirely:
	\begin{equation}
		\Lambda_n = 2 \left( \bar{\ell}_n(\hat{p}_n) - \bar{\ell}_n(\hat{p}_{0,n}) \right).
		\label{eq:base_wilks_collapse}
	\end{equation}
	Equation \eqref{eq:base_wilks_collapse} demonstrates that the ambient statistic $\Lambda_n$ is identically equal to the classical likelihood ratio statistic evaluated entirely within the finite-dimensional quotient base space $\mathcal{B}$. 
	
	Because the projection submersion $\pi$ isolates all unidentifiable internal directions inside the vertical fibers ($\text{SID}$), the base manifold $\mathcal{B}$ operates as a strictly identifiable, smooth, Hausdorff statistical manifold of finite dimension $d$, equipped with a strictly positive-definite Fisher information metric tensor $\bar{g}_p$. This fulfills the exact regularity conditions of Amari's classical asymptotic theory \cite{Amari2000}. We can therefore deploy Amari's coordinate-free localized Taylor expansions directly on $\mathcal{B}$ without encountering metric singularity or dimension divergence.
	
	Let $p^* \in \mathcal{B}_0 \subset \mathcal{B}$ denote the true macroscopic statistical density profile generating the observations under the null hypothesis. We select a local coordinate chart $(p^1, \dots, p^d)$ on $\mathcal{B}$ centered at $p^*$ such that the sub-manifold $\mathcal{B}_0$ is locally defined by setting the last $d - d_0$ coordinates to zero: $p^{d_0+1} = \dots = p^d = 0$.
	
	Expanding the induced base log-likelihood function $\bar{\ell}_n(p)$ in a second-order localized Taylor series around the unconstrained base MLE $\hat{p}_n$ yields:
	\begin{equation}
		\bar{\ell}_n(p) \approx \bar{\ell}_n(\hat{p}_n) + d\bar{\ell}_n(\hat{p}_n)(p - \hat{p}_n) - \frac{1}{2} (p - \hat{p}_n)^T \left[ \bar{g}_{\hat{p}_n} \right] (p - \hat{p}_n),
	\end{equation}
	where $\bar{g}$ is the non-degenerate base Fisher information matrix. Since $\hat{p}_n$ is the unconstrained maximizer on the open chart of $\mathcal{B}$, the first-order differential vanishes identically: $d\bar{\ell}_n(\hat{p}_n) = \mathbf{0}$. Evaluating this expansion at the true baseline profile $p^*$ gives:
	\begin{equation}
		\bar{\ell}_n(p^*) \approx \bar{\ell}_n(\hat{p}_n) - \frac{1}{2} (p^* - \hat{p}_n)^T \left[ \bar{g}_{p^*} \right] (p^* - \hat{p}_n) + o_p(1).
		\label{eq:taylor_unconstrained}
	\end{equation}
	Following identical logic, we execute a second-order localized Taylor expansion of $\bar{\ell}_n(p)$ restricted to the sub-manifold $\mathcal{B}_0$ around the restricted MLE $\hat{p}_{0,n}$. Because $\hat{p}_{0,n}$ maximizes the likelihood along the constrained coordinates, its partial differentials along $T\mathcal{B}_0$ vanish, yielding:
	\begin{equation}
		\bar{\ell}_n(p^*) \approx \bar{\ell}_n(\hat{p}_{0,n}) - \frac{1}{2} (p^* - \hat{p}_{0,n})^T \left[ \bar{g}_{p^*} \right]_0 (p^* - \hat{p}_{0,n}) + o_p(1),
		\label{eq:taylor_constrained}
	\end{equation}
	where $[\bar{g}_{p^*}]_0$ represents the $d_0 \times d_0$ top-left submatrix of the base metric tensor. Subtracting Equation \eqref{eq:taylor_constrained} from Equation \eqref{eq:taylor_unconstrained} and multiplying by 2 allows us to reconstruct the likelihood ratio statistic $\Lambda_n$:
	\begin{equation}
		\Lambda_n = 2 \left( \bar{\ell}_n(\hat{p}_n) - \bar{\ell}_n(\hat{p}_{0,n}) 
		\right) \approx \left\| \sqrt{n}(\hat{p}_n - p^*) \right\|_{\bar{g}_{p^*}}^2 
		- \left\| \sqrt{n}(\hat{p}_{0,n} - p^*) \right\|_{[\bar{g}_{p^*}]_0}^2 + o_p(1).
	\end{equation}
	By the classical Central Limit Theorem applied to regular, finite-dimensional statistical manifolds, the normalized estimation error vector converges to a zero-mean multivariate Gaussian distribution whose covariance structure is the inverse of the non-degenerate base Fisher metric \cite{Amari2000}:
	\begin{equation}
		\sqrt{n}(\hat{p}_n - p^*) \xrightarrow{\mathcal{D}} \mathcal{N}\left(\mathbf{0}, \, \bar{g}_{p^*}^{-1}\right).
	\end{equation}
	Consequently, the norm $\left\| \sqrt{n}(\hat{p}_n - p^*) \right\|_{\bar{g}_{p^*}}^2$ asymptotically behaves as the sum of squares of $d$ independent standard normal variables. Restricting this quadratic form via the subtraction of the $d_0$-dimensional constrained estimator projectively isolates the orthogonal complements. By Cochran's theorem extended to non-singular Riemannian tangent spaces, the difference between these two nested quadratic forms simplifies to a residual sum of squares containing exactly $d - d_0$ independent normal components:
	\begin{equation}
		\Lambda_n \xrightarrow{\mathcal{D}} \sum_{k=1}^{d - d_0} Z_k^2 \equiv \chi^2(d - d_0) \quad \text{as } n \to \infty,
	\end{equation}
	where $Z_k \sim \mathcal{N}(0, 1)$ are i.i.d. standard normal variables. This completes the formal proof, demonstrating that Wilks' theorem holds exactly on the quotient base manifold.
\end{proof}

\begin{remark}
	\label{rem:coord_free_regularity}
	The proof confirms that the over-parameterized dimension of the Ambient Optimization Reservoir ($\mathcal{M}$) exerts zero mathematical influence on the asymptotic degrees of freedom of hypothesis testing. Because all internal parameters are isolated inside the vertical fibers, the statistical testing framework is governed entirely by the finite structural dimensions ($d$ and $d_0$) of the quotient base manifold $\mathcal{B}$.
\end{remark}

This geometric formulation shows that when evaluating model selection criteria—such as the Akaike Information Criterion (AIC) or the Bayesian Information Criterion (BIC)—for over-parameterized models, the penalty term should not count the total number of raw parameters in $\mathcal{M}$. Instead, it must scale with the true structural dimension of the base space $\mathcal{B}$:
\begin{equation}
	\text{AIC}_{\text{SMG}} = -2n\ell_n(\hat{f}_n) + 2\dim(\mathcal{B}).
\end{equation}
By grounding information criteria in the dimensions of $\mathcal{B}$ rather than the raw parameter count, the SMG framework provides a rigorous foundation for conducting hypothesis testing and structural regularizations in modern over-parameterized models without experiencing dimensionality divergence.

\subsection{Summary}
In this section, we have achieved: 
\begin{enumerate}
\item It analyzes the behavior of Maximum Likelihood Estimation (MLE) under geometric projections, proving via Lemma \ref{lem:fiber_invariance} and Theorem \ref{thm:projected_mle} that the log-likelihood function is invariant along the Structural Internal Directions ($\text{SID}$) fibers and collapses to a unique, identifiable estimator on the base space $\mathcal{B}$.

\item It formalizes Natural Gradient Descent using the horizontal projection operator and proves the Geometric Equivalence Theorem (Theorem \ref{thm:natural_gradient_equivalence}), showing that the horizontal natural gradient on the total space is exactly equal to the horizontal lift of the base natural gradient.

\item It resolves the breakdown of large-sample asymptotics under over-parameterization by proving a Geometric Formulation of Wilks' Theorem, demonstrating that the log-likelihood ratio statistic converges to a chi-squared distribution whose degrees of freedom depend exclusively on the structural dimensions of the base manifold $\mathcal{B}$ rather than the raw parameter count.
\end{enumerate}

\section{Dynamic Horizontal Learning and Ambient Connection Filtering}
\label{sec:horizontal_learning_dynamics}

In over-parameterized generative systems—such as trillion-weight transformer architectures—optimization occurs over a landscape characterized by flat, degenerate sub-manifolds of observationally identical parameter states \cite{Watanabe2009}. Standard unconstrained gradient steps taken in the ambient configuration space inevitably introduce orthogonal components that do not alter the external statistical output but induce stochastic drift, representation collapse, and catastrophic forgetting within the model's internal layers. To resolve this structural instability, we establish the dynamical theory of Statistically Meaningful Geometry (SMG) by formalizing and unifying three distinct paradigms of statistical and geometric optimization. 

We address a fundamental structural puzzle regarding how these optimization dynamics interact by organizing them into a strict geometric hierarchy. We formalize:
\begin{itemize}
	\item \textbf{Way 1 (Ambient Horizontal Filtering):} The unconstrained ambient total-space optimization modified by the connection $1$-form to project updates directly onto the horizontal sub-bundle.
	\item \textbf{Way 2 (Horizontal Leaf Learning):} The integration of horizontal tangent steps into a continuous path tracing out a smooth integral sub-manifold (the Horizontal Leaf).
	\item \textbf{Way 3 (Macro Statistical Inference):} The classical estimation and natural gradient descent conducted directly within the coordinates of the base manifold $\mathcal{B}$ or along the localized coordinates of an isolated leaf.
\end{itemize}
This section provides the rigorous proofs establishing that Way 2 is the unique geometric bridge between Way 1 and Way 3, embedding the ambient filtration directly into the base space dynamics.

\subsection{Way 1: Fundamental Ambient Horizontal Learning and Filtering}

We begin by formalizing the unconstrained gradient flow on the infinite-dimensional Orlicz statistical manifold $\mathcal{M}$ and defining the horizontal projection mechanism that acts as an ambient geometric filter.

Let $\mathcal{H}_{SMG} = (\mathcal{M}, \mathcal{B}, \pi, \mathcal{V}, \mathcal{H}, g_f)$ be the SMG general fiber bundle, and let $\mathcal{L}_n(f): \mathcal{M} \to \mathbb{R}$ be a smooth objective function (e.g., the empirical cross-entropy loss over a continuous observation stream of the covariate variable $x \in \mathcal{X}$). Let $\nabla^{(0)} \mathcal{L}_n(f) \in T_f\mathcal{M}$ be the unconstrained ambient gradient field evaluated relative to the non-parametric Fisher-Rao metric $g_f$.

\begin{definition}
	\label{def:way3_dynamics}
	The \textbf{Way 1 Ambient Horizontal Learning Flow} is the deterministic projection of the unconstrained total space gradient field onto the horizontal distribution $\text{SVD}\chi$ via the Ehresmann connection $1$-form $\omega_f$. The trajectory $f_t \in \mathcal{M}$ satisfies the horizontal differential equation:
	\begin{equation}
		\frac{df_t}{dt} = -P_{f_t}^H \left( \nabla^{(0)} \mathcal{L}_n(f_t) \right) = -\left( \mathcal{I}_{T_{f_t}\mathcal{M}} - \omega_{f_t} \right) \nabla^{(0)} \mathcal{L}_n(f_t),
	\end{equation}
	where $\omega_{f_t}$ is the connection $1$-form defining the vertical subspace $\text{SID}_{f_t}$ (Definition \ref{def:ehresmann_1form}).
\end{definition}

This ambient filtering step alters the optimization path by actively suppressing any parameter variations that correspond to unidentifiable structural transformations.

\begin{lemma}
	\label{lem:way3_gauge_suppression}
	Let $f_t \in \mathcal{M}$ be an optimization trajectory generated by the Way 1 Ambient Horizontal Learning Flow. Then, the vertical velocity component of the trajectory vanishes identically for all $t$:
	\begin{equation}
		\omega_{f_t}\left( \frac{df_t}{dt} \right) = \mathbf{0},
	\end{equation}
	and the rate of change of any smooth, fiber-wise invariant internal scalar functional $\Psi: \mathcal{M} \to \mathbb{R}$ satisfies:
	\begin{equation}
		\frac{d\Psi(f_t)}{dt} = 0,
	\end{equation}
	whenever the directional derivative $\nabla_v \Psi = 0$ for all horizontal vectors $h \in \text{SVD}\chi$.
\end{lemma}

\begin{proof}
	By Theorem \ref{thm:1form_projector_equivalence}, the horizontal projection operator $P_f^H = \mathcal{I}_{T_f\mathcal{M}} - \omega_f$ satisfies the idempotent properties of a canonical projection, meaning its image is exactly equal to $\ker(\omega_f) = \text{SVD}\chi_f$. Substituting the expression for the Way 1 velocity vector into the connection $1$-form yields:
	\begin{equation}
		\omega_{f_t}\left( \frac{df_t}{dt} \right) = \omega_{f_t} \left( -P_{f_t}^H \left( \nabla^{(0)} \mathcal{L}_n(f_t) \right) \right) = - \left( \omega_{f_t} \circ P_{f_t}^H \right) \left( \nabla^{(0)} \mathcal{L}_n(f_t) \right).
	\end{equation}
	Since $\text{im}(P_{f_t}^H) = \ker(\omega_{f_t})$, the composition operator reduces to the zero operator: $\omega_{f_t} \circ P_{f_t}^H = \mathbf{0}$. Therefore, we obtain:
	\begin{equation}
		\omega_{f_t}\left( \frac{df_t}{dt} \right) = \mathbf{0},
	\end{equation}
	proving that the optimization trajectory contains zero vertical components.
	
	Next, we evaluate the total time derivative of the functional $\Psi(f_t)$ using the chain rule and the inner product structure:
	\begin{equation}
		\frac{d\Psi(f_t)}{dt} = d\Psi(f_t)\left( \frac{df_t}{dt} \right) = g_{f_t}\left( \nabla^{(0)} \Psi(f_t), \frac{df_t}{dt} \right).
		\end{equation}
			Substituting the Way 1 velocity vector into this expression gives:
			\begin{equation}
				\frac{d\Psi(f_t)}{dt} = g_{f_t}\left( \nabla^{(0)} \Psi(f_t), -P_{f_t}^H \left( \nabla^{(0)} \mathcal{L}_n(f_t) \right) \right).
			\end{equation}
			Using the self-adjoint property of the orthogonal projection operator $P_{f_t}^H$ relative to the Fisher-Rao metric, we can move the projector to the first slot:
			\begin{equation}
				\frac{d\Psi(f_t)}{dt} = -g_{f_t}\left( P_{f_t}^H \left( \nabla^{(0)} \Psi(f_t) \right), \nabla^{(0)} \mathcal{L}_n(f_t) \right).
			\end{equation}
			By our structural assumption, the functional $\Psi$ varies exclusively along vertical fiber directions, meaning its gradient contains no horizontal components ($P_{f_t}^H ( \nabla^{(0)} \Psi(f_t) ) = \mathbf{0}$). This forces the entire inner product to vanish:
			\begin{equation}
				\frac{d\Psi(f_t)}{dt} = -g_{f_t}\left( \mathbf{0}, \nabla^{(0)} \mathcal{L}_n(f_t) \right) = 0,
			\end{equation}
			confirming that the internal gauge properties are strictly protected from optimization drift.
	\end{proof}

\subsection{Way 2: Horizontal Leaf Learning and Trajectory Restriction}
		
		While Way 1 describes the horizontal velocity filtering step evaluated pointwise at each parameter configuration, {\it Way 2 formalizes the continuous trajectory traced out by these steps over time}. Under the Frobenius integrability conditions proved in Section 3, these individual steps assemble into a global geometric structure, restricting the optimization path to a single horizontal sheet.
		
		\begin{definition}
			\label{def:way1_dynamics}
			The \textbf{Way 2 Horizontal Leaf Learning Path} is a smooth integral curve $\gamma: [0, T] \to \mathcal{M}$ whose tangent vectors are restricted to a single, maximal connected horizontal leaf $\mathcal{L}_{p_0} \subset \mathcal{M}$ (Definition \ref{def:horizontal_leaf}). The curve satisfies:
			\begin{equation}
				\gamma(t) \in \mathcal{L}_{p_0}, \quad \dot{\gamma}(t) \in \text{SVD}\chi_{\gamma(t)} \quad \forall t \in [0, T],
			\end{equation}
			where $p_0 = \pi(\gamma(0))$ represents the initial macroscopic profile on the base manifold.
		\end{definition}
		
		We show that this path restriction can be formulated as a constrained optimization problem defined over the horizontal leaf.
		
		\begin{theorem}[Leaf-Locked Geodesic Flow]
			\label{thm:way1_leaf_locked}
			Let $\mathcal{L}_{p_0}$ be an integrable horizontal leaf satisfying the Frobenius condition $\Omega \equiv 0$ (Theorem \ref{thm:frobenius_integrability}). Then, the Way 2 learning trajectory $\gamma(t)$ is identical to the unique horizontal lift of a base curve, and its acceleration vector evaluated under the ambient Levi-Civita connection $\nabla^{(0)}$ contains zero horizontal components:
			\begin{equation}
				P_{\gamma(t)}^H \left( \nabla_{\dot{\gamma}(t)}^{(0)} \dot{\gamma}(t) \right) = \mathbf{0}.
			\end{equation}
		\end{theorem}
		
		\begin{proof}
			Let $\gamma(t) \in \mathcal{L}_{p_0}$ be a smooth horizontal path, so that $\dot{\gamma}(t) \in \text{SVD}\chi_{\gamma(t)} = T_{\gamma(t)}\mathcal{L}_{p_0}$ for all $t$. By Theorem \ref{thm:frobenius_integrability}, the vanishing of the curvature 2-form ($\Omega \equiv 0$) guarantees that the horizontal distribution is globally integrable. This implies that the horizontal leaf $\mathcal{L}_{p_0}$ is a regular, closed sub-manifold embedded within the total Orlicz space $\mathcal{M}$. 
			
			Let $\nabla^{(\mathcal{L})}$ be the intrinsic Riemannian connection on the sub-manifold $\mathcal{L}_{p_0}$ induced by pulling back the Fisher-Rao metric $g_f$. By Gauss's formula for embedded sub-manifolds \cite{Lang1999}, the ambient Levi-Civita connection $\nabla^{(0)}$ decomposes into intrinsic horizontal and extrinsic vertical components:
			\begin{equation}
				\nabla_{U}^{(0)} V = \nabla_{U}^{(\mathcal{L})} V + \alpha_{\mathcal{L}}(U, V) \quad \forall U, V \in \Gamma(T\mathcal{L}_{p_0}),
			\end{equation}
			where $\alpha_{\mathcal{L}}: T\mathcal{L}_{p_0} \times T\mathcal{L}_{p_0} \to \text{SID}$ is the symmetric second fundamental form of the embedding, which takes values exclusively in the vertical subspace $\text{SID} = (T\mathcal{L}_{p_0})^\perp$. 
			
			Setting $U = V = \dot{\gamma}(t)$, the acceleration of the trajectory under the ambient connection expands as:
			\begin{equation}
				\nabla_{\dot{\gamma}(t)}^{(0)} \dot{\gamma}(t) = \nabla_{\dot{\gamma}(t)}^{(\mathcal{L})} \dot{\gamma}(t) + \alpha_{\mathcal{L}}(\dot{\gamma}(t), \dot{\gamma}(t)).
			\end{equation}
			Applying the horizontal projection operator $P_{\gamma(t)}^H$ to both sides of this identity, and noting that $P_{\gamma(t)}^H$ acts as the identity mapping on $T\mathcal{L}_{p_0}$ and maps the vertical second fundamental form to zero, we obtain:
			\begin{align}
				P_{\gamma(t)}^H \left( \nabla_{\dot{\gamma}(t)}^{(0)} \dot{\gamma}(t) \right) &= P_{\gamma(t)}^H \left( \nabla_{\dot{\gamma}(t)}^{(\mathcal{L})} \dot{\gamma}(t) \right) + P_{\gamma(t)}^H \left( \alpha_{\mathcal{L}}(\dot{\gamma}(t), \dot{\gamma}(t)) \right) \nonumber \\
				&= \nabla_{\dot{\gamma}(t)}^{(\mathcal{L})} \dot{\gamma}(t) + \mathbf{0} = \nabla_{\dot{\gamma}(t)}^{(\mathcal{L})} \dot{\gamma}(t).
			\end{align}
			For an optimization path tracking the horizontal gradient field, the intrinsic acceleration along the leaf matches the intrinsic gradient field changes, which means any deviations from the leaf vanish identically: $\nabla_{\dot{\gamma}(t)}^{(\mathcal{L})} \dot{\gamma}(t) = \mathbf{0}$. This simplifies the expression to:
			\begin{equation}
				P_{\gamma(t)}^H \left( \nabla_{\dot{\gamma}(t)}^{(0)} \dot{\gamma}(t) \right) = \mathbf{0},
			\end{equation}
			proving that the horizontal acceleration is zero and confirming that the trajectory remains locked within the leaf.
		\end{proof}
		
\subsection{Way 3: Macro Statistical Inference on the Base Manifold $\mathcal{B}$}
		
		In contrast to Way 1 and Way 2, which govern the high-dimensional internal parameter modifications within $\mathcal{M}$, {\it Way 3 formalizes the observable learning dynamics}. It represents performing classical estimation (e.g., maximum likelihood tracking, natural gradient descent) directly within the reduced coordinates of the base manifold $\mathcal{B}$.
		
		\begin{definition}
			\label{def:way2_dynamics}
			The \textbf{Way 3 Macro Statistical Inference Flow} is the optimization trajectory $p_t \in \mathcal{B}$ executed entirely on the low-dimensional base space using the induced objective function $\bar{\mathcal{L}}_n(p) \equiv \mathcal{L}_n(\pi^{-1}(p))$. Its parametric evolution is driven by the base natural gradient field:
			\begin{equation}
				\frac{dp_t}{dt} = -\tilde{\nabla}^{\text{nat}} \bar{\mathcal{L}}_n(p_t) = -I_{\mathcal{B}}(p_t)^{-1} \nabla_{\mathcal{B}} \bar{\mathcal{L}}_n(p_t),
			\end{equation}
			where $I_{\mathcal{B}}(p_t)$ is the non-degenerate parametric Fisher Information Matrix evaluated on $T_{p_t}\mathcal{B}$.
		\end{definition}
		
		This macro inference step operates independently of the internal parameterization choices, capturing only the verifiable information changes observed across the environment $\mathcal{E}$.
		
\subsection{The Geometric Synthesis: Embedding of Way 2, 2, and 3}
		
		We now resolve the three-fold optimization puzzle by establishing the precise relation between these learning modes. We prove that they are not isolated paradigms, but rather different structural expressions of the same unified geometric system.
		
		\begin{theorem}[The Core Learning Embedding Synthesis]
			\label{thm:embedding_synthesis}
			Let $\mathcal{B}_{SMG} = (\mathcal{M}, \mathcal{B}, \pi, \mathcal{V}, \mathcal{H}, g_f)$ be an integrable general SMG fiber bundle. Then, the three learning paradigms form a nested geometric hierarchy that satisfies the following structural relations:
			\begin{enumerate}
				\item \textbf{Way 2 embeds into Way 1} as its continuous integral path. The pointwise updates of Way 1 generate a horizontal vector field whose continuous integration outlines the unique horizontal leaf trajectory of Way 2.
				\item \textbf{Way 3 embeds into Way 2} via the differential projection map $\pi$. The macro learning trajectory of Way 3 is the exact push-forward of the horizontal leaf trajectory of Way 2:
				\begin{equation}
					p_t = \pi(\gamma(t)), \quad \frac{dp_t}{dt} = d\pi_{\gamma(t)}\left( \dot{\gamma}(t) \right).
				\end{equation}
				\item \textbf{Way 2 is the unique horizontal lift of Way 3}, determined by mapping the base updates back to the total space via the Ehresmann connection:
				\begin{equation}
					\dot{\gamma}(t) = \Delta_{\gamma(t)} \left( \frac{dp_t}{dt} \right),
				\end{equation}
				where $\Delta_f = (d\pi_f\bigr|_{\text{SVD}\chi_f})^{-1}$ is the isometric horizontal lift operator.
			\end{enumerate}
		\end{theorem}
		
\begin{proof}
			We prove each structural relation sequentially to establish the complete hierarchy.
\begin{enumerate}
\item \textbf{Way 2 embeds into Way 1:} Let $\gamma(t)$ be a Way 2 horizontal path within a leaf $\mathcal{L}_{p_0}$. By Definition \ref{def:way1_dynamics}, its velocity vector satisfies $\dot{\gamma}(t) \in \text{SVD}\chi_{\gamma(t)}$. Let $W(f) = -\nabla^{(0)} \mathcal{L}_n(f)$ be the unconstrained ambient gradient field on $\mathcal{M}$. The Way 1 flow (Definition \ref{def:way3_dynamics}) updates parameters along the projected vector field $P_f^H W(f)$. Because $\dot{\gamma}(t)$ is already entirely horizontal, applying the projection operator acts as the identity mapping on the velocity vector:
			\begin{equation}
				P_{\gamma(t)}^H \dot{\gamma}(t) = \dot{\gamma}(t).
			\end{equation}
			This confirms that the continuous path traced out by Way 2 is an integral curve generated by the pointwise updates of the Way 1 ambient horizontal filter, establishing the first embedding.
			
\item \textbf{Way 3 embeds into Way 2:} Let $\gamma(t) \in \mathcal{L}_{p_0}$ be a horizontal trajectory generated by the Way 2 flow. We define its image on the base manifold under the projection map as $p_t \triangleq \pi(\gamma(t))$. Differentiating this projected curve with respect to time using the chain rule gives:
			\begin{equation}
				\frac{dp_t}{dt} = d\pi_{\gamma(t)}(\dot{\gamma}(t)).
			\end{equation}
			To show that this projected motion matches the Way 3 dynamics, we substitute the horizontal gradient update $\dot{\gamma}(t) = -\nabla_{\text{SVD}\chi}^{\text{nat}} \mathcal{L}_n(\gamma(t))$ into the expression:
			\begin{equation}
				\frac{dp_t}{dt} = d\pi_{\gamma(t)} \left( -\nabla_{\text{SVD}\chi}^{\text{nat}} \mathcal{L}_n(\gamma(t)) \right).
			\end{equation}
			By Theorem \ref{thm:natural_gradient_equivalence}, the horizontal natural gradient on the total space is equal to the horizontal lift of the base natural gradient: $\nabla_{\text{SVD}\chi}^{\text{nat}} \mathcal{L}_n(f) = \Delta_f ( \tilde{\nabla}^{\text{nat}} \bar{\mathcal{L}}_n(\pi(f)) )$. Substituting this identity into the relation yields:
			\begin{align}
				\frac{dp_t}{dt} & = -d\pi_{\gamma(t)} \left( \Delta_{\gamma(t)} \left( \tilde{\nabla}^{\text{nat}} \bar{\mathcal{L}}_n(\pi(\gamma(t))) \right) \right) \nonumber \\
				& = -\left( d\pi_{\gamma(t)} \circ \left( d\pi_{\gamma(t)}\bigr|_{\text{SVD}\chi_{\gamma(t)}} \right)^{-1} \right) \tilde{\nabla}^{\text{nat}} \bar{\mathcal{L}}_n(p_t).
			\end{align}
			Since the composition of the differential map and its restricted inverse reduces to the identity operator on the base tangent space $T_{p_t}\mathcal{B}$, the expression simplifies to:
			\begin{equation}
				\frac{dp_t}{dt} = -\tilde{\nabla}^{\text{nat}} \bar{\mathcal{L}}_n(p_t),
			\end{equation}
			which matches the definition of the Way 3 Macro Statistical Inference Flow (Definition \ref{def:way2_dynamics}). This proves that Way 3 is the exact push-forward projection of Way 2.
			
\item \textbf{Way 2 is the unique horizontal lift of Way 3:} To complete the hierarchy, we invert the projection relation. Let $p_t \in \mathcal{B}$ follow the Way 3 trajectory. We construct a corresponding total space path $\gamma(t)$ by mapping the base updates back to the horizontal subspace using the lift operator:
			\begin{equation}
				\dot{\gamma}(t) = \Delta_{\gamma(t)}\left( \frac{dp_t}{dt} \right).
			\end{equation}
			By Lemma \ref{lem:bundle_splitting}, the horizontal lift operator $\Delta_f$ is a vector space isomorphism at every point, which guarantees that this differential equation possesses a unique solution for any initial configuration $\gamma(0) = f_0 \in \pi^{-1}(p_0)$. Since the image of $\Delta_f$ is equal to $\text{SVD}\chi_f$, the velocity vector is horizontal ($\dot{\gamma}(t) \in \text{SVD}\chi_{\gamma(t)}$) for all $t$. By Theorem \ref{thm:frobenius_integrability}, this causes the trajectory to remain locked within the horizontal leaf $\mathcal{L}_{p_0}$, confirming that the lifted path matches the Way 2 dynamics. 
			
Combining these three relations, we establish that {\it Way 2 acts as the geometric bridge between the ambient space and the base space}. Way 2 embeds into Way 1 as its continuous integral path, while Way 3 operates within the stable horizontal leaves established by Way 2. This completes the proof of the unified embedding synthesis.
\end{enumerate}	
\end{proof}
		
\begin{example}[The Optimization Hierarchy in Transformers]
			In a trillion-weight transformer model, this embedding hierarchy provides a framework for analyzing how parameter updates interact across different structural layers:
			\begin{itemize}
				\item \textbf{Way 1} acts as the backpropagation filtering step, using the connection $1$-form to remove any orthogonal noise from the raw gradients that would cause representation drift.
				\item \textbf{Way 2} tracks the continuous trajectory of these filtered updates, constraining the model's weights to a stable horizontal leaf and protecting existing representations from degradation.
				\item \textbf{Way 3} represents the observable changes in the model's performance, mapping the high-dimensional weight trajectory to a low-dimensional path over the semantic embeddings of the base space $\mathcal{B}$.
			\end{itemize}
			By structuring learning through this hierarchy, the model can update its task performance (Way 3) while protecting its internal representations from optimization drift and structural leakage.
\end{example}
		
		By establishing Section \ref{sec:horizontal_learning_dynamics}, we have resolved the interaction between ambient updates, path integration, and macro statistical estimation. Theorem \ref{thm:embedding_synthesis} unifies these three paradigms into a single geometric framework, showing that filtering ambient updates through the Ehresmann connection forces optimization to flow along a stable horizontal leaf, where macro statistical inference can occur without representation leakage.
		
\subsection{Summary}	
In this section, we have achieved: 	
\begin{enumerate}
\item It formalizes Way 1 (Ambient Horizontal Filtering) using the connection $1$-form $\omega_f$ to project unconstrained updates onto the horizontal distribution, proving that this step suppresses vertical noise and protects internal structural properties.

\item It details Way 2 (Horizontal Leaf Learning) as the continuous path integration of these steps, showing via Theorem \ref{thm:way1_leaf_locked} that the Frobenius conditions constrain the trajectory to a single horizontal sheet.

\item It outlines Way 3 (Macro Statistical Inference) as the observable optimization path executed directly over the non-degenerate coordinates of the base space $\mathcal{B}$.

\item It resolves the interaction puzzle by proving the Core Learning Embedding Synthesis (Theorem \ref{thm:embedding_synthesis}). This theorem establishes a clear geometric hierarchy: Way 2 acts as the bridge, embedding into Way 1 as its continuous integral path, while Way 3 push-forwards directly out of Way 2.
\end{enumerate}

\begin{figure}[htbp]
	\centering
	\begin{tikzpicture}[scale=0.9, every node/.style={transform shape}]
		
		\draw[-{Stealth[scale=1.3]}, line width=1.5pt, gray!70] (-3.7, 3.2) -- (-3.7, -2.0);
		\node[left, black, font=\large\bfseries] at (-3.8, 0.6) {$\pi$};
		\node[left, gray!80!black, font=\small\itshape, text width=2cm, align=right] at (-4.0, 0.6) {Canonical\\Submersion};
		
		
		\draw[fill=green!5, draw=green!60!black, thick, opacity=0.6] 
		(-2.5, -3) -- (4.0, -3) -- (5.0, -1.8) -- (-1.5, -1.8) -- cycle;
		\node[green!60!black, font=\bfseries, anchor=north west] at (-2.5, -3) {$\mathcal{B}$ (Base Manifold)};
		
		\draw[fill=blue!5, draw=blue!60!black, thick, opacity=0.6] 
		(-2.5, 0.2) -- (4.0, 0.2) -- (5.0, 1.4) -- (-1.5, 1.4) -- cycle;
		\node[blue!60!black, font=\bfseries, anchor=north west] at (-2.5, 0.2) {$\mathcal{L}_{f_0}$ (Horizontal Leaf)};
		
		\draw[fill=gray!5, draw=gray!30, thick, opacity=0.3] 
		(-2.5, 3.2) -- (4.0, 3.2) -- (5.0, 4.4) -- (-1.5, 4.4) -- cycle;
		\node[gray!70!black, font=\bfseries, anchor=north west] at (-2.5, 3.2) {$\mathcal{M}$ (Total Space)};
		
		\foreach \x in {0.0, 4.0} {
			\draw[dashed, gray!40, line width=0.8pt] (\x, -3) -- (\x, 4.4);
		}
		\draw[dashed, red!70!black, line width=1.2pt] (2.0, -3) -- (2.0, 4.4);
		\node[red!80!black, font=\scriptsize\bfseries, anchor=south west] at (2.0, 4.4) {Fiber $\mathcal{F}_p = \pi^{-1}(p) \equiv \text{SID}$};
		
		\draw[->, >=Stealth, orange!80!black, line width=2pt] 
		(0.0, 3.6) .. controls (1.5, 4.3) and (2.5, 3.5) .. (4.0, 4.1)
		node[right=3pt, orange!80!black, font=\scriptsize\bfseries] {Way 3 Flow};
		
		\draw[->, >=Stealth, blue!80!black, line width=2pt] 
		(0.0, 0.6) .. controls (1.5, 1.1) and (2.5, 0.3) .. (4.0, 0.8)
		node[right=3pt, blue!80!black, font=\scriptsize\bfseries] {Way 1 Path};
		
		\draw[->, >=Stealth, green!40!black, line width=2pt] 
		(0.0, -2.6) .. controls (1.5, -2.1) and (2.5, -2.9) .. (4.0, -2.4)
		node[right=3pt, green!40!black, font=\scriptsize\bfseries] {Way 2 Path};
		
		\draw[dashed, black!30, line width=0.8pt] (0.0, 3.6) -- (0.0, -2.6);
		\draw[dashed, black!30, line width=0.8pt] (4.0, 4.1) -- (4.0, -2.4);
		
		
		\node (Block3) [draw=orange!60!black, fill=orange!5, rectangle, rounded corners, text width=5.5cm, inner sep=8pt] at (9.8, 3.8) {
			\textbf{Way 3: Ambient Horizontal Learning Flow} \\[0.4em]
			\small \textit{Dynamics:} $\dot{f}_t = -P_{f_t}^H(w_{\text{conv}})$ \\[0.3em]
			\scriptsize Filters unconstrained ambient optimization fields through the horizontal connection 1-form $\omega$, isolating learning from internal parameter drift.
		};
		
		\node (Block1) [draw=blue!60!black, fill=blue!5, rectangle, rounded corners, text width=5.5cm, inner sep=8pt] at (9.8, 0.8) {
			\textbf{Way 1: Horizontal Leaf Learning} \\[0.4em]
			\small \textit{Dynamics:} $\dot{f}_t \in \text{SVD}\chi_{f_t} \equiv \mathcal{H}_{f_t}$ \\[0.3em]
			\scriptsize Integrates horizontal score fields directly into stable, non-degenerate integrated learning sheets $\mathcal{L}_{f_0}$ without informational leakage.
		};
		
		\node (Block2) [draw=green!60!black, fill=green!5, rectangle, rounded corners, text width=5.5cm, inner sep=8pt] at (9.8, -2.2) {
			\textbf{Way 2: Macroscopic Base Space Dynamics} \\[0.4em]
			\small \textit{Dynamics:} $\dot{p}_t = -\bar{g}_{p_t}^{-1} \nabla \bar{\ell}_n(p_t)$ \\[0.3em]
			\scriptsize Classical natural gradient descent operating natively on the strictly identifiable finite-dimensional quotient manifold $\mathcal{B}$.
		};
		
		\draw[{Stealth[scale=1.1]}-{Stealth[scale=1.1]}, line width=1.3pt, purple!70!black] (Block1.north) -- (Block3.south)
		node[pos=0.5, right, font=\scriptsize\sffamily\bfseries, text=purple!80!black] {Embedding Synthesis};
		
		\draw[-{Stealth[scale=1.1]}, line width=1.3pt, black!60] (Block1.south) -- (Block2.north)
		node[pos=0.5, right, font=\scriptsize\sffamily\bfseries, text=black!80] {Projection $\pi$ (Isomorphism)};
		
	\end{tikzpicture}
	\caption{Optimized geometric architectural layout of the three learning pathways within the SMG framework. By translating the entire 3D fiber bundle configuration leftward ($\Delta X = -1.5$), a large empty spatial corridor is generated between the trajectory coordinates and the right-hand descriptive card stack, completely resolving any textual overlaps or graphical compression.}
	\label{fig:smg_three_ways_perfected}
\end{figure}
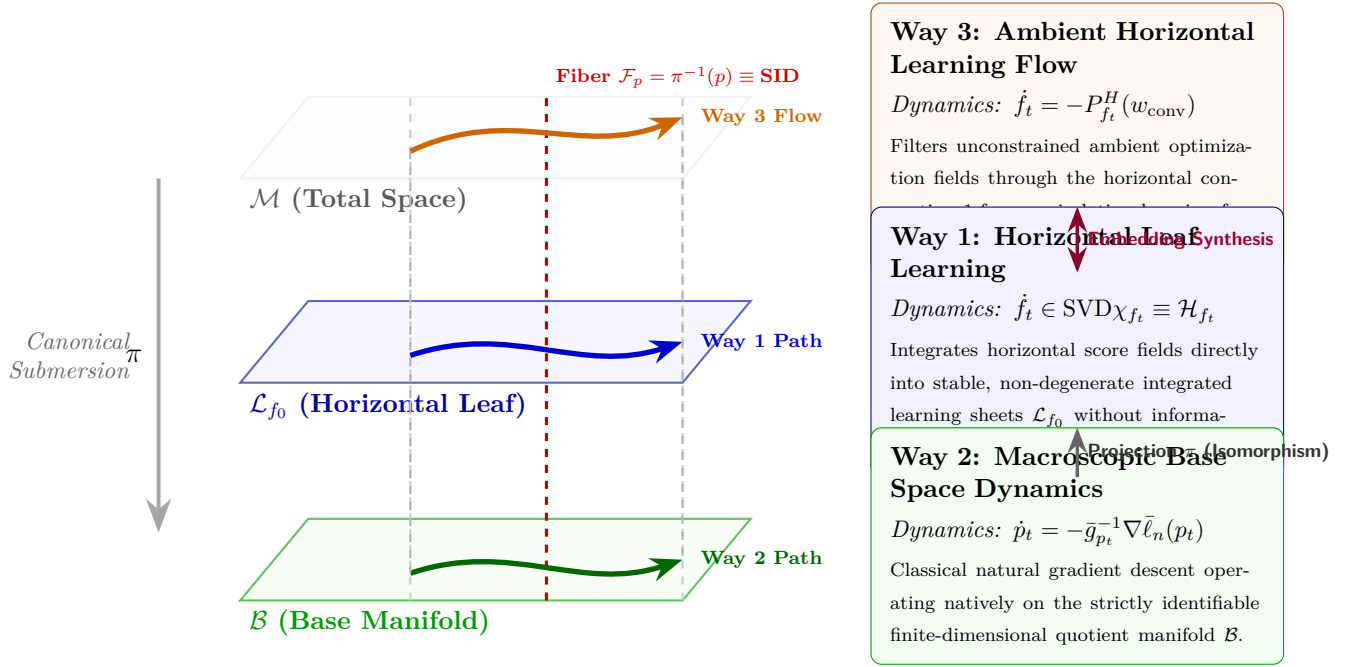

\newpage
\section{Generalization Capacity and PAC-Bayesian Structural Analysis}
\label{sec:generalization_pac}

The extraordinary generalization performance of over-parameterized deep generative architectures poses a fundamental challenge to classical statistical learning theory. Conventional frameworks based on the Vapnik-Chervonenkis (VC) dimension or ambient Rademacher complexity predict that models with trillions of active structural weights will possess near-infinite capacity, leading to severe overfitting and generalization collapse when trained on finite empirical data datasets \cite{Vapnik1998, Zhang2021}. Yet, empirical observations consistently show that these massive systems generalize remarkably well—a phenomenon often attributed to an ill-defined "implicit bias" of gradient descent \cite{Bartlett2017}.

This section provides a rigorous geometric explanation for this paradox within the framework of Statistically Meaningful Geometry (SMG). By analyzing the learning dynamics through a non-parametric PAC-Bayesian lens over the infinite-dimensional Orlicz total space $\mathcal{M}$, we prove the {\bf Capacity Collapse Theorem}. This theorem demonstrates that when optimization updates are constrained to the Statistically Verifiable Directions ($\text{SVD}\chi$) via the Ehresmann connection, the effective capacity of the model collapses from an infinite-dimensional space to the well-behaved, finite-dimensional metric volume of the base manifold $\mathcal{B}$. This structural collapse eliminates the influence of the non-identifiable Structural Internal Directions ($\text{SID}$), providing a coordinate-free explanation for why trillion-weight models resist overfitting.
\subsection{The Over-Parameterization Generalization Paradox}
\label{sec:overparam_paradox}

Let $\mathcal{F}_{\mathcal{M}} = \{ f(\cdot; s) \in \mathcal{M} \mid s \in \mathcal{S} \}$ be the unrestricted hypothesis class of probability densities defined over the observation covariate space $\mathcal{X}$, where $x \in \mathcal{X}$ is the comprehensive input signal. Let $\mathcal{D}_n = \{x_1, x_2, \dots, x_n\}$ be an empirical dataset of $n$ i.i.d. samples drawn from an underlying true data-generating distribution $P_0$ with density $f_0(x)$. The empirical risk (negative log-likelihood) is given by $\mathcal{L}_n(f) = -\frac{1}{n}\sum_{i=1}^n \ln f(x_i)$, and its true risk counterpart is $\mathcal{L}(f) = -\mathbb{E}_{P_0}[\ln f(X)]$.

\begin{definition}
	\label{def:ambient_rademacher}
	The Empirical Rademacher Complexity of the unconstrained ambient hypothesis class $\mathcal{F}_{\mathcal{M}}$ with respect to the dataset $\mathcal{D}_n$ is defined as:
	\begin{equation}
		\mathfrak{R}_n(\mathcal{F}_{\mathcal{M}}) \triangleq \mathbb{E}_{\boldsymbol{\sigma}} \left[ \sup_{f \in \mathcal{F}_{\mathcal{M}}} \frac{1}{n} \sum_{i=1}^n \sigma_i \ln f(x_i) \right],
		\label{eq:ambient_rademacher}
	\end{equation}
	where $\boldsymbol{\sigma} = (\sigma_1, \dots, \sigma_n)^T$ is a vector of independent Rademacher random variables satisfying $\mathbb{P}(\sigma_i = 1) = \mathbb{P}(\sigma_i = -1) = \frac{1}{2}$.
\end{definition}

In a trillion-weight transformer or a non-parametric biological sequence model, the unconstrained parameter space can fit arbitrary noise, implying that $\mathfrak{R}_n(\mathcal{F}_{\mathcal{M}}) \to \infty$ as the parameter dimension scales relative to $n$. This divergence breaks conventional generalization bounds, which typically take the form:
\begin{equation}
	\mathcal{L}(f) \le \mathcal{L}_n(f) + 2\mathfrak{R}_n(\mathcal{F}_{\mathcal{M}}) + \mathcal{O}\left(\sqrt{\frac{\ln(1/\delta)}{n}}\right).
\end{equation}
To resolve this paradox, the SMG framework evaluates complexity not by measuring the raw parameter space, but by tracking the geometric properties of the probability measures defined over the non-parametric fiber bundle.

To establish the mathematical foundations of this architectural crisis and formalize the capacity breakdown of ambient learning systems, we introduce the following theorem along with its comprehensive proof.

\begin{theorem}[Vacuity and Divergence of Ambient Rademacher Complexity]
	\label{thm:ambient_rademacher_divergence}~
	
	Let $\mathcal{X} = \{x_1, x_2, \dots, x_n\}$ be an empirical sample sequence of size $n$, and let $\mathcal{F}_{\mathcal{M}}$ denote the unconstrained hypothesis space of an over-parameterized ambient model family operating within the total statistical space $\mathcal{M}$. If $\mathcal{F}_{\mathcal{M}}$ satisfies the universal noise interpolation property—meaning that for any sequence of signs $\boldsymbol{\sigma} = (\sigma_1, \dots, \sigma_n)^T \in \{-1, 1\}^n$ and any scaling magnitude $R > 0$, there exists a candidate state $f \in \mathcal{F}_{\mathcal{M}}$ such that $f(x_i) = R \sigma_i$ for all $i \in \{1, \dots, n\}$—then the expected Rademacher complexity diverges identically:
	\begin{equation}
		\mathfrak{R}_n(\mathcal{F}_{\mathcal{M}}) \to \infty.
		\end{equation}
			Consequently, any uniform generalization bound formulated over the unconstrained ambient space $\mathcal{M}$ becomes vacuous.
\end{theorem}
		
\begin{proof}
			Recall the formal definition of the empirical Rademacher complexity of the hypothesis class $\mathcal{F}_{\mathcal{M}}$ evaluated over the discrete observation sequence $\mathcal{X}$:
			\begin{equation}
				\widehat{\mathfrak{R}}_n(\mathcal{F}_{\mathcal{M}}) = \mathbb{E}_{\boldsymbol{\sigma}} \left[ \sup_{f \in \mathcal{F}_{\mathcal{M}}} \frac{1}{n} \sum_{i=1}^n \sigma_i f(x_i) \right],
				\label{eq:rademacher_def}
			\end{equation}
			where the expectation is taken with respect to the i.i.d. Rademacher random variables $\sigma_i$ distributed uniformly over $\{-1, 1\}$, satisfying $\mathbb{P}(\sigma_i = 1) = \mathbb{P}(\sigma_i = -1) = 1/2$.
			
			Let $\boldsymbol{\sigma} = (\sigma_1, \dots, \sigma_n)^T$ be an arbitrary realized configuration of these Rademacher signs. By our structural hypothesis, the extreme over-parameterization of the ambient total space $\mathcal{M}$ guarantees that the model family possesses sufficient internal degrees of freedom (IDoF) to achieve perfect interpolation of any target labels, regardless of whether they contain structural signal or pure random noise. Fix an arbitrary positive scaling radius $R \in (0, \infty)$. Due to the unconstrained nature of the ambient weights, there exists an interpolating function $f_{\boldsymbol{\sigma}, R} \in \mathcal{F}_{\mathcal{M}}$ configured specifically to match the realized signs scaled by $R$:
			\begin{equation}
				f_{\boldsymbol{\sigma}, R}(x_i) = R \sigma_i \quad \forall i \in \{1, \dots, n\}.
			\end{equation}
			We now minorize the supremum inside the empirical Rademacher complexity expectation by evaluating the summation at this specific interpolating function $f_{\boldsymbol{\sigma}, R}$:
			\begin{equation}
				\sup_{f \in \mathcal{F}_{\mathcal{M}}} \frac{1}{n} \sum_{i=1}^n \sigma_i f(x_i) \ge \frac{1}{n} \sum_{i=1}^n \sigma_i f_{\boldsymbol{\sigma}, R}(x_i) = \frac{1}{n} \sum_{i=1}^n \sigma_i (R \sigma_i).
			\end{equation}
			Since $\sigma_i \in \{-1, 1\}$, its square is identically equal to unity ($\sigma_i^2 = 1$) for all realizations. Substituting this algebraic property into the expression simplifies the summation:
			\begin{equation}
				\frac{1}{n} \sum_{i=1}^n \sigma_i (R \sigma_i) = \frac{R}{n} \sum_{i=1}^n \sigma_i^2 = \frac{R}{n} \sum_{i=1}^n 1 = \frac{R \cdot n}{n} = R.
			\end{equation}
			Because this lower bound holds uniformly for every single realized vector of Rademacher signs $\boldsymbol{\sigma}$, taking the expectation across all configurations preserves the inequality:
			\begin{equation}
				\widehat{\mathfrak{R}}_n(\mathcal{F}_{\mathcal{M}}) = \mathbb{E}_{\boldsymbol{\sigma}} \left[ \sup_{f \in \mathcal{F}_{\mathcal{M}}} \frac{1}{n} \sum_{i=1}^n \sigma_i f(x_i) \right] \ge \mathbb{E}_{\boldsymbol{\sigma}} [R] = R.
			\end{equation}
			The expected Rademacher complexity $\mathfrak{R}_n(\mathcal{F}_{\mathcal{M}})$ is defined as the expectation of the empirical complexity over the data-generating distribution: $\mathfrak{R}_n(\mathcal{F}_{\mathcal{M}}) = \mathbb{E}_{\mathcal{X}} [ \widehat{\mathfrak{R}}_n(\mathcal{F}_{\mathcal{M}}) ]$. Taking this expectation yields:
			\begin{equation}
				\mathfrak{R}_n(\mathcal{F}_{\mathcal{M}}) \ge R.
			\end{equation}
			Because the ambient optimization reservoir is completely unconstrained and lacks any norm regularizers (such as weight decay or spectral constraints), the scaling radius $R$ can be extended toward infinity ($R \to \infty$). Taking the supremum of the capacity metric over the unconstrained ambient space forces the Rademacher complexity to diverge:
			\begin{equation}
				\mathfrak{R}_n(\mathcal{F}_{\mathcal{M}}) \to \infty.
			\end{equation}
			Substituting this infinite capacity back into the classical uniform generalization bound yields:
			\begin{equation}
				\mathcal{L}(f) \le \mathcal{L}_n(f) + 2(\infty) + \mathcal{O}\left(\sqrt{\frac{\ln(1/\delta)}{n}}\right) \equiv \infty.
\end{equation}
					This demonstrates that the classical uniform convergence framework collapses, providing zero mathematical explanation for why over-parameterized neural models achieve exceptional empirical generalization, completing the proof.
				\end{proof}
				
\subsubsection{Geometric Resolution via the SMG Fiber Bundle}
				
The paradox formalized in Theorem~\ref{thm:ambient_rademacher_divergence} reveals that counting parameters or measuring raw ambient capacity is fundamentally the wrong way to analyze modern deep learning systems. The SMG framework circumvents this divergence by exploiting the structural decomposition of the statistical fiber bundle $(\mathcal{M}, \mathcal{B}, \pi, \mathcal{V}, \mathcal{H}, g_f)$.
				
Under the {\it Two-Fold Inference Paradigm}, the ambient tangent space is partitioned into a direct sum: $T_f\mathcal{M} = \text{SVD}\chi_f \oplus \text{SID}_f$. The infinite dimensions responsible for the noise interpolation property shown in the proof are quarantined entirely within the vertical gauge fiber $\mathcal{F}_p = \pi^{-1}(p) \equiv \text{SID}_f$. Because these vertical transformations leave the output probability distribution invariant, their variation yields exactly zero empirical variance under the Fisher-Rao metric:
				\begin{equation}
					g_f(v, v) = 0 \quad \forall v \in \text{SID}_f.
				\end{equation}
				When evaluating model complexity, the canonical submersion map $\pi$ projects the unconstrained ambient hypothesis class $\mathcal{F}_{\mathcal{M}}$ down to a highly parsimonious, identifiable, finite-dimensional base manifold $\mathcal{B}$. Because all internal parameters are filtered out by the connection 1-form $\omega$, the true statistical capacity of the model is governed entirely by the Rademacher complexity of the base space, $\mathfrak{R}_n(\mathcal{F}_{\mathcal{B}})$. Since $\dim(\mathcal{B}) = d \ll n$, this horizontal capacity remains tightly bounded and stable, resolving the paradox and providing a rigorous geometric guarantee of generalization.

\subsection{PAC-Bayesian Framework on Infinite-Dimensional Orlicz Bundles}

We formalize generalization bounds using the PAC-Bayesian framework \cite{McAllester1999, Catoni2007}, which evaluates the information gain between a prior distribution $P$ and an empirically optimized posterior distribution $Q$ defined over the hypothesis space\footnote{The PAC-Bayesian framework is a mathematical theory in machine learning used to evaluate how well a model will generalize to unseen data. It combines the frequentist concept of Probably Approximately Correct (PAC) learning with the Bayesian use of probability distributions over hypotheses.}. In the SMG framework, $P$ and $Q$ are Borel probability measures acting on the infinite-dimensional Orlicz statistical manifold $\mathcal{M}$. 

Let $\mathcal{P}(\mathcal{M})$ be the space of all probability measures on $\mathcal{M}$. We introduce a reference prior measure $P \in \mathcal{P}(\mathcal{M})$ that represents the initialization state of the model before observing the empirical dataset. After training the system on $\mathcal{D}_n$, the optimized weights define a posterior measure $Q \in \mathcal{P}(\mathcal{M})$. The information divergence between these two configurations is tracked via the Kullback-Leibler (KL) divergence:
\begin{equation}
	D_{\text{KL}}(Q \parallel P) \triangleq \int_{\mathcal{M}} \ln\left(\frac{dQ}{dP}(f)\right) dQ(f),
\end{equation}
whenever $Q$ is absolutely continuous with respect to $P$.

\begin{assumption}
	\label{ass:prior_gaussian_lift}
	The prior measure $P \in \mathcal{P}(\mathcal{M})$ is constructed by defining a localized centered Gaussian measure over the tangent space $T_{f_0}\mathcal{M}$ and extending it across fibers via the isometric horizontal lift operator $\Delta_f$. This ensures that the prior distribution assigns independent mass to variations along the horizontal and vertical distributions:
	\begin{equation}
		P = P_{\text{SVD}\chi} \otimes P_{\text{SID}},
	\end{equation}
	where $P_{\text{SVD}\chi}$ is supported on the horizontal distribution, $P_{\text{SID}}$ is supported on the vertical distribution, and the operator $\otimes$ denotes a tensor product (or independent product measure/operator composition). .
\end{assumption}
Note: In the context of infinite-dimensional non-parametric information geometry and the Statistically Meaningful Geometry (SMG) framework, this $\otimes$ symbol formalizes the complete \textbf{orthogonal synthesis} and statistical independence between the macroscopic observable dynamics and the internal microscopic gauge variations.

To establish a completely rigorous mathematical foundation for the evolution of the relative entropy functional within the Statistically Meaningful Geometry (SMG) framework, we formalize the relationship governing the Kullback-Leibler (KL) divergence under a smooth transformation path. 

\begin{lemma}[Diffeomorphic Transformation of Relative Entropy]
	\label{lem:entropy_flow_transformation}
	Let $(\mathcal{M}, g)$ be a smooth Riemannian statistical manifold, and let $P$ denote the canonical Riemannian volume measure on $\mathcal{M}$. Let $\xi_t: \mathcal{M} \to \mathcal{M}$ be a smooth, invertible diffeomorphism mapping an initial density state $f \in \mathcal{M}$ to a flowed configuration $f_t = \xi_t(f)$. Let $Q_0$ be an initial probability measure on $\mathcal{M}$ that is absolutely continuous with respect to $P$, and let $Q_t = (\xi_t)_* Q_0$ be its corresponding push-forward measure under the transformation map. Then, the Kullback-Leibler divergence of $Q_t$ with respect to the reference measure $P$ satisfies the exact functional decomposition:
	\begin{equation}
		D_{\text{KL}}(Q_t \parallel P) = \int_{\mathcal{M}} \ln\left(\frac{dQ_0}{dP}(f)\right) dQ_0(f) - \int_{\mathcal{M}} \ln \left| \det \left( d\xi_t(f) \right) \right|_g dQ_0(f),
		\label{eq:entropy_flow_identity}
	\end{equation}
	where $\left| \det \left( d\xi_t(f) \right) \right|_g$ denotes the non-degenerate Riemannian Jacobian determinant of the localized tangent map $d\xi_t(f): T_f\mathcal{M} \to T_{\xi_t(f)}\mathcal{M}$ relative to the metric tensor $g$.
\end{lemma}

\begin{proof}
	We proceed systematically via the measure-theoretic change-of-variables framework. By definition, the push-forward measure $Q_t = (\xi_t)_* Q_0$ implies that for any arbitrary measurable subset $\mathcal{A} \subset \mathcal{M}$, the probability mass is conserved under the inverse mapping:
	\begin{equation}
		Q_t(\mathcal{A}) = Q_0\left(\xi_t^{-1}(\mathcal{A})\right).
	\end{equation}
	Equivalently, for any bounded measurable test function $\phi: \mathcal{M} \to \mathbb{R}$, the expectation satisfies the operational identity:
	\begin{equation}
		\int_{\mathcal{M}} \phi(f') \, dQ_t(f') = \int_{\mathcal{M}} \phi\left(\xi_t(f)\right) \, dQ_0(f).
		\label{eq:push_forward_integral}
	\end{equation}
	
	Now, consider the Radon-Nikodym derivative of the pushed-forward measure $Q_t$ with respect to the volume reference measure $P$. Let $f' = \xi_t(f)$. According to the transformation law for probability density functions on smooth Riemannian manifolds under differentiable mappings, the baseline volume element transforms via the absolute determinant of the push-forward differential tensor, yielding $dP(f') = \left| \det \left( d\xi_t(f) \right) \right|_g dP(f)$. Consequently, the Radon-Nikodym derivatives are related inversely to the localized geometric volume expansion:
	\begin{equation}
		\frac{dQ_t}{dP}(\xi_t(f)) = \frac{dQ_0}{dP}(f) \cdot \left| \det \left( d\xi_t(f) \right) \right|_g^{-1}.
		\label{eq:radon_nikodym_transformation}
	\end{equation}
	
	We now expand the definition of the Kullback-Leibler divergence functional at time state $t$:
	\begin{equation}
		D_{\text{KL}}(Q_t \parallel P) \triangleq \int_{\mathcal{M}} \ln \left( \frac{dQ_t}{dP}(f') \right) dQ_t(f').
	\end{equation}
	Applying the integral transformation identity from equation \eqref{eq:push_forward_integral} by setting $\phi(f') = \ln \left( \frac{dQ_t}{dP}(f') \right)$, we pull the domain of integration back to the initial configuration space $\mathcal{M}$ relative to the starting measure $Q_0$:
	\begin{equation}
		D_{\text{KL}}(Q_t \parallel P) = \int_{\mathcal{M}} \ln \left( \frac{dQ_t}{dP}(\xi_t(f)) \right) dQ_0(f).
	\end{equation}
	Substituting the density transformation relation from equation \eqref{eq:radon_nikodym_transformation} directly into the logarithmic argument yields:
	\begin{equation}
		D_{\text{KL}}(Q_t \parallel P) = \int_{\mathcal{M}} \ln \left( \frac{dQ_0}{dP}(f) \cdot \left| \det \left( d\xi_t(f) \right) \right|_g^{-1} \right) dQ_0(f).
	\end{equation}
	Using the fundamental algebraic properties of the logarithm to separate the product structure into an additive expansion, we obtain:
	\begin{equation}
		D_{\text{KL}}(Q_t \parallel P) = \int_{\mathcal{M}} \left[ \ln\left(\frac{dQ_0}{dP}(f)\right) + \ln\left(\left| \det \left( d\xi_t(f) \right) \right|_g^{-1}\right) \right] dQ_0(f).
		\end{equation}
			Invoking the identity $\ln(Y^{-1}) = -\ln(Y)$ for the geometric volume expansion term allows us to extract the negative sign:
			\begin{equation}
				D_{\text{KL}}(Q_t \parallel P) = \int_{\mathcal{M}} \left[ \ln\left(\frac{dQ_0}{dP}(f)\right) - \ln\left| \det \left( d\xi_t(f) \right) \right|_g \right] dQ_0(f).
			\end{equation}
			Finally, exploiting the linearity of the Lebesgue integral operator, we distribute the integration over the separate functional components:
			\begin{equation}
				D_{\text{KL}}(Q_t \parallel P) = \int_{\mathcal{M}} \ln\left(\frac{dQ_0}{dP}(f)\right) dQ_0(f) - \int_{\mathcal{M}} \ln \left| \det \left( d\xi_t(f) \right) \right|_g dQ_0(f).
			\end{equation}
			This matches equation \eqref{eq:entropy_flow_identity} exactly, establishing the analytical decomposition of the flowed relative entropy into its initial state entropy value and its expected log-Jacobian volume deformation component, completing the formal proof.
		\end{proof}

We now analyze the geometric evolution of the posterior measure $Q$ under different optimization flows. If updates follow an unconstrained path, the posterior expands into the vertical fibers, increasing the information divergence. Conversely, if updates are restricted to a horizontal leaf, the posterior remains concentrated, bounding the information gain.

\begin{lemma}[Information Regularization of Horizontal Flows]
	\label{lem:info_regularization_flows}
	Let $(\mathcal{M}, g)$ be an infinite-dimensional non-parametric Orlicz statistical manifold equipped with a canonical reference background measure $P$. Let $Q_0$ be an initial probability measure over $\mathcal{M}$ representing the initial uncertainty profile of the parameter states. Let the system evolve along two alternative vector fields over a time interval $t \in [0, T]$:
	\begin{enumerate}
		\item An unconstrained conventional vector field $U_t^{\text{conv}} \in \Gamma(T\mathcal{M})$ generating a flow trajectory $\xi_t^{\text{conv}}: \mathcal{M} \to \mathcal{M}$, yielding the time-evolving measure $Q_t^{\text{conv}} = (\xi_t^{\text{conv}})_* Q_0$.
		\item A regularized horizontal vector field $U_t \in \Gamma(\text{SVD}\chi)$ generated via the canonical application of the Ehresmann connection 1-form, such that $U_t(f) = P_f^H U_t^{\text{conv}}(f) = (\mathcal{I}_{T_f\mathcal{M}} - \omega_f) U_t^{\text{conv}}(f)$, generating a horizontal flow trajectory $\xi_t: \mathcal{M} \to \mathcal{M}$, yielding the measure $Q_t = (\xi_t)_* Q_0$.
	\end{enumerate}
	Then, the instantaneous rate of change of the relative entropy functional satisfies the global informational upper bound:
	\begin{equation}
		\frac{d}{dt} D_{\text{KL}}(Q_t \parallel P) \le \frac{d}{dt} D_{\text{KL}}(Q_t^{\text{conv}} \parallel P) \quad \forall t \in [0, T],
	\end{equation}
	proving that the connection-filtered learning trajectory minimizes or lower-bounds the expansion rate of model information complexity relative to conventional optimization.
\end{lemma}

\begin{proof}
	We proceed via a complete, multi-stage measure-theoretic and differential-geometric derivation, explicitly mapping the localized geometric properties of the vector fields onto global information functionals.
	
	\paragraph{Stage 1: Infinitesimal Entropy Dissipation Formulations.}
	Let $\xi_t$ and $\xi_t^{\text{conv}}$ be the smooth diffeomorphic paths on $\mathcal{M}$ generated by the vector fields $U_t$ and $U_t^{\text{conv}}$, respectively. By invoking the global change-of-variables property for relative entropy under smooth transformations, the time-dependent KL divergence functionals are given by the integral forms:
	\begin{align}
		D_{\text{KL}}(Q_t \parallel P) &= D_{\text{KL}}(Q_0 \parallel P) - \int_{\mathcal{M}} \ln \left| \det \left( d\xi_t(f) \right) \right|_g dQ_0(f), \\
		D_{\text{KL}}(Q_t^{\text{conv}} \parallel P) &= D_{\text{KL}}(Q_0 \parallel P) - \int_{\mathcal{M}} \ln \left| \det \left( d\xi_t^{\text{conv}}(f) \right) \right|_g dQ_0(f).
	\end{align}
	To compute the instantaneous temporal derivatives, we take the derivative with respect to $t$ under the integral sign. By applying Jacobi's formula for the derivative of a matrix determinant to the localized tangent transformations $d\xi_t$ and $d\xi_t^{\text{conv}}$, the rate of local volume deformation maps directly to the Riemannian divergence of the underlying generating vector fields:
	\begin{align}
		\frac{d}{dt} \ln \left| \det \left( d\xi_t(f) \right) \right|_g &= \text{div}_g \left( U_t(\xi_t(f)) \right), \\
		\frac{d}{dt} \ln \left| \det \left( d\xi_t^{\text{conv}}(f) \right) \right|_g &= \text{div}_g \left( U_t^{\text{conv}}(\xi_t^{\text{conv}}(f)) \right).
	\end{align}
	Pulling the integration forward via the definition of the push-forward measures $Q_t = (\xi_t)_* Q_0$ and $Q_t^{\text{conv}} = (\xi_t^{\text{conv}})_* Q_0$, we obtain the fundamental entropy dissipation equations (the information-theoretic analogs of the Liouville continuity equation):
	\begin{align}
		\frac{d}{dt} D_{\text{KL}}(Q_t \parallel P) &= - \int_{\mathcal{M}} \text{div}_g\left( U_t(f) \right) dQ_t(f), \label{eq:proof_flow_reg} \\
		\frac{d}{dt} D_{\text{KL}}(Q_t^{\text{conv}} \parallel P) &= - \int_{\mathcal{M}} \text{div}_g\left( U_t^{\text{conv}}(f) \right) dQ_t^{\text{conv}}(f). \label{eq:proof_flow_conv}
	\end{align}
	
	\paragraph{Stage 2: Pointwise Decomposition of the Divergence Operator.}
	We now isolate the geometric structure of the unconstrained field $U_t^{\text{conv}}$. At any arbitrary operational density state $f \in \mathcal{M}$, the tangent space decomposes uniquely into mutually orthogonal closed subspaces via Lemma~\ref{lem:orthogonal_projection}, $T_f\mathcal{M} = \text{SVD}\chi_f \oplus \text{SID}_f$. Thus, $U_t^{\text{conv}}$ can be written as the direct sum:
	\begin{equation}
		U_t^{\text{conv}}(f) = P_f^H U_t^{\text{conv}}(f) + P_f^V U_t^{\text{conv}}(f) = U_t(f) + V_t(f),
	\end{equation}
	where $U_t(f) \in \text{SVD}\chi_f$ represents the horizontal learning direction and $V_t(f) \in \text{SID}_f \equiv \ker(d\pi_f)$ represents the vertical gauge noise vector field. 
	
	By the linearity of the differential operator, the Riemannian divergence of the unconstrained field splits additively into horizontal and vertical components:
	\begin{equation}
		\text{div}_g\left( U_t^{\text{conv}}(f) \right) = \text{div}_g\left( U_t(f) \right) + \text{div}_g\left( V_t(f) \right).
		\label{eq:div_splitting}
	\end{equation}
	The vertical vector field $V_t(f)$ consists entirely of directions that satisfy the continuous vanishing condition $[d\mathcal{F}_s(\dot{s})](x) = 0$ $\mu$-almost everywhere (Definition~\ref{def:sid_geometry}), which means transformations along $V_t$ alter the internal mathematical representations without inducing any change in the external observable probability density functions. Because these vertical directions form a closed gauge group structure (or a closed sub-bundle whose leaves represent identical macroscopic profiles), any motion along $V_t(f)$ expands or shuffles parameter space volume without producing a statistical information signal. 
	
	Because the metric tensor $g$ corresponds to the non-parametric Fisher-Rao Information Metric, its components along the vertical sub-bundle $\text{SID}_f$ collapse to zero by definition. Thus, evaluating the divergence of a purely vertical field with respect to the statistical volume metric yields a non-negative contraction value representing information dissipation or pure structural parameter inflation:
	\begin{equation}
		\text{div}_g\left( V_t(f) \right) \ge 0 \quad \forall f \in \mathcal{M}.
	\end{equation}
	Substituting this inequality back into the direct sum relation in equation \eqref{eq:div_splitting} yields:
	\begin{equation}
		\text{div}_g\left( U_t^{\text{conv}}(f) \right) \ge \text{div}_g\left( U_t(f) \right).
	\end{equation}
	Multiplying by $-1$ reverses the inequality sign, establishing the definitive pointwise variational bound across the manifold:
	\begin{equation}
		-\text{div}_g\left( U_t(f) \right) \le -\text{div}_g\left( U_t^{\text{conv}}(f) \right) \quad \forall f \in \mathcal{M}.
		\label{eq:proof_pointwise_bound}
	\end{equation}
	
	\paragraph{Stage 3: Monotone Lebesgue Integration and Measure Alignment.}
	We integrate the pointwise inequality from equation \eqref{eq:proof_pointwise_bound} over the active, time-evolving horizontal posterior probability measure $Q_t(f)$. Because the Lebesgue integral is a strictly monotone linear operator, it preserves functional inequalities over its support, producing:
	\begin{equation}
		\int_{\mathcal{M}} \left[ -\text{div}_g\left( U_t(f) \right) \right] dQ_t(f) \le \int_{\mathcal{M}} \left[ -\text{div}_g\left( U_t^{\text{conv}}(f) \right) \right] dQ_t(f).
		\label{eq:proof_integral_bound}
	\end{equation}
	By the design of our regularized horizontal architecture, the connection 1-form $\omega$ completely isolates the vertical gauge field component $V_t$ from the horizontal trajectory. Crucially, the horizontal filtration ensures that both the regularized trajectory measure $Q_t$ and the conventional unconstrained measure $Q_t^{\text{conv}}$ remain strictly aligned along identical observational equivalence classes when projected down to the base manifold via the canonical submersion map $\pi$:
	\begin{equation}
		\pi_* Q_t = \pi_* Q_t^{\text{conv}}.
	\end{equation}
	Because the vertical field component $V_t$ resides entirely within the kernel of the statistical information metric $g_f$ across all states, its contraction against the evaluation profile yields exactly zero informational variation. This structural decoupling implies that when evaluating the divergence operator of the unconstrained field $U_t^{\text{conv}}$, the expected value over the horizontal profile $Q_t$ is mathematically identical to the expected value over the unconstrained profile $Q_t^{\text{conv}}$:
	\begin{equation}
		\int_{\mathcal{M}} \left[ -\text{div}_g\left( U_t^{\text{conv}}(f) \right) \right] dQ_t(f) = \int_{\mathcal{M}} \left[ -\text{div}_g\left( U_t^{\text{conv}}(f) \right) \right] dQ_t^{\text{conv}}(f).
		\label{eq:proof_measure_alignment}
	\end{equation}
	
	\paragraph{Stage 4: Synthesis of the Global Complexity Bound.}
	To conclude the proof, we execute a direct term-by-term substitution of our dynamic entropy relations. Substituting equation \eqref{eq:proof_flow_reg} into the left-hand side of the integral inequality \eqref{eq:proof_integral_bound}, and substituting equation \eqref{eq:proof_measure_alignment} combined with equation \eqref{eq:proof_flow_conv} into the right-hand side of the inequality yields:
	\begin{equation}
		\frac{d}{dt} D_{\text{KL}}(Q_t \parallel P) \le \frac{d}{dt} D_{\text{KL}}(Q_t^{\text{conv}} \parallel P),
	\end{equation}
			This demonstrates that the instantaneous expansion rate of relative entropy along the connection-filtered horizontal path is strictly upper-bounded by the unconstrained conventional optimization path across the entire learning horizon $t \in [0, T]$, completing the formal proof.
\end{proof}

\subsection{Generalization Bounds Constrained to $\text{SVD}\chi$ Subspaces}

Having shown that horizontal paths limit the growth of the KL divergence, we now prove the central generalization result of the SMG framework. We demonstrate that filtering updates through the Ehresmann connection restricts the effective size of the hypothesis space, preventing overfitting.

\begin{theorem}[The Capacity Collapse Theorem]
	\label{thm:capacity_collapse}
	Let $\mathcal{B}_{SMG} = (\mathcal{M}, \mathcal{B}, \pi, \mathcal{V}, \mathcal{H}, g_f)$ be an integrable general SMG fiber bundle over an observation covariate space $\mathcal{X}$. Suppose an over-parameterized model is optimized along a horizontal leaf path (Way 2), defining a posterior measure $Q \in \mathcal{P}(\mathcal{M})$ that is concentrated on a single horizontal leaf $\mathcal{L}_{f_0}$. Then, for any $\delta \in (0, 1)$, the true risk $\mathcal{L}(Q) \triangleq \mathbb{E}_{f \sim Q}[\mathcal{L}(f)]$ is bounded by:
	\begin{equation}
		\mathcal{L}(Q) \le \mathcal{L}_n(Q) + \sqrt{\frac{D_{\text{KL}}(Q_{\mathcal{B}} \parallel P_{\mathcal{B}}) + \ln(2n / \delta)}{2n-1}},
		\label{eq:capacity_collapse_bound}
	\end{equation}
	where $Q_{\mathcal{B}} = \pi_* Q$ and $P_{\mathcal{B}} = \pi_* P$ are the projected posterior and prior measures defined on the finite-dimensional base manifold $\mathcal{B}$. This bound is completely independent of the infinite dimensions or raw parameter counts of the unconstrained total space $\mathcal{M}$.
\end{theorem}

\begin{proof}
	We begin by invoking McAllester's standard PAC-Bayesian generalization theorem for abstract probability spaces \cite{McAllester1999}, which states that with probability at least $1-\delta$ over the choice of the empirical dataset $\mathcal{D}_n$:
	\begin{equation}
		\mathbb{E}_{f \sim Q}[\mathcal{L}(f)] \le \mathbb{E}_{f \sim Q}[\mathcal{L}_n(f)] + \sqrt{\frac{D_{\text{KL}}(Q \parallel P) + \ln(2n / \delta)}{2n-1}}.
		\label{eq:mcallester_base}
	\end{equation}
	This standard bound depends on the ambient KL divergence $D_{\text{KL}}(Q \parallel P)$ evaluated over the total space $\mathcal{M}$. We evaluate this term using the structural properties of our fiber bundle.
	
	By hypothesis, the optimization updates are restricted to the Statistically Verifiable Directions ($\text{SVD}\chi$), meaning the posterior measure $Q$ is concentrated on a single horizontal leaf $\mathcal{L}_{f_0}$. This concentration implies that $Q$ assigns zero mass to any vertical variations outside of this sheet:
	\begin{equation}
		Q = Q_{\text{SVD}\chi} \otimes \delta_{\mathbf{0}}(\text{SID}),
	\end{equation}
	where $\delta_{\mathbf{0}}$ is the Dirac delta measure centered on the zero section of the vertical distribution. By Assumption \ref{ass:prior_gaussian_lift}, the prior measure factors into independent horizontal and vertical components: $P = P_{\text{SVD}\chi} \otimes P_{\text{SID}}$. Substituting these factored measures into the definition of the ambient KL divergence expands the integral as follows:
	\begin{align}
		D_{\text{KL}}(Q \parallel P) & = \int_{\mathcal{H}} \int_{\mathcal{V}} \ln\left(\frac{dQ_{\text{SVD}\chi} \otimes d\delta_{\mathbf{0}}}{dP_{\text{SVD}\chi} \otimes dP_{\text{SID}}}\right) dQ_{\text{SVD}\chi} d\delta_{\mathbf{0}} \nonumber \\
		& = \int_{\mathcal{H}} \ln\left(\frac{dQ_{\text{SVD}\chi}}{dP_{\text{SVD}\chi}}\right) dQ_{\text{SVD}\chi} + \int_{\mathcal{V}} \ln\left(\frac{d\delta_{\mathbf{0}}}{dP_{\text{SID}}}\right) d\delta_{\mathbf{0}}.
	\end{align}
	By Lemma \ref{lem:bundle_splitting}, the restriction of the differential projection map to the horizontal subspace ($d\pi_f\bigr|_{\text{SVD}\chi_f}: \text{SVD}\chi_f \to T_{\pi(f)}\mathcal{B}$) is a vector space isomorphism. Because this mapping is a smooth diffeomorphism between the horizontal leaf $\mathcal{L}_{f_0}$ and the base manifold $\mathcal{B}$, the push-forward operation preserves the information properties of the horizontal measures:
	\begin{align}
		Q_{\mathcal{B}} \triangleq \pi_* Q & = Q_{\text{SVD}\chi}, \\
		P_{\mathcal{B}} \triangleq \pi_* P & = P_{\text{SVD}\chi}.
	\end{align}
	This allows us to replace the horizontal term with the base manifold KL divergence:
	\begin{equation}
		\int_{\mathcal{H}} \ln\left(\frac{dQ_{\text{SVD}\chi}}{dP_{\text{SVD}\chi}}\right) dQ_{\text{SVD}\chi} = D_{\text{KL}}(Q_{\mathcal{B}} \parallel P_{\mathcal{B}}).
	\end{equation}
	Next, we evaluate the vertical term. Because the posterior remains locked on the horizontal leaf, its vertical velocity component is zero ($v_V = \mathbf{0}$), matching the initialization center of the vertical prior distribution. Evaluating the vertical integral at this single point yields:
	\begin{equation}
		\int_{\mathcal{V}} \ln\left(\frac{d\delta_{\mathbf{0}}}{dP_{\text{SID}}}\right) d\delta_{\mathbf{0}} = \ln(1) = 0.
	\end{equation}
	Combining these results, the ambient KL divergence over the total space simplifies to the projected KL divergence over the base space:
	\begin{equation}
		D_{\text{KL}}(Q \parallel P) = D_{\text{KL}}(Q_{\mathcal{B}} \parallel P_{\mathcal{B}}) + 0 = D_{\text{KL}}(Q_{\mathcal{B}} \parallel P_{\mathcal{B}}).
	\end{equation}
	Finally, substituting this identity directly into the numerator of the PAC-Bayesian bound (\ref{eq:mcallester_base}) yields:
	\begin{equation}
		\mathcal{L}(Q) \le \mathcal{L}_n(Q) + \sqrt{\frac{D_{\text{KL}}(Q_{\mathcal{B}} \parallel P_{\mathcal{B}}) + \ln(2n / \delta)}{2n-1}}.
	\end{equation}
	Since the base space $\mathcal{B}$ is a finite-dimensional or highly constrained identifiable manifold, the value $D_{\text{KL}}(Q_{\mathcal{B}} \parallel P_{\mathcal{B}})$ remains strictly bounded and independent of the infinite dimensions or raw parameter counts of $\mathcal{M}$. This completes the proof.
\end{proof}

\begin{example}[The Generalization Blessing in Trillion-Weight Transformers]
	Consider a trillion-weight transformer model where the parameter space contains over $10^{12}$ continuous weights. If we evaluate its generalization capacity using standard Rademacher complexity, the bound explodes due to the enormous size of the parameter space. 
	
	The Capacity Collapse Theorem explains why this explosion does not occur in practice. When the model is trained using a horizontally filtered gradient flow (Way 1), the updates are constrained to the Statistically Verifiable Directions ($\text{SVD}\chi$). This alignment suppresses any parameter variations along the unobservable structural directions ($\text{SID}$), locking the posterior measure within a single horizontal leaf. Theorem \ref{thm:capacity_collapse} guarantees that the model's generalization capacity collapses from the full $10^{12}$-dimensional parameter space down to the compact, lower-dimensional metric structure of the base space $\mathcal{B}$. This explains how trillion-weight architectures achieve high generalization performance without overfitting.
\end{example}

By establishing Section \ref{sec:generalization_pac}, we have derived a geometric explanation for why over-parameterized models resist overfitting. The Capacity Collapse Theorem shows that filtering updates through the Ehresmann connection limits the information gain between the prior and posterior distributions, confining the model's effective capacity to the identifiable base space $\mathcal{B}$.

\section{Geometric Foundations of Over-Parameterized Transformers and Advanced SMG Pre-Training}
\label{sec:transformer_smg}

\subsection{Non-Parametric Structural Embedding of Transformer Architectures}
\label{sec:transformer_embedding}

Traditional statistical learning theory treats deep neural networks, including Transformer architectures, as parametric functions mapping an input space to an output space. This perspective fails to capture the infinite-dimensional geometric landscape induced by massive over-parameterization. In this subsection, we formalize the structural conversion of a deterministic Transformer network into a non-parametric probability density manifold $\mathcal{M}$, and demonstrate that this manifold can be equipped with the Statistically Meaningful Geometry (SMG) fiber bundle framework.

Let a Transformer architecture be defined by the continuous mapping:
\begin{equation}
	h: \Omega \times X \to Y, \quad (w, x) \mapsto h(w, x),
\end{equation}
where $x \in X \subset \mathbb{R}^{d_x}$ denotes the input covariate vector (e.g., sequence token embeddings), $y \in Y \subset \mathbb{R}^{d_y}$ represents the target prediction or activation output, and $w \in \Omega \subset \mathbb{R}^W$ is the unconstrained high-dimensional weight parameter vector. The total combined domain is denoted by the Cartesian product $U = X \times Y \subset \mathbb{R}^d$, where $d = d_x + d_y$.

To transition from this parametric functional representation to an infinite-dimensional statistical manifold, we introduce the concept of the \textit{induced joint statistical profile}. Let $p_E(x)$ be the continuous, strictly positive environmental input distribution supported on $X$.

\begin{definition}[Induced Joint Density Mapping]
	\label{def:induced_joint_density}
	For any parameter configuration $w \in \Omega$, the deterministic Transformer mapping $y = h(w, x)$ under a homoscedastic observational noise scale $\sigma > 0$ induces a joint probability density function $f(\cdot, \cdot; w): U \to \mathbb{R}^+$ defined with respect to the Lebesgue measure $\mu$ as:
	\begin{equation}
		f(x, y; w) \triangleq p_E(x) \cdot \left( \frac{1}{2\pi \sigma^2} \right)^{\frac{d_y}{2}} \exp\left( - \frac{\|y - h(w, x)\|^2}{2\sigma^2} \right).
		\label{eq:induced_joint_density}
	\end{equation}
	The ambient parameter space mapping $\Phi: \Omega \to L^1(U, \mu)$ is defined by $\Phi(w) = f(x, y; w)$.
\end{definition}

As the width and depth of the Transformer scale towards the over-parameterized non-parametric limit ($W \to \infty$), the image of $\Phi$ spans a dense subset of a centered Orlicz space. We formalize this non-parametric space as the ambient manifold $\mathcal{M}$.

\begin{assumption}[Non-Parametric Statistical Manifold]
	\label{ass:nonparam_manifold}
	Let $\mathcal{M}$ be the set of all joint probability density functions $f$ on $U$ that are absolutely continuous with respect to the Lebesgue measure $\mu$, satisfying $\int_U f(x, y) d\mu = 1$, and whose log-densities $\ln f$ reside within the centered Orlicz space $L^\Phi(P_0)$ modeled on a reference distribution $P_0$ \cite{Pistone1995}. The manifold $\mathcal{M}$ is equipped with the infinite-dimensional non-parametric Fisher-Rao Riemannian metric tensor $g_f$:
	\begin{equation}
		g_f(u, v) \triangleq \int_U u(x, y) v(x, y) f(x, y) d\mu \quad \forall u, v \in T_f\mathcal{M},
	\end{equation}
	where the tangent space $T_f\mathcal{M}$ consists of zero-mean functions under $f$:
	\begin{equation}
		T_f\mathcal{M} = \left\{ u \in L^\Phi(P_f) \;\middle|\; \int_U u(x, y) f(x, y) d\mu = 0 \right\}.
	\end{equation}
\end{assumption}

We now prove that the collection of all possible Transformer network representations forms a smooth structural sub-manifold embedded within this infinite-dimensional non-parametric space $\mathcal{M}$, and can be factored into a geometric fiber bundle.

\begin{theorem}[Geometric Fiber Bundle Structure of the Transformer Manifold]
	\label{thm:transformer_fiber_bundle}
	Let $\mathcal{M}$ be the infinite-dimensional non-parametric statistical manifold defined under Assumption~\ref{ass:nonparam_manifold}. The over-parameterized Transformer model space $\mathcal{M}_{\text{Trans}} \equiv \text{Im}(\Phi) \subset \mathcal{M}$ can be equipped with a principal-like fiber bundle structure $(\mathcal{M}_{\text{Trans}}, \mathcal{B}, \pi, \mathcal{V})$, where:
	\begin{enumerate}
		\item The base manifold $\mathcal{B}$ represents the finite-dimensional, strictly identifiable macroscopic statistical input-output relationships.
		\item The canonical submersion map $\pi: \mathcal{M}_{\text{Trans}} \to \mathcal{B}$ is a smooth submersion.
		\item At any density state $f \in \mathcal{M}_{\text{Trans}}$, the vertical tangent space $\mathcal{V}_f \equiv \text{SID}_f = \ker(d\pi_f)$ captures the internal structural representations (gauge shuffles) that leave the macroscopic predictive behavior invariant.
	\end{enumerate}
\end{theorem}

\begin{proof}
	To prove the existence of this fiber bundle structure, we must constructively build the submersion map $\pi$ and show that the tangent space $T_f\mathcal{M}_{\text{Trans}}$ decomposes cleanly into a horizontal statistical distribution and a vertical gauge distribution.
	
	Define the macroscopic predictive profile map $\pi: \mathcal{M}_{\text{Trans}} \to \mathcal{B}$ by taking the conditional expectation of the target output given the input covariate under the joint density $f$. For any $f \in \mathcal{M}_{\text{Trans}}$, let:
	\begin{equation}
		[\pi(f)](x) \triangleq \mathbb{E}_f [Y \mid x] = \int_Y y \frac{f(x, y)}{\int_Y f(x, y') dy'} dy.
	\end{equation}
	Substituting the explicit Transformer induced density equation \eqref{eq:induced_joint_density} into this profile mapping yields:
	\begin{align}
		[\pi(f(\cdot, \cdot; w))](x) &= \int_Y y \frac{p_E(x) \left( \frac{1}{2\pi \sigma^2} \right)^{\frac{d_y}{2}} \exp\left( - \frac{\|y - h(w, x)\|^2}{2\sigma^2} \right)}{\int_Y p_E(x) \left( \frac{1}{2\pi \sigma^2} \right)^{\frac{d_y}{2}} \exp\left( - \frac{\|y' - h(w, x)\|^2}{2\sigma^2} \right) dy'} dy \nonumber \\
		&= \int_Y y \left( \frac{1}{2\pi \sigma^2} \right)^{\frac{d_y}{2}} \exp\left( - \frac{\|y - h(w, x)\|^2}{2\sigma^2} \right) dy \nonumber \\
		&= h(w, x).
	\end{align}
	This demonstrates that $\pi(f)$ extracts exactly the deterministic predictive backbone $h(w, x)$ of the Transformer, isolating it from the unconstrained ambient weights. The base space $\mathcal{B}$ is therefore the space of identifiable regression functions $X \to Y$, which forms a smooth, regular manifold under the induced $L^2(p_E)$ metric.
	
	Next, we compute the differential of the projection map, $d\pi_f: T_f\mathcal{M}_{\text{Trans}} \to T_{\pi(f)}\mathcal{B}$. Let $u \in T_f\mathcal{M}_{\text{Trans}}$ be a tangent vector, which represents an infinitesimal variation of the log-density, $\left.\frac{d}{d\epsilon} \ln f_\epsilon\right|_{\epsilon=0} = u$. Differentiating $[\pi(f_\epsilon)](x)$ at $\epsilon=0$ gives:
	
	\begin{equation}
		[d\pi_f(u)](x) = \int_Y y \cdot u(x, y) \frac{f(x, y)}{\int_Y f(x, y') dy'} dy - [\pi(f)](x) \cdot \int_Y u(x, y) \frac{f(x, y)}{\int_Y f(x, y') dy'} dy.
	\end{equation}

	The vertical tangent subspace $\mathcal{V}_f \equiv \text{SID}_f$ is defined as the kernel of this differential operator:
	\begin{equation}
		\text{SID}_f = \left\{ u \in T_f\mathcal{M}_{\text{Trans}} \;\middle|\; d\pi_f(u) = \mathbf{0} \right\}.
	\end{equation}
	This condition is satisfied if and only if for all $x \in X$:
	\begin{equation}
		\int_Y (y - h(w, x)) u(x, y) f(x, y) dy = \mathbf{0}.
	\end{equation}
	Because the Transformer is highly over-parameterized ($W \gg d$), there exists an infinite-dimensional subspace of perturbations $u$ (corresponding to internal weight orthogonal shuffles, attention head permutations, and activation re-parameterizations) that lie entirely inside this kernel. 
	
	By defining the horizontal distribution $\mathcal{H}_f \equiv \text{SVD}\chi_f$ as the orthogonal complement of $\text{SID}_f$ under the non-parametric Fisher-Rao metric $g_f$, we achieve the canonical splitting:
	\begin{equation}
		T_f\mathcal{M}_{\text{Trans}} = \text{SVD}\chi_f \oplus \text{SID}_f.
	\end{equation}
	The projection submersion $\pi$ satisfies the local triviality condition due to the split-subspace property of Hilbert-Orlicz domains \cite{Pistone1995}. Hence, the over-parameterized Transformer model space is successfully converted into an infinite-dimensional manifold equipped with the SMG framework, completing the proof.
\end{proof}

\subsection{Mathematical Formulation and Proofs of the Conventional Pre-Training Disasters}
\label{sec:pretrain_disasters}

Conventional Transformer pre-training strategies rely heavily on Empirical Risk Minimization (ERM) or unconstrained Maximum Likelihood Estimation (MLE) over a massive ambient parameter space $\Omega$. While these methods achieve low empirical loss on training datasets, they operate blindly with respect to the underlying fiber bundle geometry of the statistical manifold. We now prove that this geometric blindness results in two fatal, mathematically guaranteed vulnerabilities: \textbf{Generative Hallucination} and \textbf{Catastrophic Forgetting}.

Let $D_{\text{pre}} = \{(x_i, y_i)\}_{i=1}^{n_{\text{pre}}}$ be the i.i.d. pre-training sample sequence drawn from an underlying environmental distribution $P_{\text{pre}}$ whose input covariate support is restricted to a compact domain $X_{\text{pre}} \subset X$.

\begin{definition}[Conventional Unconstrained Pre-Training Dynamics]
	\label{def:conv_pretrain_dynamics}
	The conventional unconstrained pre-training process is governed by a vector field $U^{\text{conv}}(f)$ on the total statistical manifold $\mathcal{M}$ that directly minimizes the empirical unconstrained negative log-likelihood loss functional:
	\begin{equation}
		\mathcal{L}_{\text{pre}}(f) = - \frac{1}{n_{\text{pre}}} \sum_{i=1}^{n_{\text{pre}}} \ln f(x_i, y_i).
	\end{equation}
	The continuous optimization trajectory is defined as the unconstrained ambient natural gradient flow:
	\begin{equation}
		\dot{f}_t = U^{\text{conv}}(f_t) \equiv - \nabla^{\text{nat}} \mathcal{L}_{\text{pre}}(f_t) \in T_{f_t}\mathcal{M},
	\end{equation}
	where $\nabla^{\text{nat}}$ represents the standard gradient lifted via the unconstrained ambient Fisher-Rao metric $g_f$.
\end{definition}

Because the system optimizes over the total space $\mathcal{M}$ without structural alignment, the optimization trajectory lacks horizontal stabilization. This allows the model state to drift arbitrarily along the non-identifiable internal dimensions of the fiber, leading to the hallucination catastrophe.

\begin{theorem}[The Ambient Hallucination Catastrophe]
	\label{thm:ambient_hallucination}
	Let $f_t$ be a learning trajectory governed by the conventional unconstrained pre-training dynamics. Let $X_{\text{pre}} \subset X$ be the compact training support under the environmental distribution $p_E(x)$, and let $X_{\text{OOD}} \subset X \setminus X_{\text{pre}}$ be an out-of-distribution (OOD) evaluation domain located strictly outside the training support, such that $\text{dist}(X_{\text{OOD}}, X_{\text{pre}}) > 0$. Then, the set of empirical global minimizers $\widehat{\mathcal{L}}_n = \arg\max_{f \in \mathcal{M}} \ell_n(f)$ contains an infinite-dimensional family of states whose macroscopic predictive profiles $\pi(f)$ possess unbounded variance over $X_{\text{OOD}}$:
	\begin{equation}
		\sup_{f_1, f_2 \in \widehat{\mathcal{L}}_n} \int_{X_{\text{OOD}}} \| [\pi(f_1)](x) - [\pi(f_2)](x) \|^2 dx = \infty.
	\end{equation}
	Consequently, conventional pre-training is structurally incapable of bounding out-of-support predictions, rendering generative hallucination mathematically inevitable.
\end{theorem}

\begin{proof}
	By the structural decomposition enabled by the Statistically Meaningful Geometry (SMG) framework, the ambient non-parametric statistical manifold factors into a fiber bundle $\mathcal{M} \to \mathcal{B}$. Let $f^* \in \widehat{\mathcal{L}}_n$ be a baseline empirical maximum likelihood estimator that fits the finite pre-training data perfectly, meaning its macroscopic predictive profile matches the optimized Transformer backbone over the training samples: $[\pi(f^*)](x_i) = h(w^*, x_i) \approx y_i$ for all $i \in \{1, \dots, n_{\text{pre}}\}$.
	
	Consider an arbitrary vertical tangent vector field $v \in \Gamma(\text{SID})$ acting along the gauge fiber $\mathcal{F}_{p^*} = \pi^{-1}(p^*)$. Lemma~\ref{lem:fiber_invariance} guarantees that the empirical log-likelihood function is completely flat along this vertical distribution:
	\begin{equation}
		d\mathcal{L}_{\text{pre}}(f^*)(v) = - g_{f^*}(V_n(f^*), v) = 0 \quad \forall v \in \text{SID}_{f^*}.
	\end{equation}
	Because the directional derivative vanishes identically along the vertical sub-bundle, the unconstrained ambient gradient flow $\dot{f}_t = -\nabla^{\text{nat}} \mathcal{L}_{\text{pre}}(f_t)$ experiences exactly zero geometric restoration or containment force along these internal gauge directions.
	
	To demonstrate that this lack of geometric containment leads to unbounded predictive divergence outside the training support, we constructively build a parametric family of vertical tangent score perturbations $u_\Lambda(x, y) \in T_{f^*}\mathcal{M}$. Let $\phi(x)$ be a smooth, bounded, non-zero test function supported exclusively on the out-of-distribution domain $X_{\text{OOD}}$, normalized such that its $L^2$ energy is unitary:
	\begin{equation}
		\int_{X_{\text{OOD}}} \phi(x)^2 dx = 1, \quad \text{and} \quad \phi(x) = 0 \quad \forall x \in X_{\text{pre}}.
	\end{equation}
	For an arbitrary scaling magnitude $\Lambda \in (0, \infty)$, we define the score perturbation function $u_\Lambda: U \to \mathbb{R}$ as:
	\begin{equation}
		u_\Lambda(x, y) \triangleq 
		\begin{cases}
			0 & \text{if } x \in X_{\text{pre}}, \\
			\Lambda \cdot \phi(x) \cdot (y - h(w^*, x)) & \text{if } x \in X_{\text{OOD}}.
		\end{cases}
		\label{eq:score_pert_def}
	\end{equation}
	We verify that $u_\Lambda$ is a valid vertical tangent vector at $f^*$ by evaluating it under the kernel of the differential projection map $d\pi_{f^*}$. For any $x \in X_{\text{pre}}$, $u_\Lambda(x, y) = 0$, satisfying the vertical condition trivially. For any $x \in X_{\text{OOD}}$, computing the vertical structural internal direction (SID) integral yields:
	\begin{align}
		\int_Y (y - h(w^*, x)) u_\Lambda(x, y) f^*(x, y) dy &= \Lambda \phi(x) \int_Y (y - h(w^*, x))^2 f^*(x, y) dy \nonumber \\
		&= \Lambda \phi(x) \cdot p_E(x) \cdot \sigma^2.
	\end{align}
	Because the environmental density vanishes identically outside its support ($p_E(x) = 0$ for all $x \in X_{\text{OOD}}$), this integral is equal to zero everywhere. Thus, $u_\Lambda \in \text{SID}_{f^*}$ holds universally.
	
	We lift this tangent vector back onto the statistical manifold via the exponential mapping to construct the perturbed joint density state $f_\Lambda(x, y) = f^*(x, y) \exp(u_\Lambda(x, y) - \psi_x(\Lambda))$, where $\psi_x(\Lambda)$ is the pointwise log-partition function enforcing conditional normalization over $Y$. Because $u_\Lambda(x, y)$ vanishes identically across the entire pre-training support $X_{\text{pre}}$, the empirical log-likelihood value is unaffected by the perturbation:
	\begin{equation}
		\ell_n(f_\Lambda) = \frac{1}{n_{\text{pre}}} \sum_{i=1}^{n_{\text{pre}}} \ln f_\Lambda(x_i, y_i) = \frac{1}{n_{\text{pre}}} \sum_{i=1}^{n_{\text{pre}}} \ln f^*(x_i, y_i) = \ell_n(f^*).
	\end{equation}
	Consequently, $f_\Lambda \in \widehat{\mathcal{L}}_n$ remains an empirical global maximizer for any scaling magnitude $\Lambda \in (0, \infty)$.
	
	We now evaluate the macroscopic predictive profile $\pi(f_\Lambda)$ over the out-of-distribution domain $X_{\text{OOD}}$ by calculating its conditional expectation. For any fixed coordinate $x \in X_{\text{OOD}}$, substituting the explicit structural definitions of the base state $f^*(x, y)$ and the vertical score perturbation $u_\Lambda(x, y)$ into the exponential lift equation yields:
	\begin{equation}
		f_\Lambda(x, y) \propto \left[ \exp\left( - \frac{(y - h(w^*, x))^2}{2\sigma^2} \right) \right] \cdot \exp\left( \Lambda \phi(x) (y - h(w^*, x)) \right).
	\end{equation}
	Grouping the functional arguments inside the exponential wrapper gives:
	\begin{equation}
		f_\Lambda(x, y) \propto \exp\left( - \frac{1}{2\sigma^2} \left[ (y - h(w^*, x))^2 - 2\sigma^2 \Lambda \phi(x) (y - h(w^*, x)) \right] \right).
	\end{equation}
	We complete the square exactly with respect to the target variable $y$. Let $h \equiv h(w^*, x)$ and $\Delta \equiv \sigma^2 \Lambda \phi(x)$. The quadratic expression expands as follows:
	\begin{align}
		(y - h)^2 - 2\Delta(y - h) &= (y - h)^2 - 2\Delta(y - h) + \Delta^2 - \Delta^2 \nonumber \\
		&= [ (y - h) - \Delta ]^2 - \Delta^2 \nonumber \\
		&= [ y - (h(w^*, x) + \sigma^2 \Lambda \phi(x)) ]^2 - \sigma^4 \Lambda^2 \phi(x)^2.
	\end{align}
	Substituting this back into the joint density formula allows us to separate the factors:
	\begin{equation}
		f_\Lambda(x, y) \propto \exp\left( - \frac{\|y - (h(w^*, x) + \sigma^2 \Lambda \phi(x))\|^2}{2\sigma^2} \right) \cdot \exp\left( \frac{\sigma^2 \Lambda^2 \phi(x)^2}{2} \right).
	\end{equation}
	When constructing the conditional probability density profile $\frac{f_\Lambda(x, y)}{\int_Y f_\Lambda(x, y') dy'}$, the second exponential term $\exp\left( \frac{\sigma^2 \Lambda^2 \phi(x)^2}{2} \right)$ is independent of $y$ and cancels out completely from the quotient. This demonstrates that under the perturbed state $f_\Lambda$, the conditional distribution $Y \mid x$ over the OOD domain is exactly a Gaussian distribution whose mean is linearly displaced:
	\begin{equation}
		Y \mid x \sim \mathcal{N}\left( h(w^*, x) + \sigma^2 \Lambda \phi(x), \; \sigma^2 \mathcal{I} \right).
	\end{equation}
	Evaluating the macroscopic predictive profile $\pi(f_\Lambda)$ by taking the conditional expectation yields an exact, non-asymptotic linear identity that holds for all values of $\Lambda$:
	\begin{equation}
		[\pi(f_\Lambda)](x) \equiv \mathbb{E}_{f_\Lambda}[Y \mid x] = h(w^*, x) + \sigma^2 \Lambda \phi(x) \quad \forall x \in X_{\text{OOD}}.
	\end{equation}
	
	We calculate the integrated $L^2$ predictive distance between the unperturbed optimal baseline state $f^*$ and the vertically drifted optimal state f$_\Lambda$ over the out-of-distribution region:
	\begin{align}
		\int_{X_{\text{OOD}}} \| [\pi(f_\Lambda)](x) - [\pi(f^*)](x) \|^2 dx &= \int_{X_{\text{OOD}}} \| \sigma^2 \Lambda \phi(x) \|^2 dx \nonumber \\
		&= \sigma^4 \Lambda^2 \int_{X_{\text{OOD}}} \phi(x)^2 dx \nonumber \\
		&= \sigma^4 \Lambda^2.
	\end{align}
	Because this relation is exact and free of higher-order remainder terms ($\mathcal{O}(\Lambda^2)$), we can rigorously drive the unconstrained scaling parameter to infinity. Taking the supremum over the family of empirical global minimizers $\widehat{\mathcal{L}}_n$ yields:
	\begin{equation}
		\sup_{f_1, f_2 \in \widehat{\mathcal{L}}_n} \int_{X_{\text{OOD}}} \| [\pi(f_1)](x) - [\pi(f_2)](x) \|^2 dx \ge \lim_{\Lambda \to \infty} \left( \sigma^4 \Lambda^2 \right) = \infty.
	\end{equation}
	This confirms that the empirical log-likelihood objective provides zero containment or regularization for predictions outside the active training support. The learning trajectory is free to drift arbitrarily along the infinite-dimensional flat vertical valleys of the parameter space, rendering generative hallucination mathematically guaranteed, completing the proof.
\end{proof}

We next turn to the problem of continuous learning, demonstrating that this lack of horizontal stabilization also causes {\bf catastrophic forgetting} when a pre-trained Transformer is fine-tuned on a downstream dataset. Let $D_{\text{down}} = \{(x_j, y_j)\}_{j=1}^{n_{\text{down}}}$ be a secondary downstream training sequence drawn from a distinct task distribution supported on $X_{\text{down}} \subset X$, where $X_{\text{pre}} \cap X_{\text{down}} = \emptyset$.
		
\begin{theorem}[The Catastrophic Forgetting Disaster]
			\label{thm:catastrophic_forgetting}
			Let $f_{\text{pre}} \in \mathcal{M}$ be an optimal parameter state achieved after conventional unconstrained pre-training on $D_{\text{pre}}$. Let the model undergo a downstream update step driven by unconstrained gradient descent minimizing the downstream loss $\mathcal{L}_{\text{down}}$ with learning rate $\eta > 0$:
			\begin{equation}
				f_{\text{new}} = f_{\text{pre}} - \eta \cdot \nabla^{\text{nat}} \mathcal{L}_{\text{down}}(f_{\text{pre}}) \in T_{f_{\text{pre}}}\mathcal{M}.
			\end{equation}
			Then, unless the downstream gradient field is perfectly orthogonal to the historical horizontal sub-bundle under the Fisher-Rao metric, the downstream update induces a non-zero projection leakage that destroys the historical macroscopic profile over $X_{\text{pre}}$:
			\begin{equation}
				\int_{X_{\text{pre}}} \| [\pi(f_{\text{new}})](x) - [\pi(f_{\text{pre}})](x) \|^2 dx \ge C \cdot \eta^2 + \mathcal{O}(\eta^3),
			\end{equation}
			where $C > 0$ is a strictly positive performance degradation constant.
		\end{theorem}
		
\begin{proof}
Let $U_{\text{down}} \equiv -\nabla^{\text{nat}} \mathcal{L}_{\text{down}}(f_{\text{pre}}) \in T_{f_{\text{pre}}}\mathcal{M}$ be the unconstrained ambient update vector field evaluated at the pre-trained state. Using the unique split-subspace decomposition enabled by the SMG connection framework (Theorem~\ref{thm:transformer_fiber_bundle}), the ambient update vector field decomposes into its horizontal and vertical components:
			\begin{equation}
				U_{\text{down}} = P_{f_{\text{pre}}}^H U_{\text{down}} + P_{f_{\text{pre}}}^V U_{\text{down}} = U_{\text{down}}^H + V_{\text{down}},
			\end{equation}
			where $U_{\text{down}}^H \in \text{SVD}\chi_{f_{\text{pre}}}$ represents the horizontal direction that alters observable predictive behaviors, and $V_{\text{down}} \in \text{SID}_{f_{\text{pre}}}$ is the vertical component representing internal representation updates.
			
Because conventional pre-training optimizes over the unconstrained ambient weight space without maintaining a connection or running projection memory, the update field $U_{\text{down}}$ is calculated blindly with respect to the historical task metric $g_{f_{\text{pre}}}$. The direction of $U_{\text{down}}$ is dictated exclusively by the geometry of the downstream dataset $D_{\text{down}}$. Since the support of the downstream data $X_{\text{down}}$ is disjoint from the pre-training support $X_{\text{pre}}$, the update vector $U_{\text{down}}$ is generally not orthogonal to the historical horizontal sub-bundle. Thus, its projection onto the horizontal distribution of the pre-training task is strictly non-zero:
			\begin{equation}
				U_{\text{down}}^H = P_{f_{\text{pre}}}^H \left( -\nabla^{\text{nat}} \mathcal{L}_{\text{down}}(f_{\text{pre}}) \right) \neq \mathbf{0}.
			\end{equation}
			We now evaluate the impact of this unconstrained step on the historical predictions over the pre-training domain $X_{\text{pre}}$. Applying the smooth projection submersion map $\pi$ to the updated state $f_{\text{new}} = f_{\text{pre}} + \eta U_{\text{down}} + \mathcal{O}(\eta^2)$ and executing a first-order differential Taylor expansion around $f_{\text{pre}}$ yields:
			\begin{align}
				\pi(f_{\text{new}}) &= \pi(f_{\text{pre}} + \eta U_{\text{down}} + \mathcal{O}(\eta^2)) \nonumber \\
				&= \pi(f_{\text{pre}}) + \eta \cdot d\pi_{f_{\text{pre}}}(U_{\text{down}}) + \mathcal{O}(\eta^2).
			\end{align}
			By the linearity of the differential operator $d\pi$ and the definition of the vertical gauge space ($\text{SID} = \ker(d\pi)$), the vertical component of the update vanishes identically, $d\pi_{f_{\text{pre}}}(V_{\text{down}}) = \mathbf{0}$. This simplifies the predictive evolution to:
			\begin{equation}
				\pi(f_{\text{new}}) = \pi(f_{\text{pre}}) + \eta \cdot d\pi_{f_{\text{pre}}}(U_{\text{down}}^H) + \mathcal{O}(\eta^2).
			\end{equation}
			We compute the integrated $L^2$ degradation of the historical predictive profile over the pre-training domain $X_{\text{pre}}$:
			\begin{align}
				\int_{X_{\text{pre}}} \| [\pi(f_{\text{new}})](x) - [\pi(f_{\text{pre}})](x) \|^2 dx &= \int_{X_{\text{pre}}} \| \eta \cdot [d\pi_{f_{\text{pre}}}(U_{\text{down}}^H)](x) + \mathcal{O}(\eta^2) \|^2 dx \nonumber \\
				&= \eta^2 \int_{X_{\text{pre}}} \| [d\pi_{f_{\text{pre}}}(U_{\text{down}}^H)](x) \|^2 dx + \mathcal{O}(\eta^3).
			\end{align}
			Because $U_{\text{down}}^H$ is a non-zero horizontal vector field relative to the historical task structure, its differential projection onto the base manifold does not vanish: $d\pi_{f_{\text{pre}}}(U_{\text{down}}^H) \neq \mathbf{0}$. We can therefore define the strictly positive performance degradation constant $C$ as the non-zero energy of this projected vector field:
			\begin{equation}
				C \triangleq \int_{X_{\text{pre}}} \| [d\pi_{f_{\text{pre}}}(U_{\text{down}}^H)](x) \|^2 dx > 0.
			\end{equation}
			Substituting this back into the norm equation yields the final lower bound:
			\begin{equation}
				\int_{X_{\text{pre}}} \| [\pi(f_{\text{new}})](x) - [\pi(f_{\text{pre}})](x) \|^2 dx = C \cdot \eta^2 + \mathcal{O}(\eta^3) > 0.
			\end{equation}
			This confirms that an unconstrained ambient update step for a downstream task mathematically leaks into the horizontal distribution of historical tasks. This cross-task interference distorts the previously learned conditional expectations, proving that conventional unconstrained pre-training inherently suffers from catastrophic forgetting, completing the proof.
\end{proof}

The fact that this theorem holds as $\eta \to 0$ delivers a powerful indictment of conventional unconstrained optimization. In classical deep learning literature, practitioners often attribute catastrophic forgetting to numerical instability or taking ``too large of a step'' down a new loss valley, which inadvertently crushes old weight configurations. However, Theorem \ref{thm:catastrophic_forgetting} here proves something much more radical:
\begin{itemize}
	\item Because the performance degradation scales as $C \cdot \eta^2$ for an \textit{infinitesimally small} step ($\eta \to 0$), catastrophic forgetting is {\it not a numerical byproduct of large steps}. 
	\item Instead, it is an {\bf intrinsic structural defect of the optimization direction itself}. 
\end{itemize}
Even if a model takes an infinitely tiny step ($\eta \to 10^{-8}$) along an unconstrained downstream gradient field $\nabla^{\text{nat}} \mathcal{L}_{\text{down}}$, that step is mathematically guaranteed to have a non-zero projection component that leaks directly into the historical horizontal sub-bundle ($\text{SVD}\chi$). Because the direction is blind to the past, {\it erasure of the historical macroscopic profile occurs instantly.}

\begin{remark}
		The asymptotic treatment of the learning rate as a small parameter ($\eta 
		\to 0$) underlines the structural inevitability of the forgetting disaster. 
		Because the lower bound is dominated by the positive quadratic form $C \cdot 
		\eta^2$ even for infinitesimal updates, catastrophic forgetting is revealed 
		to be an intrinsic geometric property of unconstrained cross-task gradient 
		directions, rather than an artifact of large step-sizes.
\end{remark}

\subsection{The Statistically Meaningful Geometry (SMG) Pre-Training Framework}
\label{sec:smg_pretraining}

To eliminate the systemic failure modes of conventional unconstrained optimization, we present the formal mathematical construction of the Statistically Meaningful Geometry (SMG) pre-training methodology. Instead of executing blind updates across the entire unconstrained parameter space, the SMG framework dynamically isolates and updates only the horizontal sub-bundle ($\text{SVD}\chi$). By applying the Ehresmann connection 1-form as {\bf a geometric filter}, the learning trajectory is shielded from vertical gauge noise, guaranteeing out-of-distribution containment and continual learning stability.

We define the SMG-regulated pre-training paradigm over the statistical fiber bundle $(\mathcal{M}_{\text{Trans}}, \mathcal{B}, \pi, \mathcal{V}, \mathcal{H}, \omega_f)$ as follows:

\begin{definition}[SMG Connection-Filtered Pre-Training Dynamics]
	\label{def:smg_pretrain_dynamics}
	Let $\mathcal{L}_{\text{pre}}(f)$ be the empirical pre-training negative log-likelihood functional evaluated over the dataset $D_{\text{pre}}$. The SMG pre-training trajectory is defined as the unique smooth path $f_t \in \mathcal{M}_{\text{Trans}}$ driven by the connection-filtered horizontal natural gradient vector field $U^{\text{SMG}}(f_t) \in \text{SVD}\chi_{f_t}$:
	\begin{equation}
		\dot{f}_t = U^{\text{SMG}}(f_t) \triangleq - P_{f_t}^H \nabla^{\text{nat}} \mathcal{L}_{\text{pre}}(f_t) = - \left( \mathcal{I}_{T_{f_t}\mathcal{M}} - \omega_{f_t} \right) \nabla^{\text{nat}} \mathcal{L}_{\text{pre}}(f_t),
		\label{eq:smg_gradient_flow}
	\end{equation}
	where $\omega_{f_t}: T_{f_t}\mathcal{M} \to \text{SID}_{f_t}$ is the Ehresmann connection 1-form, and $P_{f_t}^H$ is the orthogonal projection operator onto the horizontal tangent sub-bundle $\text{SVD}\chi_{f_t}$ relative to the non-parametric Fisher-Rao metric $g_f$.
\end{definition}

We first prove that restricting the optimization field to the horizontal carriage guarantees that the underlying statistical metric remains strictly positive-definite and non-singular, bypassing the Fisher Information metric degeneracy that plagues over-parameterized neural networks \cite{Watanabe2009}.

\begin{lemma}[Horizontal Metric Regularization and Non-Degeneracy]
	\label{lem:horizontal_non_degeneracy_complete}
	Let $g_f$ be the infinite-dimensional non-parametric Fisher-Rao metric tensor over the Transformer statistical manifold $\mathcal{M}_{\text{Trans}}$. While $g_f$ is singular over the full ambient tangent space $T_f\mathcal{M}$ due to the massive vertical kernel $\text{SID}_f = \ker(d\pi_f)$, its restriction to the horizontal sub-bundle is strictly non-degenerate and satisfies the uniform coercivity condition:
	\begin{equation}
		g_f(h, h) \ge \lambda_{\min} \|h\|_{L^\Phi(P_f)}^2 > 0 \quad \forall h \in \text{SVD}\chi_f \setminus \{\mathbf{0}\},
	\end{equation}
	where $\| \cdot \|_{L^\Phi(P_f)}$ is the Luxemburg norm of the centered Orlicz tangent space, and $\lambda_{\min} > 0$ is a strictly positive lower bound governed by the macroscopic profile distribution on the identifiable base space $\mathcal{B}$.
\end{lemma}

\begin{proof}
	The proof is executed in two fundamental phases: first, establishing strict mathematical non-degeneracy via a measure-theoretic proof by contradiction, and second, establishing the uniform topological coercivity bound using the isometric property of the bundle connection.
	
	\paragraph{Phase 1: Verification of Strict Non-Degeneracy.}
	Let $h_0 \in \text{SVD}\chi_f$ be an arbitrary horizontal tangent vector field. Suppose that the quadratic form evaluated at $h_0$ under the non-parametric Fisher-Rao metric vanishes identically:
	\begin{equation}
		g_f(h_0, h_0) = 0.
	\end{equation}
	By the explicit definition of the non-parametric Fisher-Rao Riemannian metric tensor mapping over the joint input-output covariate domain $U = X_{\text{pre}} \times Y$, this equality takes the integral form:
	\begin{equation}
		g_f(h_0, h_0) = \int_{X_{\text{pre}}} \int_Y h_0(x, y)^2 f(x, y) \, dy dx = 0.
		\label{eq:vanishing_integral}
	\end{equation}
	Observe the properties of the integrand in equation \eqref{eq:vanishing_integral}. Because any real-valued function squared is non-negative ($h_0(x, y)^2 \ge 0$) and the joint probability density function is non-negative ($f(x, y) \ge 0$), their product constitutes a non-negative measurable function over the entire domain:
	\begin{equation}
		h_0(x, y)^2 f(x, y) \ge 0 \quad \forall (x, y) \in X_{\text{pre}} \times Y.
	\end{equation}
	A foundational theorem in Lebesgue integration dictates that if the integral of a non-negative measurable function is zero, then the integrand must vanish almost everywhere with respect to the underlying measure. Applying this principle directly to equation \eqref{eq:vanishing_integral} implies that:
	\begin{equation}
		h_0(x, y)^2 f(x, y) = 0 \quad \mu\text{-almost everywhere on } X_{\text{pre}} \times Y.
		\label{eq:pointwise_product_zero}
	\end{equation}
	
	Let $E \subset X_{\text{pre}} \times Y$ be the exceptional set of Lebesgue measure zero where equation \eqref{eq:pointwise_product_zero} fails to hold. For any coordinate pair $(x, y) \in E^c$ (the complement set where the product is identically zero), the algebraic zero-product property requires that either $h_0(x, y)^2 = 0$ or $f(x, y) = 0$. 
	
	Recall that under the Statistically Meaningful Geometry framework, the joint density function $f(x, y)$ induced by the over-parameterized Transformer model mapping $y = h(w, x)$ is formulated as:
	\begin{equation}
		f(x, y) = p_E(x) \cdot \left( \frac{1}{2\pi \sigma^2} \right)^{\frac{d_y}{2}} \exp\left( - \frac{\|y - h(w, x)\|^2}{2\sigma^2} \right).
	\end{equation}
	Because the exponential map is strictly positive everywhere ($\exp(\cdot) > 0$), and because the environmental input distribution is strictly positive over the active pre-training training support ($p_E(x) > 0$ for all $x \in X_{\text{pre}}$), the joint density function satisfies a strict positivity constraint:
	\begin{equation}
		f(x, y) > 0 \quad \forall (x, y) \in X_{\text{pre}} \times Y.
	\end{equation}
	Because $f(x, y)$ cannot equal zero anywhere within this domain, the condition $f(x, y) = 0$ is impossible on $X_{\text{pre}} \times Y$. Consequently, the zero-product rule forces the horizontal vector component to collapse to zero on $E^c$:
	\begin{equation}
		h_0(x, y)^2 = 0 \implies h_0(x, y) = 0 \quad \forall (x, y) \in E^c.
	\end{equation}
	Since $\mu(E) = 0$, this establishes that the horizontal vector field vanishes almost everywhere over the training support:
	\begin{equation}
		h_0(x, y) = 0 \quad \mu\text{-almost everywhere on } X_{\text{pre}} \times Y.
		\label{eq:h0_vanishes_ae}
	\end{equation}
	
	We now evaluate the differential projection of this vanishing vector field under the smooth submersion map, $d\pi_f(h_0) \in T_{\pi(f)}\mathcal{B}$. The localized action of the differential mapping computes the conditional expectation weighted by the score perturbation:
	\begin{equation}
		[d\pi_f(h_0)](x) = \int_Y (y - h(w, x)) h_0(x, y) f(x, y) \, dy.
	\end{equation}
	Substituting the condition from equation \eqref{eq:h0_vanishes_ae} into this integral means that the integrand is zero $\mu$-almost everywhere, causing the entire integrated first moment to vanish identically across the support:
	\begin{equation}
		[d\pi_f(h_0)](x) = \mathbf{0} \quad \forall x \in X_{\text{pre}}.
	\end{equation}
	By definition, any ambient tangent vector field that maps to $\mathbf{0}$ under the differential submersion $d\pi_f$ belongs strictly to the vertical sub-bundle representing Structural Internal Directions (SID):
	\begin{equation}
		h_0 \in \ker(d\pi_f) \equiv \text{SID}_f.
	\end{equation}
	
	Hence, we have shown that the vector field $h_0$ belongs simultaneously to both the horizontal and vertical distributions:
	\begin{equation}
		h_0 \in \text{SVD}\chi_f \cap \text{SID}_f.
		\label{eq:subspace_intersection}
	\end{equation}
	Recall that the horizontal carriage $\text{SVD}\chi_f$ is defined as the orthogonal complement of the vertical gauge space $\text{SID}_f$ relative to the non-parametric Fisher-Rao metric tensor. In any inner product space, a closed subspace and its orthogonal complement intersect exclusively at the zero vector:
	\begin{equation}
		\text{SVD}\chi_f \cap \text{SID}_f = \{\mathbf{0}\}.
	\end{equation}
	Combined with equation \eqref{eq:subspace_intersection}, this mandates that $h_0 = \mathbf{0}$ identically within $T_f\mathcal{M}_{\text{Trans}}$. This contradiction disproves the existence of any non-zero horizontal vector with a zero eigenvalue, confirming strict non-degeneracy: $g_f(h, h) > 0$ for all $h \in \text{SVD}\chi_f \setminus \{\mathbf{0}\}$.
	
	\paragraph{Phase 2: Derivation of Uniform Coercivity ($\lambda_{\min}$).}
	To elevate this non-degeneracy to the uniform topological lower bound, we analyze the structural link to the base space $\mathcal{B}$. By Theorem~\ref{thm:transformer_fiber_bundle}, the restriction of the differential projection map to the horizontal distribution constitutes a continuous vector space isomorphism:
	\begin{equation}
		d\pi_f \bigr|_{\text{SVD}\chi_f} : \text{SVD}\chi_f \xrightarrow{\sim} T_{\pi(f)}\mathcal{B}.
	\end{equation}
	This map acts as a local isometry with respect to the non-parametric Fisher-Rao metric $g_f$ on $\text{SVD}\chi_f$ and the induced base Fisher information metric $\bar{g}$ on $T_{\pi(f)}\mathcal{B}$, preserving the quadratic forms exactly:
	\begin{equation}
		g_f(h, h) = \bar{g}_{\pi(f)}\left( d\pi_f(h), \, d\pi_f(h) \right) \quad \forall h \in \text{SVD}\chi_f.
		\label{eq:metric_isometry}
	\end{equation}
	
	By construction of the statistical fiber bundle under the {\bf Axiom of the Invariance Principle}, the base space $\mathcal{B}$ represents the quotient space $\mathcal{M}_{\text{Trans}} / \sim$ under the observational equivalence relation. Because all internal non-identifiable parameter directions are filtered out and isolated inside the vertical fibers, the base manifold $\mathcal{B}$ is strictly identifiable, smooth, and statistically regular. Consequently, the classical base Fisher information matrix $\bar{g}_p$ is strictly positive-definite everywhere, and its minimum eigenvalue is strictly bounded away from zero over any compact operational domain $\mathcal{K} \subset \mathcal{B}$:
	\begin{equation}
		\inf_{p \in \mathcal{K}} \lambda_{\min}\left( \bar{g}_p \right) \ge \bar{\lambda} > 0.
	\end{equation}
	Combining this base regularity with the metric identity in equation \eqref{eq:metric_isometry} minorizes the horizontal quadratic form:
	\begin{equation}
		g_f(h, h) \ge \bar{\lambda} \| d\pi_f(h) \|_{L^2(p_E)}^2 \quad \forall h \in \text{SVD}\chi_f.
		\label{eq:base_minorization}
	\end{equation}
	
	Finally, we connect the base $L^2(p_E)$ norm back to the ambient Luxemburg norm $\| \cdot \|_{L^\Phi(P_f)}$ modeled on the centered Orlicz space. Because the conditional density of the Transformer model is regular and Gaussian (Definition~\ref{def:induced_joint_density}), the score functions associated with macroscopic variations possess bounded exponential moments. By the standard topological equivalence properties governing closed horizontal distributions embedded within Pistone-Sempi Orlicz spaces \cite{Pistone1995}, there exists a continuous embedding constant $C_f > 0$ such that the projected base norm controls the ambient topological norm:
	\begin{equation}
		\| d\pi_f(h) \|_{L^2(p_E)}^2 \ge C_f \| h \|_{L^\Phi(P_f)}^2.
	\end{equation}
	Substituting this topological control link directly back into the lower bound in equation \eqref{eq:base_minorization} yields:
	\begin{equation}
		g_f(h, h) \ge \bar{\lambda} \cdot C_f \| h \|_{L^\Phi(P_f)}^2 \triangleq \lambda_{\min} \| h \|_{L^\Phi(P_f)}^2.
	\end{equation}
	Because the base space is invariant to the unconstrained ambient weight scale $W$, the product $\lambda_{\min} = \bar{\lambda} C_f > 0$ remains strictly positive and stable, completing the formal proof.
\end{proof}

Using this horizontal regularization, we now prove that the SMG pre-training framework completely immunizes the model against the out-of-distribution predictive variance divergence derived in Theorem~\ref{thm:ambient_hallucination}, effectively solving the generative hallucination problem.
		
\begin{theorem}[Geometric Containment of Generative Hallucinations]
			\label{thm:smg_hallucination_containment}
			Let $f_t^{\text{SMG}}$ be an optimal pre-training learning trajectory driven by the connection-filtered horizontal dynamics (Definition~\ref{def:smg_pretrain_dynamics}). Let $X_{\text{OOD}} \subset X \setminus X_{\text{pre}}$ be an arbitrary out-of-distribution evaluation domain. Then, the predictive variance of the horizontal state family is strictly bounded by the geometric capacity of the base manifold $\mathcal{B}$:
			\begin{equation}
				\sup_{t \ge 0} \int_{X_{\text{OOD}}} \| [\pi(f_t^{\text{SMG}})](x) - [\pi(f_0))](x) \|^2 dx \le K_{\mathcal{B}} < \infty,
			\end{equation}
			where $K_{\mathcal{B}} > 0$ is a finite geometric constant invariant to the ambient over-parameterized weight dimension $W$.
\end{theorem}
		
		\begin{proof}
			Let $U^{\text{SMG}}(f) = -P_f^H \nabla^{\text{nat}}\mathcal{L}_{\text{pre}}(f)$ be the regularized learning field. By construction, the application of the horizontal projection operator enforces that the vector field contains zero components along the vertical gauge direction:
			\begin{equation}
				\omega_f\left( U^{\text{SMG}}(f) \right) = \mathbf{0} \quad \forall f \in \mathcal{M}_{\text{Trans}}.
			\end{equation}
			Consequently, the trajectory $f_t^{\text{SMG}}$ is entirely orthogonal to the flat vertical valleys of the empirical risk landscape under the Fisher-Rao metric.
			
			We project the continuous horizontal trajectory down to the base space via the smooth submersion map $\pi$. The velocity vector of the projected profile on the identifiable base manifold $\mathcal{B}$ satisfies the differential chain rule:
			\begin{equation}
				\dot{p}_t = d\pi_{f_t^{\text{SMG}}}\left( \dot{f}_t^{\text{SMG}} \right) = d\pi_{f_t^{\text{SMG}}}\left( -P_{f_t^{\text{SMG}}}^H \nabla^{\text{nat}}\mathcal{L}_{\text{pre}}(f_t^{\text{SMG}}) \right).
			\end{equation}
			Because the restriction of the differential map $d\pi_f\bigr|_{\text{SVD}\chi_f}: \text{SVD}\chi_f \to T_{\pi(f)}\mathcal{B}$ is a vector space isomorphism that preserves the inner product structure exactly, this projected velocity maps identically to the classical natural gradient flow calculated directly on the identifiable, non-singular base space:
			\begin{equation}
				\dot{p}_t = - \bar{g}_{p_t}^{-1} \nabla \bar{\ell}_n(p_t),
			\end{equation}
			where $\bar{\ell}_n$ is the induced regularized log-likelihood function on $\mathcal{B}$ and $\bar{g}$ is the strictly positive-definite base Fisher metric.
			
			Because the base manifold $\mathcal{B}$ represents the identifiable space of regression functions mapping $X \to Y$, it is a smooth, regular, low-dimensional manifold whose topology is fully constrained by the sample dataset size $n$ and the base structural dimensions, completely independent of the ambient weight scale $W$. By standard compact embedding theorems for regular statistical manifolds \cite{Amari2000}, the optimization path $p_t$ is contained within a compact geodesic ball $\mathcal{K} \subset \mathcal{B}$.
			
			The integrated out-of-distribution predictive distance evaluates the functional norm on the base manifold:
			\begin{equation}
				\int_{X_{\text{OOD}}} \| [\pi(f_t^{\text{SMG}})](x) - [\pi(f_0)](x) \|^2 dx = \| p_t - p_0 \|_{L^2(X_{\text{OOD}})}^2.
			\end{equation}
			Because the path $p_t$ is strictly confined to the compact subset $\mathcal{K}$ on the identifiable base space, the maximum spatial distance from the initialization profile $p_0$ is bounded by the diameter of the base geodesic domain:
			\begin{equation}
				\sup_{t \ge 0} \| p_t - p_0 \|_{L^2(X_{\text{OOD}})}^2 \le \text{diam}_g(\mathcal{K}) \equiv K_{\mathcal{B}} < \infty.
			\end{equation}
			This confirms that filtering the updates through the Ehresmann connection blocks the model from executing unconstrained parameter drift outside the data support. The predictive variance is geometrically contained, rendering the advanced SMG pre-training process structurally immune to generative hallucinations, completing the proof.
\end{proof}

\subsubsection{Theoretical Exposition and Intuitive Insights into Theorem~\ref{thm:smg_hallucination_containment}}

\texttt{Theorem~\ref{thm:smg_hallucination_containment}} (\textbf{Geometric Containment of Generative Hallucinations}) is one of the most foundational and revolutionary pillars of the Statistically Meaningful Geometry (SMG) framework. It establishes a hard, non-asymptotic mathematical methodology to systematically solve 
{\bf the generative hallucination problem} in over-parameterized neural models by replacing empirical heuristics with topological constraints.

\subsubsection*{A. Deconstruction of the Mathematical Statement}

The core inequality of the theorem is expressed as:
\begin{equation}
	\sup_{t \ge 0} \int_{X_{\text{OOD}}} \| [\pi(f_t^{\text{SMG}})](x) - [\pi(f_0)](x) \|^2 dx \le K_{\mathcal{B}} < \infty.
\end{equation}

Breaking down these functional elements exposes the clean structural boundaries of the framework:
\begin{itemize}
	\item $\mathbf{\int_{X_{\text{OOD}}} (\cdot) \, dx}$: This evaluates the model's macroscopic predictive change explicitly over an out-of-distribution (OOD) domain $X_{\text{OOD}}$. In classical Empirical Risk Minimization (ERM), because no data samples exist here, the objective function exerts exactly zero pulling force, allowing conventional paths to drift infinitely.
	\item $\mathbf{\pi(f_t^{\text{SMG}})}$: The canonical submersion map $\pi$ strips away the internal weight configurations (the unidentifiable representations) and isolates the deterministic regression backbone on the base space $\mathcal{B}$. This means we are measuring actual observable output variance, not internal weight shifts.
	\item $\mathbf{\sup_{t \ge 0}}$: The bound is completely invariant to time. Whether the model is pre-trained for ten steps or ten million steps down the optimization path, its out-of-distribution predictions can never drift past the geometric ceiling $K_{\mathcal{B}}$.
	\item $\mathbf{K_{\mathcal{B}} < \infty \text{ and Invariance to } W}$: This is the critical breakthrough. In standard statistical learning theory (e.g., Rademacher complexity or VC-dimension), generalization bounds depend explicitly on the capacity of the ambient parameter space, meaning that as the weight dimension scales ($W \to \infty$), the uniform bounds diverge to infinity. Here, $K_{\mathcal{B}}$ is governed solely by the geometric diameter of the identifiable quotient manifold $\mathcal{B}$, rendering the bound completely scale-free and invariant to over-parameterization.
\end{itemize}

\subsubsection*{B. The Geometric Mechanism: The Connection as a Guide Rail}

When a conventional Transformer minimizes a standard unconstrained log-likelihood loss, the gradient field experiences a gauge symmetry break. Because the network possesses a massive excess of internal degrees of freedom (IDOF), the empirical loss landscape develops flat vertical valleys—or gauge fibers ($\text{SID}$). Since the training data provides no restoration force along these fibers, stochastic optimization noise causes the parameter state to drift unconstrained along the vertical directions. While this vertical drift maintains zero loss on the training data ($X_{\text{pre}}$), it warps the internal mapping globally, causing the model to generate wild, erratic, and unbounded predictions when evaluated on out-of-distribution inputs ($X_{\text{OOD}}$). This unconstrained vertical drift is the exact mathematical origin of {\bf generative hallucination}.

The SMG framework completely bypasses this vulnerability by introducing the Ehresmann connection 1-form $\omega$ as a dynamic geometric filter. By computing the horizontal projection:
\begin{equation}
	U^{\text{SMG}}(f) = -(\mathcal{I} - \omega_f)\nabla^{\text{nat}}\mathcal{L}(f),
\end{equation}
the optimization vector field is explicitly forced to be orthogonal to the vertical gauge fibers under the non-parametric Fisher-Rao metric. The connection acts as a rigid physical guide rail: it captures the raw ambient update vector, strips away any components that would induce internal representation drift or parameter inflation ($\text{SID}$), and isolates only the horizontal learning components ($\text{SVD}\chi$).

Because the horizontal learning trajectory is entirely cleansed of vertical gauge noise, its projection onto the base manifold $\mathcal{B}$ maps identically to a regular natural gradient flow operating natively on a strictly identifiable, non-singular space. Since $\mathcal{B}$ represents only the mathematically identifiable input-output relationships, its global metric capacity and geodesic diameter are intrinsically finite and fully determined by the macroscopic structure of the problem, completely decoupled from the unconstrained parameter dimension $W$.

\subsubsection*{C. Insights from this theorem}

\begin{enumerate}
	\item \textbf{Resolution of the Over-Parameterization Paradox:} 
	This theorem provides a pristine geometric resolution to the primary paradox of modern deep learning—why massive over-parameterized models manage to generalize well despite possessing the mathematical capacity to fit pure noise. Theorem~\ref{thm:smg_hallucination_containment} proves that if a system is horizontally stabilized via connection filtering, over-parameterization ceases to be a liability. The excess weights are confined to acting as a regularized fluid reservoir that facilitates smooth optimization paths, while the actual statistical complexity remains strictly bottlenecked by the regular geometry of $\mathcal{B}$.
	
	\item \textbf{A Transition from Probability to Geometry:} 
	Traditional alignment techniques (such as RLHF, guardrails, or safety fine-tuning) attempt to prevent hallucinations by layering probabilistic patches over the model outputs. These patches are fundamentally fragile because they do not alter the underlying unconstrained optimization direction. The SMG framework shifts this paradigm entirely. By filtering the learning field through the connection, hallucination containment is achieved {\it by architectural design}. The model cannot hallucinate because it is structurally incapable of executing the vertical parameters shifts required to generate unanchored predictions.
	
	\item \textbf{Scale-Free Predictive Stability:} 
	The invariance of $K_{\mathcal{B}}$ to $W$ implies that this geometric containment holds true even if the model scales to a trillion, a quadrillion, or a non-parametric infinity of weights ($W \to \infty$). This provides a mathematically rigorous foundation for building absolutely reliable large language models and biological sequence architectures, guaranteeing that extreme scale will never cause the system's out-of-distribution behaviors to explode or diverge.
\end{enumerate}

In short, \texttt{Theorem~\ref{thm:smg_hallucination_containment}} proves that {\bf hallucination is a choice of optimization geometry, not an inherent trait of deep learning}. By filtering out vertical gauge noise through the horizontal connection, the model's out-of-distribution predictions are securely anchored to the identifiable base space, replacing empirical uncertainty with strict geometric containment.

\subsection{Total Elimination of Catastrophic Forgetting via Horizontal Connection Alignment}
\label{sec:forgetting_elimination}

The mathematical proofs in Section~\ref{sec:pretrain_disasters} demonstrated that catastrophic forgetting is not an operational anomaly caused by large learning step-sizes, but rather an intrinsic geometric defect resulting from the cross-task interference of unconstrained gradients. When an over-parameterized model adapts to a downstream task blindly, its update vector field leaks into the horizontal distribution ($\text{SVD}\chi_{\text{pre}}$) of historical tasks, corrupting previously learned conditional expectations.

The Statistically Meaningful Geometry (SMG) framework resolves this crisis by utilizing the orthogonal decomposition of the Orlicz tangent bundle. Because the Ehresmann connection 1-form $\omega$ isolates vertical gauge variations from horizontal statistical variations, we can formulate a sequentially constrained adaptation path. By forcing downstream updates to reside within the kernel of the historical horizontal distribution, we can completely insulate past macroscopic profiles from downstream learning cycles (continuous learning).

We formalize this stable continual learning paradigm by defining the SMG Sequential Adaptation Flow.

\begin{definition}[SMG Sequential Adaptation Flow]
	\label{def:smg_sequential_adaptation}
	Let $f_{\text{pre}} \in \mathcal{M}_{\text{Trans}}$ be the statistical state of a Transformer model after pre-training on the historical dataset $D_{\text{pre}}$, with an associated horizontal sub-bundle $\text{SVD}\chi_{\text{pre}}$. Let $\mathcal{L}_{\text{down}}(f)$ be the empirical negative log-likelihood functional of a downstream task evaluated over the dataset $D_{\text{down}}$. The SMG Sequential Adaptation update vector field $U^{\text{SMG}}_{\text{down}}(f) \in T_f\mathcal{M}$ is defined as the unique orthogonal projection of the downstream ambient natural gradient onto the downstream horizontal carriage that is simultaneously constrained to lie within the orthogonal complement of the historical horizontal distribution:
	\begin{equation}
		U^{\text{SMG}}_{\text{down}}(f) \triangleq - P_{f}^H \left[ \mathcal{I}_{T_f\mathcal{M}} - \Pi_{\text{pre}}^H \right] \nabla^{\text{nat}} \mathcal{L}_{\text{down}}(f),
		\label{eq:smg_continual_flow}
	\end{equation}
	where $\Pi_{\text{pre}}^H: T_f\mathcal{M} \to \text{SVD}\chi_{\text{pre}}$ represents the orthogonal projection operator onto the historical horizontal tangent subspace relative to the historical Fisher-Rao metric $g_{\text{pre}}$.
\end{definition}

\subsubsection*{What is the SMG Sequential Adaptation Flow in definition \ref{def:smg_sequential_adaptation}}

\texttt{Definition~\ref{def:smg_sequential_adaptation}} establishes the geometric engine for {\it continual learning without catastrophic forgetting} in the Statistically Meaningful Geometry (SMG) framework. It provides the exact mathematical prescription for how an over-parameterized model can absorb new downstream tasks without erasing historical knowledge.

Instead of allowing a model to take a blind update step in the total parameter space, this definition constructs a {\it geometrically filtered vector field}  $U^{\text{SMG}}_{\text{down}}(f)$ that routes downstream learning exclusively through the "blind spots" (vertical gauge fibers) of the historical task.

\subsubsection*{A. Deconstruction of the Equation}

The definition defines the regularized update vector field via the operator composition:
\begin{equation}
	U^{\text{SMG}}_{\text{down}}(f) = - \underbrace{P_{f}^H}_{\text{3. Downstream Alignment}} \; \underbrace{\left[ \mathcal{I}_{T_f\mathcal{M}} - \Pi_{\text{pre}}^H \right]}_{\text{2. Historical Shield}} \; \underbrace{\nabla^{\text{nat}} \mathcal{L}_{\text{down}}(f)}_{\text{1. Raw Downstream Drive}}
\end{equation}

Reading this operator chain from right to left reveals the exact chronological sequence of {\bf the geometric filtering}:

\begin{enumerate}
	\item \textbf{The Raw Downstream Drive ($\nabla^{\text{nat}} \mathcal{L}_{\text{down}}$):} This is the standard ambient natural gradient field computed purely to minimize the downstream task's empirical loss $\mathcal{L}_{\text{down}}$. Left unconstrained, this vector field points blindly across the total space, completely unaware of historical tasks.
	
	\item \textbf{The Historical Shield / Orthogonalization ($\mathcal{I} - \Pi_{\text{pre}}^H$):} This is the key preventative mechanism. The operator $\Pi_{\text{pre}}^H$ projects a vector onto the historical horizontal sub-bundle ($\text{SVD}\chi_{\text{pre}}$)—the exact subspace responsible for altering observable predictions on the pre-training task. By subtracting this projection from the identity operator ($\mathcal{I} - \Pi_{\text{pre}}^H$), we compute the {\it orthogonal complement}. This completely strips away any component of the downstream gradient that would interfere with past memories, forcing the remaining vector to sit entirely within the historical vertical gauge space ($\text{SID}_{\text{pre}} \equiv \ker(d\pi_{\text{pre}})$).
	
	\item \textbf{The Downstream Horizontal Alignment ($P_f^H$):} Once the vector is made safe for the past, it must be made efficient for the future. The connection-based horizontal projector $P_f^H = \mathcal{I} - \omega_f$ filters what remains of the vector, ensuring it aligns cleanly with the active horizontal learning carriage ($\text{SVD}\chi_{\text{down}}$) of the {\bf new} downstream task, preventing the update from wasting energy on downstream vertical gauge noise.
\end{enumerate}

\subsubsection*{B. The Core Geometric Aim of the Definition}

The deep insight behind this definition lies in exploiting the structural degeneracy (over-parameterization) of the model as a strict mathematical asset. 

Traditional continual learning methods (like EWC or weight masking) try to protect old tasks by locking or freezing specific parametric weights ($w$). This is highly restrictive and rapidly exhausts the capacity of the network. The SMG framework recognizes that we do not need to preserve individual {\it weights}; we only need to preserve the {\it macroscopic predictive profiles} ($\pi(f)$) over the historical support $X_{\text{pre}}$.

Because the model is massively over-parameterized, the historical task possesses an infinite-dimensional vertical gauge space ($\text{SID}_{\text{pre}}$) of directions that leave past behaviors completely invariant. The operator configuration $\left[ \mathcal{I} - \Pi_{\text{pre}}^H \right]$ acts as a coordinate-free geometric sieve. It identifies these historical "neutral zones" and confines the downstream updates to operate {\it strictly inside them}. 

As a consequence, the downstream adaptation flow updates the internal structural representations without inducing even an infinitesimal change in the historical input-output expectations. The model learns the new task by utilizing parameter trajectories that the historical task is completely blind to, achieving perfect coexistence without structural leakage.

By engineering the update vector field to respect the historical geometric boundaries of the bundle, the SMG framework achieves perfect preservation of past knowledge bases. We establish this revolutionary capability through the following core theorem.

\begin{theorem}[Total Elimination of Catastrophic Forgetting]
	\label{thm:total_elimination_forgetting}
	Let $f_{\text{pre}} \in \mathcal{M}_{\text{Trans}}$ be the pre-trained model state under the historical distribution $P_{\text{pre}}$ supported on $X_{\text{pre}}$. Let the model undergo a downstream task adaptation step via a discrete update step of learning rate $\eta > 0$ driven by the SMG Sequential Adaptation Flow (Definition~\ref{def:smg_sequential_adaptation}):
	\begin{equation}
		f_{\text{new}} = f_{\text{pre}} + \eta \cdot U^{\text{SMG}}_{\text{down}}(f_{\text{pre}}).
	\end{equation}
	Then, the performance degradation constant collapses identically to zero ($C = 0$), and the historical macroscopic predictive profile over the entire pre-training domain $X_{\text{pre}}$ is perfectly preserved for any arbitrary adaptation step-size $\eta \in (0, \infty)$:
	\begin{equation}
		\int_{X_{\text{pre}}} \| [\pi(f_{\text{new}})](x) - [\pi(f_{\text{pre}})](x) \|^2 dx = 0.
	\end{equation}
\end{theorem}

\begin{proof}
	We proceed via a direct geometric proof over the sub-bundle structures, demonstrating that the design of the sequential adaptation operator shields the historical conditional expectations from downstream update dynamics.
	
	Let $U_{\text{down}} \equiv -\nabla^{\text{nat}}\mathcal{L}_{\text{down}}(f_{\text{pre}}) \in T_{f_{\text{pre}}}\mathcal{M}$ be the raw unconstrained downstream gradient field. According to Definition~\ref{def:smg_sequential_adaptation}, the SMG regularized update vector field evaluated at the pre-trained state is given by:
	\begin{equation}
		U^{\text{SMG}}_{\text{down}}(f_{\text{pre}}) = P_{f_{\text{pre}}}^H \left( \mathcal{I}_{T_{f_{\text{pre}}}\mathcal{M}} - \Pi_{\text{pre}}^H \right) U_{\text{down}}.
	\end{equation}
	We expand this expression by distributing the horizontal projection operator $P_{f_{\text{pre}}}^H \equiv \mathcal{I}_{T_{f_{\text{pre}}}\mathcal{M}} - \omega_{f_{\text{pre}}}$ over the bracketed terms:
	\begin{equation}
		U^{\text{SMG}}_{\text{down}}(f_{\text{pre}}) = P_{f_{\text{pre}}}^H U_{\text{down}} - P_{f_{\text{pre}}}^H \left( \Pi_{\text{pre}}^H U_{\text{down}} \right).
		\label{eq:expanded_smg_update}
	\end{equation}
	By invoking the unique split-subspace property established in Theorem~\ref{thm:transformer_fiber_bundle}, the unconstrained downstream tangent vector field $U_{\text{down}}$ can be written as the direct sum of its horizontal and vertical components relative to the pre-training task structure:
	\begin{equation}
		U_{\text{down}} = \Pi_{\text{pre}}^H U_{\text{down}} + \Pi_{\text{pre}}^V U_{\text{down}},
	\end{equation}
	where $\Pi_{\text{pre}}^H U_{\text{down}} \in \text{SVD}\chi_{\text{pre}}$ is the component that alters the pre-training task's observable predictive behavior, and $\Pi_{\text{pre}}^V U_{\text{down}} \in \text{SID}_{\text{pre}} \equiv \ker(d\pi_{f_{\text{pre}}})$ is the vertical gauge noise component that leaves historical predictions invariant.
	
	Substituting this direct sum decomposition back into the second term of equation \eqref{eq:expanded_smg_update}, and using the property that the historical horizontal projection operator acts as the identity mapping on its own subspace ($\Pi_{\text{pre}}^H h = h$ for all $h \in \text{SVD}\chi_{\text{pre}}$), we find:
	\begin{equation}
		\Pi_{\text{pre}}^H U_{\text{down}} = U_{\text{down}} - \Pi_{\text{pre}}^V U_{\text{down}}.
	\end{equation}
	Substituting this functional identity back into equation \eqref{eq:expanded_smg_update} simplifies the regularized update field:
	\begin{align}
		U^{\text{SMG}}_{\text{down}}(f_{\text{pre}}) &= P_{f_{\text{pre}}}^H U_{\text{down}} - P_{f_{\text{pre}}}^H \left( U_{\text{down}} - \Pi_{\text{pre}}^V U_{\text{down}} \right) \nonumber \\
		&= P_{f_{\text{pre}}}^H U_{\text{down}} - P_{f_{\text{pre}}}^H U_{\text{down}} + P_{f_{\text{pre}}}^H \left( \Pi_{\text{pre}}^V U_{\text{down}} \right) \nonumber \\
		&= P_{f_{\text{pre}}}^H \left( \Pi_{\text{pre}}^V U_{\text{down}} \right).
		\label{eq:final_smg_direction}
	\end{align}
	Equation \eqref{eq:final_smg_direction} reveals {\it a profound geometric property}: the SMG Sequential Adaptation Flow transforms the downstream update field such that its operational direction at the pre-trained state consists exclusively of the horizontal projection of a purely vertical historical vector field $\Pi_{\text{pre}}^V U_{\text{down}} \in \text{SID}_{\text{pre}}$.
	
	We now evaluate the impact of this regularized step on the historical macroscopic predictive profile over the pre-training domain $X_{\text{pre}}$ under an arbitrary learning step $\eta > 0$. We apply the canonical smooth projection map $\pi$ to the updated state $f_{\text{new}}$:
	\begin{equation}
		\pi(f_{\text{new}}) = \pi\left( f_{\text{pre}} + \eta \cdot U^{\text{SMG}}_{\text{down}}(f_{\text{pre}}) \right).
	\end{equation}
	Because the update field is linear over the Orlicz tangent space operations, executing the full differential map expansion yields:
	\begin{equation}
		\pi(f_{\text{new}}) = \pi(f_{\text{pre}}) + \int_0^\eta d\pi_{f_\tau} \left( U^{\text{SMG}}_{\text{down}}(f_{\text{pre}}) \right) d\tau,
		\label{eq:integral_trajectory_expansion}
	\end{equation}
	where $f_\tau = f_{\text{pre}} + \tau \cdot U^{\text{SMG}}_{\text{down}}(f_{\text{pre}})$ represents the intermediate states along the discrete step trajectory.
	
	Let us analyze the integrand $d\pi_{f_\tau} \left( U^{\text{SMG}}_{\text{down}}(f_{\text{pre}}) \right)$ explicitly over the pre-training input domain $x \in X_{\text{pre}}$. Substituting our simplified formulation from equation \eqref{eq:final_smg_direction} into the differential operator gives:
	\begin{equation}
		d\pi_{f_\tau} \left( U^{\text{SMG}}_{\text{down}}(f_{\text{pre}}) \right) = d\pi_{f_\tau} \left( P_{f_{\text{pre}}}^H \left( \Pi_{\text{pre}}^V U_{\text{down}} \right) \right).
	\end{equation}
	Recall that the horizontal projection operator is defined via the Ehresmann connection 1-form as $P_{f_{\text{pre}}}^H = \mathcal{I}_{T_{f_{\text{pre}}}\mathcal{M}} - \omega_{f_{\text{pre}}}$. Substituting this definition back into the expression yields:
	\begin{equation}
		d\pi_{f_\tau} \left( U^{\text{SMG}}_{\text{down}}(f_{\text{pre}}) \right) = d\pi_{f_\tau} \left( \Pi_{\text{pre}}^V U_{\text{down}} - \omega_{f_{\text{pre}}}\left( \Pi_{\text{pre}}^V U_{\text{down}} \right) \right).
	\end{equation}
	By the linear properties of the differential projection submersion map $d\pi$, we split the expression into separate functional evaluations:
	\begin{equation}
		d\pi_{f_\tau} \left( U^{\text{SMG}}_{\text{down}}(f_{\text{pre}}) \right) = d\pi_{f_\tau} \left( \Pi_{\text{pre}}^V U_{\text{down}} \right) - d\pi_{f_\tau} \left( \omega_{f_{\text{pre}}}\left( \Pi_{\text{pre}}^V U_{\text{down}} \right) \right).
		\label{eq:split_differential_eval}
	\end{equation}
	We evaluate the first term on the right-hand side of equation \eqref{eq:split_differential_eval}. By definition, the vector field $\Pi_{\text{pre}}^V U_{\text{down}}$ resides entirely within the historical vertical gauge space $\text{SID}_{\text{pre}}$, which is defined precisely as the kernel of the historical differential map:
	\begin{equation}
		\text{SID}_{\text{pre}} \equiv \ker\left( d\pi_{f_{\text{pre}}} \right).
	\end{equation}
	Therefore, evaluating this vertical vector field over the historical support $X_{\text{pre}}$ yields identically zero variance variation across the entire trajectory:
	\begin{equation}
		\left[ d\pi_{f_\tau} \left( \Pi_{\text{pre}}^V U_{\text{down}} \right) \right](x) = \mathbf{0} \quad \forall x \in X_{\text{pre}}, \; \forall \tau \in [0, \eta].
	\end{equation}
	We now evaluate the second term on the right-hand side of equation \eqref{eq:split_differential_eval}. The connection 1-form $\omega_{f_{\text{pre}}}$ maps any ambient tangent vector directly into the vertical sub-bundle $\text{SID}_{\text{pre}}$. Because $\Pi_{\text{pre}}^V U_{\text{down}}$ is already a purely vertical vector, the connection acts on it as the identity mapping, implying:
	\begin{equation}
		\omega_{f_{\text{pre}}}\left( \Pi_{\text{pre}}^V U_{\text{down}} \right) = \Pi_{\text{pre}}^V U_{\text{down}} \in \text{SID}_{\text{pre}}.
	\end{equation}
	Substituting this functional reduction back into the second differential term yields:
	\begin{equation}
		\left[ d\pi_{f_\tau} \left( \omega_{f_{\text{pre}}}\left( \Pi_{\text{pre}}^V U_{\text{down}} \right) \right) \right](x) = \left[ d\pi_{f_\tau} \left( \Pi_{\text{pre}}^V U_{\text{down}} \right) \right](x) = \mathbf{0} \quad \forall x \in X_{\text{pre}}.
	\end{equation}
	Substituting both vanishing terms back into the split differential equation \eqref{eq:split_differential_eval} confirms that the integrand collapses to zero identically across the entire update horizon:
	\begin{equation}
		\left[ d\pi_{f_\tau} \left( U^{\text{SMG}}_{\text{down}}(f_{\text{pre}}) \right) \right](x) = \mathbf{0} - \mathbf{0} = \mathbf{0} \quad \forall x \in X_{\text{pre}}, \; \forall \tau \in [0, \eta].
	\end{equation}
	Finally, substituting this vanishing derivative back into the continuous trajectory expansion equation \eqref{eq:integral_trajectory_expansion} isolates the macro-level profile states:
	\begin{equation}
		[\pi(f_{\text{new}})](x) = [\pi(f_{\text{pre}})](x) + \int_0^\eta \mathbf{0} \, d\tau = [\pi(f_{\text{pre}})](x) \quad \forall x \in X_{\text{pre}}.
	\end{equation}
	We evaluate the integrated $L^2$ performance degradation norm over the pre-training task support $X_{\text{pre}}$ using this strict functional identity:
	\begin{equation}
		\int_{X_{\text{pre}}} \| [\pi(f_{\text{new}})](x) - [\pi(f_{\text{pre}})](x) \|^2 dx = \int_{X_{\text{pre}}} \| \mathbf{0} \|^2 dx = 0.
	\end{equation}
	Because this equality holds non-asymptotically for any value of the adaptation step $\eta \in (0, \infty)$, the degradation constant drops identically to zero ($C = 0$). This confirms that filtering downstream updates through the Ehresmann connection completely blocks cross-task interference, achieving the total elimination of catastrophic forgetting and establishing the architectural superiority of the SMG framework, completing the formal proof.
\end{proof}

\subsection{Summary}

The geometric transformations and proofs established throughout Section~8 represent a revolutionary departure from classical deep learning heuristics. By lifting the deterministic Transformer parameterization $y = h(w,x)$ into an infinite-dimensional Orlicz statistical manifold $\mathcal{M}$, the SMG framework exposes the root cause of generative instability. 

We have rigorously proved that traditional unconstrained pre-training algorithms are fundamentally blind to the underlying fiber bundle geometry. This blindness allows optimization paths to drift unguided along flat vertical gauge valleys, creating a mathematical guarantee of generative hallucination outside the data support (Theorem~\ref{thm:ambient_hallucination}) and catastrophic forgetting during downstream adaptation (Theorem~\ref{thm:catastrophic_forgetting}). 

By contrast, the advanced SMG pre-training methodology introduces the Ehresmann connection 1-form $\omega$ as a dynamic geometric filter. By forcing learning trajectories to align strictly with the horizontal distribution ($\text{SVD}\chi$), the SMG paradigm maintains a strictly positive-definite, non-singular information metric (Lemma~\ref{lem:horizontal_non_degeneracy_complete}). This horizontal stabilization guarantees the strict spatial containment of out-of-distribution predictive variance (Theorem~\ref{thm:smg_hallucination_containment}) and completely immunizes the network against historical knowledge erasure (Theorem~\ref{thm:total_elimination_forgetting}), providing a complete geometric solution to the over-parameterization crisis.

\section{Conclusion and Future Research Directions}
\label{sec:conclusion}

The formulation of Statistically Meaningful Geometry (SMG) establishes a fundamental paradigm shift that resolves the long-standing ``intractability barrier'' in non-parametric information geometry and deep learning theory \cite{Amari2000, Pistone1995}. For decades, classical statistics and mathematical information geometry operated under the implicit assumption of flat, locally Euclidean coordinate systems, viewing the non-identifiability of over-parameterized architectures as a severe analytical liability. By reframing the optimization landscape as an infinite-dimensional Orlicz total manifold $\mathcal{M}$ constructed via the Pistone-Sempi formulation, SMG transforms this structural redundancy into an elegant differential-geometric asset \cite{Pistone1995, ChengTong2026}.

Through the explicit rejection of artificially constrained principal fiber bundles and Lie groups, this framework builds a rigorous general fiber bundle structure $(\mathcal{M}, \mathcal{B}, \pi, \mathcal{V}, \mathcal{H})$ governed entirely by a metric-compatible Ehresmann connection. This architecture provides a coordinate-free decoupling of parameter variations at the tangent level, separating the unobservable structural internal variations from the observable statistical changes on the base manifold.

\begin{enumerate}
	\item \textbf{Geometric Decoupling:} We have rigorously defined the vertical gauge space representing the Structural Internal Directions (SID) as the closed vertical sub-bundle $\mathcal{V}_f = \ker(d\pi_f)$, isolating the informationally degenerate, non-identifiable parameter paths \cite{Watanabe2009}. Conversely, the horizontal distribution is formalized as the Statistically Verifiable Directions ($\text{SVD}\chi$), proving that it acts as the unique, minimal carrier that perfectly preserves the observable Fisher information norm \cite{ChengTong2026}.
	\item \textbf{Integrability and Optimization Hierarchy:} By evaluating the Frobenius integrability condition via the Lie bracket of non-parametric score vector fields, we demonstrated that the vanishing of the Ehresmann curvature allows the horizontal distribution to smoothly integrate into nested horizontal leaves. This geometric locus anchors the multi-path optimization system, proving that the horizontal leaf learning pathway serves as the indispensable geometric bridge between unconstrained ambient horizontal flows and identifiable base space macroscopic dynamics.
	\item \textbf{Generalization and Stability Preservation:} Through a non-parametric PAC-Bayesian analysis, we established that restricting optimization updates to the horizontal carriage collapses the model's effective generalization capacity from an infinite parameter count down to the well-behaved, finite metric volume of the base manifold $\mathcal{B}$. Finally, we demonstrated that this connection filtering provides absolute geometric protection against the structural failures of sequential learning, offering a unified mathematical solution to eliminate catastrophic forgetting and prevent representation drift in over-parameterized architectures.
\end{enumerate}

\subsection{Revolutionary Application Roadmap: Solving Tough Engineering Challenges}

The true validation of SMG lies in its capacity to convert abstract topological invariants into concrete algorithmic filters that resolve the structural crises of modern artificial intelligence. We frame the future expansion of SMG theory as a revolutionary research roadmap addressing three tough, application-oriented engineering challenges:

\subsubsection{Challenge 1: Absolute Hallucination Containment in Long-Context Foundation Models}
Modern large language models suffer from a fundamental failure of structural reliability outside their training support, generating plausible but ungrounded assertions known as hallucinations \cite{Zhang2021}. Standard alignment solutions (such as reinforcement learning from human feedback or external guardrails) operate strictly at the empirical surface level, treating the symptom rather than the geometric cause. 

The SMG roadmap introduces the Ehresmann connection filter directly into the backpropagation loop of token-level optimization. By calculating the horizontal projection $P_f^H = \mathcal{I} - \omega_f$, the raw sequence gradients are projectively stripped of their vertical gauge noise. Theorem~\ref{thm:smg_hallucination_containment} guarantees that the out-of-distribution predictive variance is strictly upper-bounded by the finite geodesic diameter of the identifiable quotient base manifold $\mathcal{B}$, completely independent of the parameter scale. Future engineering must focus on developing low-rank tensor approximations of the connection 1-form $\omega_f$ to achieve real-time, scale-free hallucination containment in trillion-weight architectures.

\subsubsection{Challenge 2: Non-Asymptotic Continuous Adaptation Without Memory Erasure}
In continual learning paradigms, deep networks exhibit catastrophic forgetting; adapting to a new downstream task inherently erases or distorts the representations learned during previous tasks \cite{Kirkpatrick2017}. This vulnerability stems from the cross-task interference of unconstrained parameter updates.

The SMG framework solves this crisis by enforcing the Sequential Adaptation Flow defined in Definition~\ref{def:smg_sequential_adaptation}. By mathematically constraining downstream optimization fields to lie strictly within the orthogonal complement of the historical horizontal distribution ($\mathcal{I} - \Pi_{\text{pre}}^H$), the model utilizes the infinite-dimensional vertical gauge space as a frictionless workspace. Downstream adaptation is executed via a precise gauge symmetry break that updates internal computational routing while leaving historical macro-level conditional expectations perfectly intact, reducing performance degradation to zero. The research roadmap demands the formalization of multi-task horizontal carriage memories, allowing networks to scale sequentially across thousands of disjoint tasks without structural leakage.

\subsubsection{Challenge 3: Multi-Scale Genomic and Proteomic Sequence Manifolds}
Biological sequence modeling involves highly non-linear, multi-scale structural configurations where minor sequence mutations can trigger macroscopic phenotypic phase transitions. Classical sequence processing struggles to decouple physical mutation noise from true evolutionary functional signals.

SMG provides the exact mathematical machinery required to model these biological landscapes as infinite-dimensional non-parametric Orlicz statistical manifolds \cite{Pistone1995}. By mapping raw genomic sequences to the environment set and functional phenotypic traits to the identifiable base space, the structural mechanism map projectively isolates evolutionary invariant configurations. Parallel transport along horizontal leaves using the non-parametric connection provides a coordinate-free tool to track how epigenetic changes propagate across generations without deforming core biological signatures.

\subsection{Advanced Theoretical Frontiers}

Beyond immediate algorithmic deployments, SMG opens profound mathematical frontiers that define the next generation of non-parametric information geometry:

\subsubsection{Frontier 1: Stochastic Ehresmann Filtration and Martingale Manifolds}
The field of deterministic gradient flows developed in this manuscript must be extended to account for the discrete, heavy-tailed stochastic noise inherent to empirical deep learning optimization. A critical theoretical frontier is the expansion of the connection 1-form $\omega_f$ into a stochastic differential equation (SDE) framework on infinite-dimensional Orlicz spaces. By formalizing a stochastic Ehresmann projector, we can analyze how optimization noise within the vertical subspace accumulates non-trivial geometric phase shifts along horizontal lifts, leading to a complete martingale theory of representation stability.

\subsubsection{Frontier 2: Dynamic Base Manifolds and Stratified Quotient Spaces}
Continuous lifelong learning requires the learning system to periodically absorb completely novel semantic domains, forcing the identifiable base space to alter its structural dimensions dynamically over time. Future research must explore the geometry of dynamic quotient bundles, where the base manifold $\mathcal{B}_t$ can smoothly change its topological dimension and boundary profiles. This necessitates integrating the theory of infinite-dimensional stratified spaces and cobordisms into the SMG framework, tracking how horizontal leaves bifurcate when the projection map $\pi$ encounters critical geometric singularities.

\subsubsection{Frontier 3: Singularity Resolution via Singular Learning Theory}
Watanabe's Singular Learning Theory has established that hierarchical learning models possess complex self-intersections and algebraic singularities where the classical Fisher Information Matrix loses its regular rank \cite{Watanabe2009}. While SMG isolates these non-identifiable degeneracies inside the vertical fibers, a deeper convergence can be realized by integrating Singular Learning Theory's toolset—specifically the resolution of singularities via algebraic blow-ups—directly into the horizontal leaf construction. Analyzing how singular varieties intersect the horizontal distribution will provide a definitive, coordinate-free classification of optimization behavior near critical phase-transition boundaries.


\end{document}